\documentclass{article}

\usepackage{amsmath, amsfonts, amsthm, bm}

\def\eqref#1{equation~\ref{#1}}

\def\1{\bm{1}}

\DeclareMathAlphabet{\mathsfit}{\encodingdefault}{\sfdefault}{m}{sl}
\SetMathAlphabet{\mathsfit}{bold}{\encodingdefault}{\sfdefault}{bx}{n}

\newcommand{\E}{\mathbb{E}}

\newcommand{\R}{\mathbb{R}}

\DeclareMathOperator*{\argmax}{arg\,max}
\DeclareMathOperator*{\argmin}{arg\,min}

\def\argmax{\mathop{\rm arg\,max}}
\def\argmin{\mathop{\rm arg\,min}}

\def\R{\mathbb{R}}

\usepackage{microtype}
\usepackage{graphicx}
\usepackage{subcaption}
\usepackage{booktabs} 
\usepackage{eqparbox}

\usepackage{hyperref}

\usepackage[accepted]{icml2025}

\usepackage{amsmath}
\usepackage{amssymb}
\usepackage{mathtools}
\usepackage{amsthm}

\usepackage[capitalize,noabbrev]{cleveref}

\theoremstyle{plain}
\newtheorem{theorem}{Theorem}[section]

\newtheorem{lemma}[theorem]{Lemma}
\newtheorem{corollary}[theorem]{Corollary}
\theoremstyle{definition}
\newtheorem{definition}[theorem]{Definition}
\newtheorem{assumption}[theorem]{Assumption}
\theoremstyle{remark}
\newtheorem{remark}[theorem]{Remark}

\usepackage[textsize=tiny]{todonotes}

\icmltitlerunning{Sharpness-aware Adaptive Second-order Optimization
with Stable Hessian Approximation}

\usepackage{xspace}
\def\sassha{\textsc{Sassha}\xspace}
\def\msassha{\textsc{M-Sassha}\xspace}
\def\gsassha{\textsc{G-Sassha}\xspace}

\usepackage{wrapfig}
\usepackage{colortbl}
\usepackage{xcolor}
\usepackage{float}
\usepackage{multirow}
\usepackage{subcaption}
\usepackage{framed}
\usepackage{hhline}
\usepackage[T1]{fontenc}

\usepackage{xparse}

\NewDocumentCommand{\alignednum}{m o}{%
  \makebox[6em][c]{%
    \makebox[0pt][r]{#1}%
    \makebox[0pt][l]{%
        \IfValueTF{#2}
          {\textcolor{black!60}{$_{\pm#2}$}}
          {\phantom{$_{\pm0.00}$}}
    }
  }%
}

\makeatletter
\DeclareRobustCommand\onedot{\futurelet\@let@token\@onedot}
\def\@onedot{\ifx\@let@token.\else.\null\fi\xspace}
\def\eg{\emph{e.g}\onedot} 
\def\ie{\emph{i.e}\onedot}

\makeatother

\definecolor{lgray}{rgb}{.906, .902, .902}
\definecolor{pred}{RGB/cmyk}{200,1,80/.1,1,.30,.10}

\usepackage{eqparbox,array}

\begin{document}

\twocolumn[
\icmltitle{\sassha: Sharpness-aware Adaptive Second-order Optimization\\with Stable Hessian Approximation}



\icmlsetsymbol{equal}{*}

\begin{icmlauthorlist}
\icmlauthor{Dahun Shin}{equal,yyy}
\icmlauthor{Dongyeop Lee}{equal,yyy}
\icmlauthor{Jinseok Chung}{yyy}
\icmlauthor{Namhoon Lee}{yyy}
\end{icmlauthorlist}

\icmlaffiliation{yyy}{POSTECH}

\icmlcorrespondingauthor{Dahun Shin}{dahunshin@postech.ac.kr}
\icmlcorrespondingauthor{Dongyeop Lee}{dylee23@postech.ac.kr}

\icmlkeywords{Machine Learning, ICML}

\vskip 0.3in
]



\printAffiliationsAndNotice{\icmlEqualContribution} 

\begin{abstract}

Approximate second-order optimization methods often exhibit poorer generalization compared to first-order approaches. 
In this work, we look into this issue through the lens of the loss landscape and find that existing second-order methods tend to converge to sharper minima compared to SGD.
In response, we propose \sassha, a novel second-order method designed to enhance generalization by explicitly reducing sharpness of the solution, while stabilizing the computation of approximate Hessians along the optimization trajectory.
In fact, this sharpness minimization scheme is crafted also to accommodate lazy Hessian updates, so as to secure efficiency besides flatness.
To validate its effectiveness, we conduct a wide range of standard deep learning experiments where \sassha demonstrates its outstanding generalization performance that is comparable to, and mostly better than, other methods.
We provide a comprehensive set of analyses including convergence, robustness, stability, efficiency, and cost.
\end{abstract}

\section{Introduction}\label{sec:intro}

Approximate second-order methods have recently gained a surge of interest due to their potential to accelerate the large-scale training process with minimal computational and memory overhead \citep{adahessian, sophia, gupta2018shampoo}.
However, studies also suggest that these methods may undermine generalization, trying to identify underlying factors behind this loss \citep{wilson2017marginal, zhou2020towards, zou2022understanding}.
For instance, \citet{amari2021when} shows that preconditioning hinders achieving the optimal bias for population risk, and \citet{wadia2021whitening} points to negative effect of whitening data.

While the precise understanding is still under investigation, many studies have suggested a strong correlation between the flatness of minima and their generalization capabilities \citep{keskar2016large}, spurring the development of optimization techniques aimed at inducing flat minima \citep{chaudhari2017entropy, izmailov2018averaging, sam, antipgd_orvieto22a}.
Inspired by this, we raise an important question in this work:
what type of minima do second-order methods converge to, and is there any potential for improving their generalization performance based on that?

\begin{figure}
    \centering
    \resizebox{0.6\linewidth}{!}{
        \includegraphics[width=1.3\linewidth,trim={0 6em 0 0},clip]{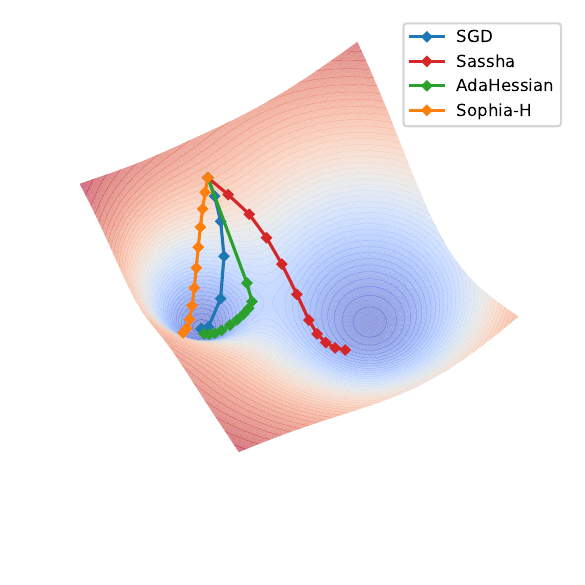}
    }
    \caption{
        Motivating toy example (a mixture of bivariate Gaussian densities).
        \sassha converges to a flat minimum unlike others.
    }
    \label{fig:motivate}    
\end{figure}

To answer these questions, we first measure the sharpness of different second-order methods using diverse metrics, suggesting that they converge to significantly sharper minima compared to stochastic gradient descent (SGD).
Then, we propose \sassha---\underline{\textbf{S}}harpness-aware \underline{\textbf{A}}daptive \underline{\textbf{S}}econd-order optimization with \underline{\textbf{S}}table \underline{\textbf{H}}essian \underline{\textbf{A}}pproximation---designed to enhance the generalization of approximate second-order methods by explicitly reducing sharpness (see \cref{fig:motivate} for the basic results).

\sassha incorporates a sharpness minimization scheme similar to SAM \citep{sam} into the second-order optimization framework, in which the Hessian diagonal is estimated.
Such estimates, however, can become numerically unstable when enforcing the sharpness reduction process.
To increase stability while preserving the benefits of reduced sharpness, we make a series of well-engineered design choices based on principles studied in the literature.
This not only smoothly adjusts underestimated curvature, but also enables efficient reuse of previously computed Hessians, resulting in a stable and efficient algorithm.

We extensively evaluate the effectiveness of \sassha across diverse vision and natural language tasks.
Our results reveal that \sassha consistently achieves flatter minima and attains stronger generalization performance, all compared to existing practical second-order methods, and interestingly, to first-order methods including SGD, AdamW, and SAM.
Furthermore, we provide an array of additional analyses to comprehensively study \sassha including convergence, robustness, stability, efficiency, and cost.

\bgroup
\begin{table*}[!t]
    \centering
    \caption{
    Sharpness measurements of the solutions found by different optimizers and their generalization for ResNet-32 on CIFAR-100.
    Approximate second-order methods tend to yield highly sharp solutions and poor generalization compared to \sassha.
    We provide more results for other workloads in \cref{app:sharp} where the same trend holds.}
    \vskip 0.1in
    \resizebox{0.95\linewidth}{!}{
        \begin{tabular}{lcccccccc}
            \toprule
             &  \multicolumn{4}{c}{Sharpness} & \multicolumn{2}{c}{Generalization} \\
             \cmidrule(l{3pt}r{3pt}){2-5} \cmidrule(l{3pt}r{3pt}){6-7}
             & {$\lambda_{max}(H)$}              & {$\operatorname{tr}(H)_{\times 10^3}$} & $\delta L_\text{grad}$
             &  $\delta L_{\text{avg}\times 10^{-3}}$ & $L_{\text{val}}$ & $\text{Acc}_\text{val}$ (\%) \\ \midrule
             SGD        &
             \alignednum{265}[25.00]          &
             \alignednum{7.290}[0.300]        &
             \alignednum{0.703}[0.132]        &
             \alignednum{1.310}[1.030]        &
             \alignednum{1.260}[0.001]        &
             \alignednum{69.32}[0.19]         \\
             
             AdaHessian  &
             \alignednum{11992}[5779]         & 
             \alignednum{46.94}[17.60]        & 
             \alignednum{4.119}[1.136]        & 
             \alignednum{12.50}[6.080]        & 
             \alignednum{1.377}[0.070]        & 
             \alignednum{68.06}[0.22]         \\

             Sophia-H    & 
             \alignednum{22797}[10857]        &
             \alignednum{68.15}[20.19]        & 
             \alignednum{8.130}[3.082]        &
             \alignednum{19.19}[6.380]        & 
             \alignednum{1.463}[0.022]        & 
             \alignednum{67.76}[0.37]         \\

             Shampoo &
             \alignednum{436374}[9017]        &
             \alignednum{6823}[664.7]         &
             \alignednum{73.27}[12.51]        & 
             \alignednum{49307489}[56979794]  & 
             \alignednum{1.386}[0.010]        & 
             \alignednum{64.08}[0.46]         \\
             
             \midrule
             
            \rowcolor{green!20} \sassha       &
            \alignednum{107}[40.00]           &
            \alignednum{1.870}[0.650]         &
            \alignednum{0.238}[0.088]         &
            \alignednum{0.650}[0.860]         &
            \alignednum{0.961}[0.005]         & 
            \alignednum{72.14}[0.16]          \\
            \bottomrule
        \end{tabular}
    }
    \vskip 0.1in
    \label{tab:sharp}
\end{table*}
\egroup

\begin{figure*}[t!]
    \centering
        \begin{subfigure}[b]{0.245\textwidth}
            \centering
            \includegraphics[width=\textwidth, trim={0 2em 4em 4em}, clip]{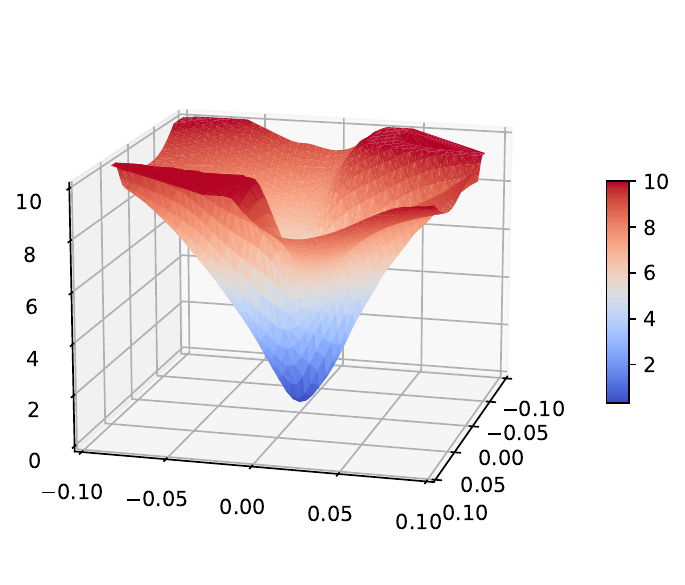}
            \caption{AdaHessian}
        \end{subfigure}
        \begin{subfigure}[b]{0.245\textwidth}
            \centering
            \includegraphics[width=\textwidth, trim={0 2em 4em 4em}, clip]{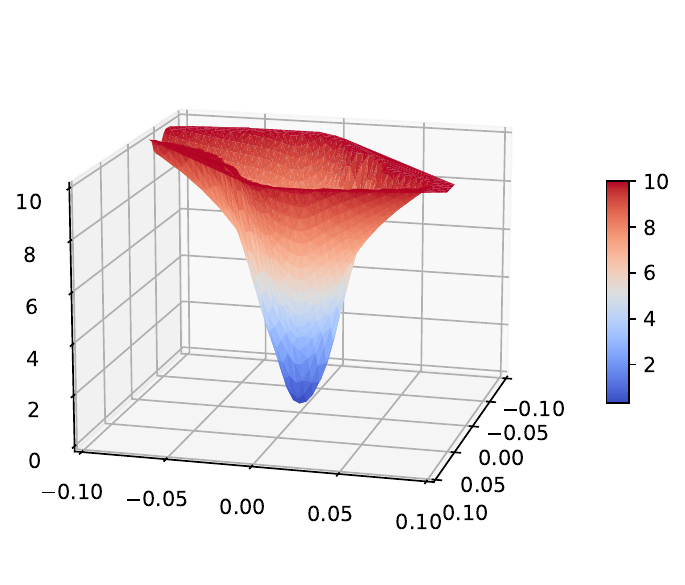}
            \caption{Sophia-H}
        \end{subfigure}
     \begin{subfigure}[b]{0.245\textwidth}
        \centering
        \includegraphics[width=\textwidth, trim={0 2em 4.6em 4em}, clip]{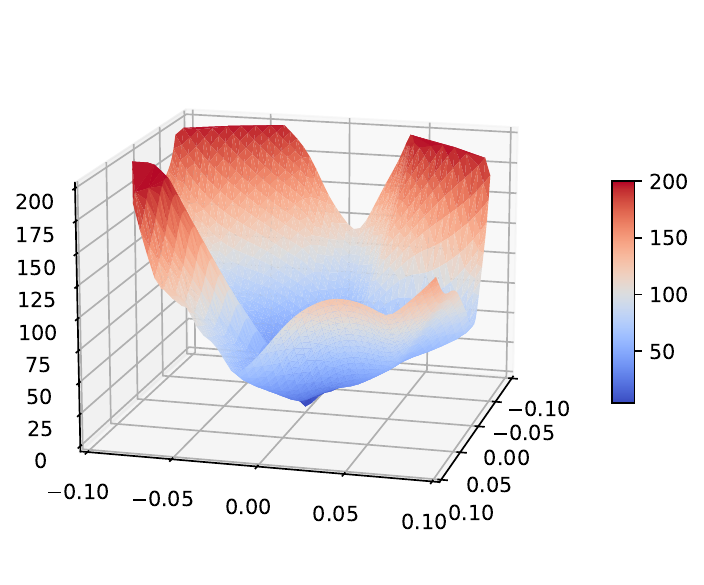}
        \caption{Shampoo}
     \end{subfigure}
    \begin{subfigure}[b]{0.245\textwidth}
        \centering
        \includegraphics[width=\textwidth, trim={0 2em 4em 4em}, clip]{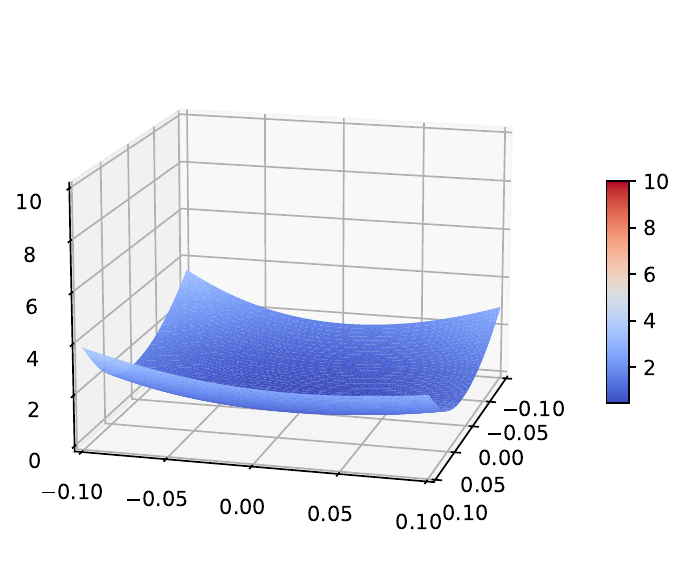}
        \caption{\sassha}
    \end{subfigure}
    \caption{Visualization of the found solutions along the directions of the dominant eigenvectors.}
    \label{fig:landscape}
    
\end{figure*}

\section{Related Works} \label{sec:related_works}
\paragraph{Second-order optimization for deep learning.}
First-order methods such as SGD are popular optimization methods for deep learning due to their low per-iteration cost and good generalization performance \citep{hardt2016train}.
However, these methods have two major drawbacks: slow convergence under ill-conditioned landscapes and high sensitivity to hyperparameter choices such as learning rate \citep{doi:10.1137/1.9781611976236}.
Adaptive methods \citep{duchi2011adaptive, hinton2012neural, kingma2014adam} propose using empirical Fisher-type preconditioning to alleviate these issues, though recent studies suggest their insufficiency to do so \citep{kunstner2019limitations}.
This has led to recent interest in developing approximate second-order methods such as Hessian-Free Inexact Newton methods \citep{martens2010deep, kiros2013training}, stochastic quasi-Newton methods \citep{byrd2016stochastic, gower2016stochastic}, Gauss-Newton methods \citep{schraudolph2002fast, botev2017practical}, natural gradient methods \citep{amari2000adaptive}, and Kronecker-factored approximations \citep{martens2015optimizing, goldfarb2020practical}.
However, these approaches still incur non-trivial memory and computational costs, or are difficult to parallelize, limiting their applicability to large-scale problems such as deep learning.
This has driven growing interest in developing more scalable and efficient second-order approaches, particularly through diagonal scaling methods \citep{bottou, adahessian, sophia}, to better accommodate large-scale deep learning scenarios.

\paragraph{Sharpness minimization for generalization.}
The relationship between the geometry of the loss landscape and the generalization ability of neural networks was first discussed in the work of \citet{NIPS1994_Hochreiter}, and the interest in this subject has persisted over time.
Expanding on this foundation, subsequent studies have explored the impact of flat regions on generalization both empirically and theoretically \citep{Hochreiter1997, keskar2016large, DR17, NIPS2017_Neyshabur, dinh17b, 2020Fantastic}. 
Motivated by this, various approaches have been proposed to achieve flat minima such as regularizing local entropy \citep{chaudhari2017entropy}, averaging model weights \citep{izmailov2018averaging}, explicitly regularizing sharpness by solving a min-max problem \citep{sam}, and injecting anti-correlated noise \citep{antipgd_orvieto22a}, to name a few.
In particular, the sharpness-aware minimization (SAM) \citep{sam} has attracted significant attention for its strong generalization performance across various domains \citep{chenvision, bahri2022sharpness, qu2022generalized} and its robustness to label noise \citep{baek2024why}.
Nevertheless, to our knowledge, the sharpness minimization scheme has not been studied to enable second-order methods to find flat minima as of yet.

\section{Practical Second-order Optimizers Converge to Sharp Minima} \label{sec:3_sharpness_measurement}

In this section, we investigate the sharpness of minima obtained by approximate second-order methods and their generalization properties.
We posit that poor generalization of second-order methods reported in the literature \citep{amari2021when, wadia2021whitening} can potentially be attributed to sharpness of their solutions.

We employ four metrics frequently used in the literature: maximum eigenvalue of the Hessian, the trace of Hessian, gradient-direction sharpness, and average sharpness \citep{Hochreiter1997, jastrzkebski2018relation, xie2020diffusion, saf, chenvision}.
The first two, denoted as $\lambda_{\max}(H)$ and $\operatorname{tr}(H)$, are often used as standard mathematical measures for the worst-case and the average curvature computed using the power iteration method and the Hutchinson trace estimation, respectively.
The other two measures, $\delta L_\text{grad}$ and $\delta L_\text{avg}$, assess sharpness based on the loss difference under perturbations.
$\delta L_\text{grad}$ evaluates sharpness in the gradient direction and is computed as $L(x^\star+\rho\nabla L(x^\star)/\|\nabla L(x^\star)\|)-L(x^\star)$.
$\delta L_\text{avg}$ computes the average loss difference over Gaussian random perturbations, expressed as
$\E_{z\sim \mathcal{N}(0, 1)}[L(x^\star+\rho z/\|z\|)-L(x^\star)]$.
Here we choose $\rho=0.1$ for the scale of the perturbation.

With these, we measure the sharpness of the minima found by three approximate second-order methods designed for deep learning; Sophia-H \citep{sophia}, AdaHessian \citep{adahessian}, and Shampoo \citep{gupta2018shampoo}, and compare them with \sassha as well as SGD for reference.
We also compute the validation loss and accuracy to see any correlation between sharpness and generalization of these solutions.
The results are presented in \cref{tab:sharp}.

We observe that existing second-order methods produce solutions with significantly higher sharpness compared to \sassha in all sharpness metrics, which also correlates well with their generalization.
We also provide a visualization of the loss landscape for the found solutions, where we find that the solutions obtained by second-order methods are indeed much sharper than that of \sassha (\cref{fig:landscape}).

\section{Method}

In the previous section, we observe that the generalization performance of approximate second-order algorithms anti-correlates with the sharpness of their solutions.
Based on this, we introduce \sassha---a novel adaptive second-order method designed to improve generalization by reducing sharpness without adversely impacting the Hessian.

\subsection{Sharpness-aware Second-order Optimization}\label{sec:psol}

We consider a min-max problem, similar to \citet{keskar2016large, sam}, to minimize sharpness. 
This is defined as minimizing the objective $f$ within the entire $\rho$-ball neighborhood:
\begin{equation} \label{eq:sam_optimization}
    \min_{x\in \R^d} \max_{\|\epsilon\|_2\leq \rho} f(x+\epsilon),
\end{equation}

Based on this, we construct our sharpness minimization technique for second-order optimization as follows.
We first follow a similar procedure as \citet{sam} by solving for $\epsilon$ on the first-order approximation of the objective, which exactly solves the dual norm problem as follows:
\begin{align}
    \epsilon_t^\star = \argmax_{\|\epsilon\|_2\leq \rho} f(x_t) + \epsilon^\top\nabla f(x_t) 
    &= \rho\frac{\nabla f(x_t)}{\|\nabla f(x_t)\|_2}.
\end{align}
We plug this back to yield the following perturbed objective function:
\begin{equation}
    \Tilde{f}_{t}(x) \coloneqq f\left(x+\rho\frac{\nabla f(x_t)}{\|\nabla f(x_t)\|_2}\right),
\end{equation}
which shifts the point of the approximately highest function value within the neighborhood to the current iterate.

With this sharpness-penalized objective, we proceed to make a second-order Taylor approximation: 
\begin{align*}
    x_{t+1} = \argmin_x~ & \Tilde{f}_{t}\left(x_t\right) + \nabla \Tilde{f}_{t}\left(x_t\right)^\top (x-x_t) \nonumber \\ 
    &+ (x-x_t)^\top \Tilde{H}_{t}\left(x_t\right) (x-x_t),
\end{align*}
where $\Tilde{H}_{t}$ denotes the Hessian of $\Tilde{f}_{t}$.
Using the first-order optimality condition, we derive the basis update rule for our sharpness-aware second-order optimization: 
\begin{align}
    x_{t+1} &= x_t - \Tilde{H}_{t}\left(x_t\right)^{-1} \nabla \Tilde{f}_{t}\left(x_t\right) \nonumber \\
            &= x_t - H \left(x_t+ \epsilon_t^\star \right)^{-1} \nabla f \left(x_t+ \epsilon_t^\star \right),
\label{eq: update}
\end{align}
where $H$ denotes the Hessian of the original objective $f$.

Practical second-order methods must rely on approximately estimated Hessians (\ie, $H \rightarrow \widehat{H}$) since the exact computation is prohibitively expensive for large-scale problems.
We choose to employ the diagonal approximation via Hutchinson's method.
However, as we will show in our analysis (\cref{sec:sqrt_ablation}), we find that these estimates can become numerically unstable during the sharpness reduction process, as it penalizes Hessian entries close to zero.
This can lead to fatal underestimation of the diagonal Hessian compared to scenarios without sharpness minimization, significantly disrupting training.
We propose a stable Hessian approximation to address these issues in the following sections.

\subsection{Improving Stability} \label{sec:method_stability}

\paragraph{Alleviating divergence.} \label{sec: alleviating divergence}
Approximate second-order methods can yield overly large steps when their diagonal Hessian estimations underestimate the curvature \citep{dauphin2015equilibrated}.
However, this instability seems to be more present under sharpness minimization, presumably due to smaller top Hessian eigenvalue $\lambda_1$ \citep{agarwala2023sam, shin2024critical} yielding smaller estimated diagonal entries on average: 
\begin{equation*}
    \E\left[\frac{1}{d}\sum_{i=1}^d{\widehat{H}_{ii}}\right] 
    = \frac{1}{d}\sum_{i=1}^d{\E[\widehat{H}]_{ii}} 
    = \frac{\operatorname{tr}(H)}{d} 
    = \frac{1}{d}\sum_{i=1}^d{\lambda_i} 
    \leq \lambda_1.
\end{equation*}
This tendency toward zero intensifies numerical instability during Hessian inversion, increasing the risk of training failures.

A conventional approach to mitigating this involves damping or clipping, which, while stabilizing it reasonably well, requires carefully tuning their additional hyperparameters.
Instead, we propose square rooting the Hessian (\ie, $|\widehat{H}|^{1/2}$), which effectively mitigates instability and allows improved generalization performance over other alternatives without additional hyperparameters. 
We present empirical validation of this in \cref{sec:sqrt_ablation} and \cref{app:sqrt_alternatives}.

Its benefits can be understood from two perspectives.
First, the square root preserves the relative scale between each element of the Hessian while smoothly increasing the magnitude of the near-zero diagonal Hessian entries in the denominator (\ie, $h < \sqrt{h} $ if $0 < h < 1$). 
This property is particularly valuable when the sharpness minimization is underway, where the overall Hessian values tend to be small.
In such cases, even small differences between Hessian elements may carry nontrivial curvature information.
Square-rooting can help retain this relative structure while also mitigating numerical instability caused by underestimated curvature.
In contrast, both damping and clipping operate by entirely shifting or abruptly replacing Hessian values based on a predefined and fixed threshold criterion.
As a result, when the Hessian is generally small due to sharpness minimization, informative dimensions may fall below the threshold, removing potentially critical variations and hence deteriorating the quality of preconditioning.
This behavior can also make the method more sensitive to the choice of the threshold hyperparameter.

Second, it can further be interpreted as a geometric interpolation between the identity matrix $ I$ and the preconditioning matrix $|\widehat{H}|^{\alpha}$, which, as theoretically analyzed in \citep{amari2021when}, provide a natural mechanism for balancing between bias and variance of the population risk, thereby improving generalization.
We specifically adopt $\alpha = 1 / 2$ (i.e., square root), as it has consistently demonstrated robust performance across various scenarios \citep{amari2021when, kingma2014adam}.

\paragraph{Avoiding critical points.}

A well-known limitation of applying second-order optimization to deep learning objectives is the risk of convergence to saddle points or local maxima.
To mitigate this, we attend to the prior works of \citet{becker1988improving, adahessian} and employ the absolute function to conservatively adjust the negative entries of the diagonal Hessian to be positive, \ie
\begin{equation}
    \lvert \widehat{H} \rvert := \sum_{i=1}^d \lvert \widehat{H}_{ii} \rvert \mathbf{e}_i\mathbf{e}_i^\top
\end{equation}
where $\widehat{H}_{ii}$ and $\mathbf{e}_i$ are the $i^{\text{th}}$ diagonal entry of the approximate diagonal Hessian and the $i^{\text{th}}$ standard basis vector, respectively.
Here, the basic idea is to maintain the directionality of the gradient by flipping the sign of the negative entries in the Hessian, preserving the original magnitude of its values.
This allows our method not only to take descent steps along directions of originally negative curvature, but also to faithfully preserve the relative scale specifically among the negative elements of the Hessian.
As a result, it mitigates the risk of convergence to the undesired critical points.
We empirically validate the effectiveness of this approach in \cref{sec:abs_ablation}.

\subsection{Improving Efficiency via Lazy Hessian Update} \label{sec:lazy_hessian}

While the diagonal Hessian approximation can significantly reduce computations, it still requires at least twice as much backpropagation compared to first-order methods.
Here we attempt to further alleviate this by lazily computing the Hessian every $k$ steps:

{\small
\begin{equation*}\label{eq:lazy_update}
    D_t = \begin{cases}
    \beta_2 D_{t-1} + (1-\beta_2) \lvert \widehat{H} (x_t+\epsilon^\star_t) \rvert & \texttt{if } t \  \operatorname{mod} \  k = 1 \\ 
    D_{t-1} & \texttt{otherwise} 
    \end{cases},
\end{equation*}
}

where $D_t$ and $\beta_2$ are the moving average of the Hessian and its hyperparameter, respectively.
This reduces the overhead from additional Hessian computation by $1/k$.
We set $k=10$ for all experiments in this work unless stated otherwise.

However, extensive Hessian reusing will lead to significant performance degradation since it would no longer accurately reflect the current curvature \citep{lazyhessian}. 
Interestingly, \sassha is quite resilient against prolonged reusing, keeping its performance relatively high over longer Hessian reusing compared to other approximate second-order methods.
Our investigation reveals that along the trajectory of \sassha, the Hessian tends to change less frequently than existing alternatives.
We hypothesize that the introduction of sharpness minimization plays an integral role in this phenomenon by biasing the optimization path toward regions with lower curvature change, allowing the prior Hessian to remain relevant over more extended steps. 
We provide a detailed analysis of the lazy Hessian updates in \cref{sec:emp_lazy_hess}.

\subsection{Algorithm}

The exact steps of \sassha is outlined in \cref{algo:sassha}.
We also compare \sassha with other adaptive and second-order methods in detail in \cref{app:comparison}, where one can see the exact differences between these sophisticated methods.

\subsection{Convergence Analysis} \label{sec:main_conv}

In this section, we present a standard convergence analysis of \sassha under the following assumptions.
\begin{assumption} \label{method_assumption:lowbound}
    The function $ f $ is bounded from below, i.e., \( f^* := \inf_x f(x) > -\infty \).
\end{assumption}
\begin{assumption} \label{method_assumption:smooth}
    The function $ f $ is twice differentiable, convex, and $\beta$-smooth. That is, $0\preceq \nabla^2f \preceq \beta.$
\end{assumption}
\begin{assumption}    
    \label{method_assumption:bounded_steps}
    \textit{The gradient \( \nabla f(x_t) \) is nonzero for a finite number of iterations, i.e., \( \nabla f(x_t) \neq 0 \) for all \( t \in \{1, 2, \dots, n\} \).}
\end{assumption}
Under these assumptions, we derive a descent inequality for $f(x_t)$ by leveraging Adam-like proof techniques from \citet{li2023convergence} to handle the diagonal Hessian and employing smoothness-based bounds to account for the perturbation step based on analyses of \citet{khanh2024fundamental}.
Now we give the convergence results as follows:
\begin{theorem}
    \label{thm:method_theorem}
    Under Assumptions \ref{method_assumption:lowbound}-\ref{method_assumption:bounded_steps}, given any initial point $x_0 \in \mathbb{R}^d$, let $\{x_t\}$ be generated by the update rule \sassha \cref{remark:sassha_iteration_adamlike} with step sizes \( \eta_t \) and perturbation radii \( \rho_t \) satisfying  
    $
        \sum_{t=1}^{\infty} \eta_t = \infty, \quad \sum_{t=1}^{\infty} \eta_t^2 < \infty, \quad \sum_{t=1}^{\infty} \rho_t^2 \eta_t < \infty
    $. 
    Then, we have $\liminf_{t \to \infty} \|\nabla f(x_t)\| = 0$.
\end{theorem}
This preliminary result indicates that any limit point of \sassha is a stationary point of $f$, ensuring progress towards optimal solutions.
We refer to \cref{app:convergence} for the full proof details.

\begin{algorithm}[t!]
    \small
    \caption{\sassha algorithm}
    \begin{algorithmic}[1]
          \STATE {\bf Input:} Initial parameter $x_0$, learning rate $\{\eta_\mathnormal{t}\}$,  moving average parameters $\beta_{1},\beta_{2}$, Hessian update interval $k$, weight decay parameter $\lambda$
          \STATE Set $\mathnormal{m}_{-1} = 0$, $D_{-1} = 0$
          \FOR{$\mathnormal{t}=1$ {\bf to} $T$}
            \STATE $\mathnormal{g}_\mathnormal{t} = \nabla f_\mathcal{B}(x_t)$
            \STATE $\epsilon^\star_t=\rho g_t/\|g_t\|_2$  
            \STATE $\Tilde{g}_{t} = \nabla f_\mathcal{B}(x_t + \epsilon^\star_t)$
            \STATE  $m_t = \beta_{1} m_{t-1} + (1 - \beta_{1}) \Tilde{g}_{t}$ 
            \STATE $\overline{m}_t = m_t / (1 - \beta_1^t)$ 
            
            \IF{$t \operatorname{mod} k = 1$}
                \STATE $\Tilde{H}_t = \widehat{H}(x_t + \epsilon^\star_t)$ \hfill$\triangleright$ \cref{sec:psol} 
                \STATE $D_\mathnormal{t} = \beta_2 D_{t-1} + (1 - \beta_2) |\Tilde{H}_t| $ 
                \STATE $\overline{D}_{t} = \sqrt{D_t/(1 - \beta_2^t)}$ \hfill $\triangleright$ 
                \cref{sec:method_stability}
            \ELSE
                \STATE $\overline{D}_t=\overline{D}_{t-1}$ \hfill$\triangleright$ \cref{sec:lazy_hessian}
            \ENDIF
            \STATE $x_{t+1} = x_t - \eta_{t} \overline{D}_{t}^{-1} \overline{m}_t - \eta_t \lambda x_t$
        \ENDFOR
    \end{algorithmic}
    \label{algo:sassha}
\end{algorithm}

\subsection{Flatness Analysis}\label{sec:linstab}

To understand how \sassha can end up in flatter minima as observed in \cref{sec:3_sharpness_measurement}, we attend to linear stability analysis.
Originally developed to explain similar behavior in SGD \citep{wu2018sgd} and also later extended to SAM \citep{shin2024critical}, this framework suggests that an optimizer does not settle in every minimum it approaches, but instead escapes unless it encounters one that satisfies specific stability conditions—conditions that can vary between optimizers.
Based on this, we demonstrate that \sassha requires a minimum to possess a certain level of flatness and Hessian uniformity to settle in it, whereas approximate second-order optimizers do not necessarily require such restriction, thus allowing it to stay in much sharper minima.

Consider a general optimizer $x_{t+1}=x_t-G(x_t;\xi_t)$ with randomness induced by i.i.d. variables $\xi_t$ that are independent from the iterate $ x_t $.
Also, let us assume that the minima possess a fixed point $x^\star$ such that $\nabla f_{\xi}(x^\star) = 0$ for any $\xi$.

With these, we define the linear stability of the fixed point $x^\star$ as follows:
\begin{definition}
    (Linear stability).
    A fixed point $ x^\star $ is \emph{linearly stable} for the optimizer $G$ if there exists a constant $ C $ such that
    \begin{equation}
    \mathbb{E} \bigl[ \lVert \hat{x}_t - x^\star \rVert^2 \bigr] \leq C \lVert \hat{x}_0 - x^\star \rVert^2, \text{ for all } t > 0 \nonumber 
    \end{equation}
    under the linearized dynamics near $x^\star$: $\hat{x}_{t+1} = \hat{x}_t - \nabla G (x^\star)(\hat{x}_t - x^\star)$, \ie, when it does not deviate infinitely far from $x^\star$.
\end{definition}
Here, this linear dynamic arises when $x_t$ has approached sufficiently close to $x^\star$ such that the loss becomes approximately quadratic.

Under this framework, we present the necessary conditions of the linearly stable minima of \sassha in the following corollary.

\begin{corollary}
    Assume without loss of generality that $ x^\star = 0 $.
    Then, the linearly stable fixed point $ x^\star $ of \sassha satisfy the following necessary conditions:
    \begin{align}
        0 \leq&  a(1 + \rho a ) \leq \frac{2 \epsilon }{\eta } , \quad 
        0 \leq s^2_2 \leq \frac{ \epsilon^2 }{\eta^2 -2 \eta \rho \epsilon }, \nonumber \\
        0 \leq& s^2_3 \leq \frac{ \epsilon^2 }{2 \eta^2 \rho }, \quad
        0 \leq s^2_4 \leq \frac{ \epsilon^2 }{\eta^2 \rho^2 }, \label{cond:necess}
    \end{align}
    where $ a = \lambda_{max} (\E[H_\xi]) $ and $ s_k = \lambda_{max} ((\E[H_\xi^k] - \E[H_\xi]^k)^{1/k})$ are the sharpness and the non-uniformity of the stochastic Hessian measured with the $k$-th moment, respectively.
\end{corollary}

A detailed proof is provided in \cref{app:linear_stability}.

These results indicate that \sassha escapes from minima unless they satisfies both low sharpness and uniformly distributed Hessian moments, with the conditions becoming much stricter with larger $\rho$ and $\eta$.
In comparison, standard approximate second-order methods of the form $x_{t+1}=x_t-\eta P_t\nabla f(x_t)$ with $P \approx H^{-1}$ remain stable without such conditions, as shown by \citet{wu2018sgd}:
\begin{equation}
    \lambda_{\max}(P^{-1}H) \approx 1 \leq \frac{2}{\eta}, \nonumber    
\end{equation}
allowing convergence to minima of any sharpness provided $\eta \leq 2$.

\begin{table*}[t!]
    \centering
    \caption{Image classification results of various optimization methods in terms of final validation accuracy (mean$\pm$std).
    \sassha consistently outperforms the other methods for all workloads.
    * means \emph{omitted} due to excessive computational requirements.}
    \vskip 0.1in
    \resizebox{0.9\linewidth}{!}{
        \begin{tabular}{clcccccc}
        \toprule
         & 
         & \multicolumn{2}{c}{CIFAR-10} 
         & \multicolumn{2}{c}{CIFAR-100} 
         & \multicolumn{2}{c}{ImageNet} \\
         \cmidrule(l{3pt}r{3pt}){3-4} \cmidrule(l{3pt}r{3pt}){5-6} \cmidrule(l{3pt}r{3pt}){7-8}
         \multicolumn{1}{c}{ Category }
         & \multicolumn{1}{c}{ Method }
         & \multicolumn{1}{c}{ ResNet-20 } 
         & \multicolumn{1}{c}{ ResNet-32 } 
         & \multicolumn{1}{c}{ ResNet-32 }  
         & \multicolumn{1}{c}{ WRN-28-10} 
         & \multicolumn{1}{c}{ ResNet-50 } 
         & \multicolumn{1}{c}{ ViT-s-32} \\ \midrule

       \multirow{4}{*}{First-order}  
       & SGD       & 
         $ 92.03 _{ \textcolor{black!60}{\pm 0.32} } $    &
         $ 92.69 _{\textcolor{black!60}{\pm 0.06} }  $    &
         $ 69.32 _{\textcolor{black!60}{\pm 0.19} }  $    &
         $ 80.06 _{\textcolor{black!60}{\pm 0.15} }  $    &
         $ 75.58 _{\textcolor{black!60}{\pm 0.05} }  $    &
         $ 62.90 _{\textcolor{black!60}{\pm 0.36} }  $   \\

        & AdamW      & 
        $ 92.04 _{\textcolor{black!60}{\pm 0.11} }  $     &
        $ 92.42 _{\textcolor{black!60}{\pm 0.13} }  $     &
        $ 68.78 _{\textcolor{black!60}{\pm 0.22} }  $     &
        $ 79.09 _{\textcolor{black!60}{\pm 0.35} }  $     &
        $ 75.38 _{\textcolor{black!60}{\pm 0.08} }  $     &
        $ 66.46 _{\textcolor{black!60}{\pm 0.15} }  $    \\
        
        & SAM $_{\text{SGD}}$  &
        $ 92.85 _{\textcolor{black!60}{\pm 0.07} }  $    &
        $ 93.89 _{\textcolor{black!60}{\pm 0.13} }  $    &
        $ 71.99 _{\textcolor{black!60}{\pm 0.20} }  $    &
        $ 83.14 _{\textcolor{black!60}{\pm 0.13} }  $    &
        $ 76.36 _{\textcolor{black!60}{\pm 0.16} }  $    &
        $ 64.54 _{\textcolor{black!60}{\pm 0.63} }  $    \\
        
        & SAM $_{\text{AdamW}}$  &
        $ 92.77 _{\textcolor{black!60}{\pm 0.29} }  $    &
        $ 93.45 _{\textcolor{black!60}{\pm 0.24} }  $    &
        $ 71.15 _{\textcolor{black!60}{\pm 0.37} }  $    &
        $ 82.88 _{\textcolor{black!60}{\pm 0.31} }  $    &
        $ 76.35 _{\textcolor{black!60}{\pm 0.16} }  $    &
        $ 68.31 _{\textcolor{black!60}{\pm 0.17} }  $    \\

        \midrule
        
        \multirow{4}{*}{Second-order} &
        AdaHessian &
        $ 92.00 _{\textcolor{black!60}{\pm 0.17} } $  &
        $ 92.48 _{\textcolor{black!60}{\pm 0.15} } $  &
        $ 68.06 _{\textcolor{black!60}{\pm 0.22} } $  &
        $ 76.92 _{\textcolor{black!60}{\pm 0.26} } $  &
        $ 73.64 _{\textcolor{black!60}{\pm 0.16} } $  &
        $ 66.42 _{\textcolor{black!60}{\pm 0.23} } $  \\
        
        & Sophia-H   & 
        $ 91.81 _{\textcolor{black!60}{\pm 0.27} } $  &
        $ 91.99 _{\textcolor{black!60}{\pm 0.08} } $  &
        $ 67.76 _{\textcolor{black!60}{\pm 0.37} } $  & 
        $ 79.35 _{\textcolor{black!60}{\pm 0.24} } $  & 
        $ 72.06 _{\textcolor{black!60}{\pm 0.49} } $  &
        $ 62.44 _{\textcolor{black!60}{\pm 0.36} } $  \\
        
        & Shampoo    & 
        $ 88.55 _ {\textcolor{black!60}{\pm 0.83}}$  &
        $ 90.23 _{\textcolor{black!60}{\pm 0.24}} $  &
        $ 64.08 _{\textcolor{black!60}{\pm 0.46}} $  &
        $ 74.06 _{\textcolor{black!60}{\pm 1.28}} $  &
        $*$                                          &
        $*$  \\
        
        \cmidrule(l{3pt}r{3pt}){2-8}
        
        \rowcolor{green!20} \cellcolor{white}  &
        \sassha    &
        $ \textbf{92.98} _{\textcolor{black!60}{\pm 0.05} }  $ &
        $ \textbf{94.09} _{\textcolor{black!60}{\pm 0.24} }  $ &
        $ \textbf{72.14} _{\textcolor{black!60}{\pm 0.16} }  $ & 
        $ \textbf{83.54} _{\textcolor{black!60}{\pm 0.08} }  $ &
        $ \textbf{76.43} _{\textcolor{black!60}{\pm 0.18} }  $ &
        $ \textbf{69.20} _{\textcolor{black!60}{\pm 0.30} }  $ \\
        
        \bottomrule
        \end{tabular}
    }
    \vskip 0.1in
    \label{tab:im_cls_results}
\end{table*}

\begin{figure*}[t!]
    \centering
    \resizebox{0.9\linewidth}{!}{
    \includegraphics[width=0.325\linewidth]{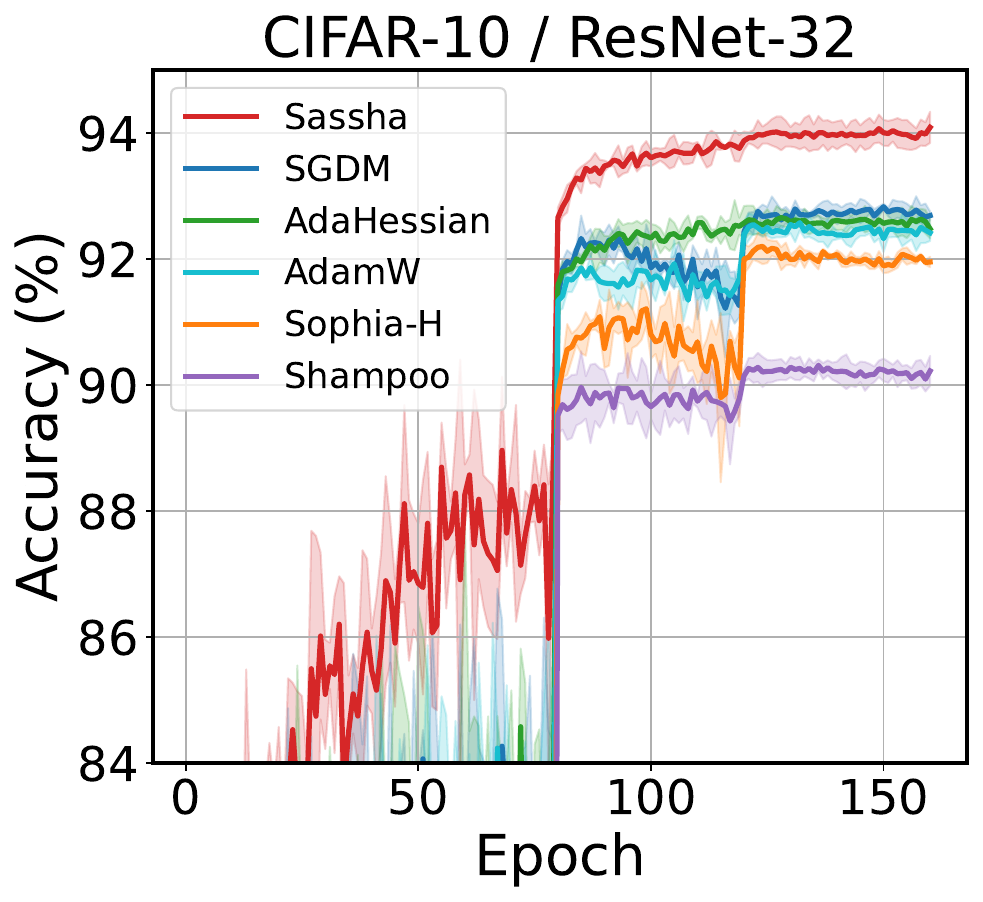}
    \includegraphics[width=0.325\linewidth]{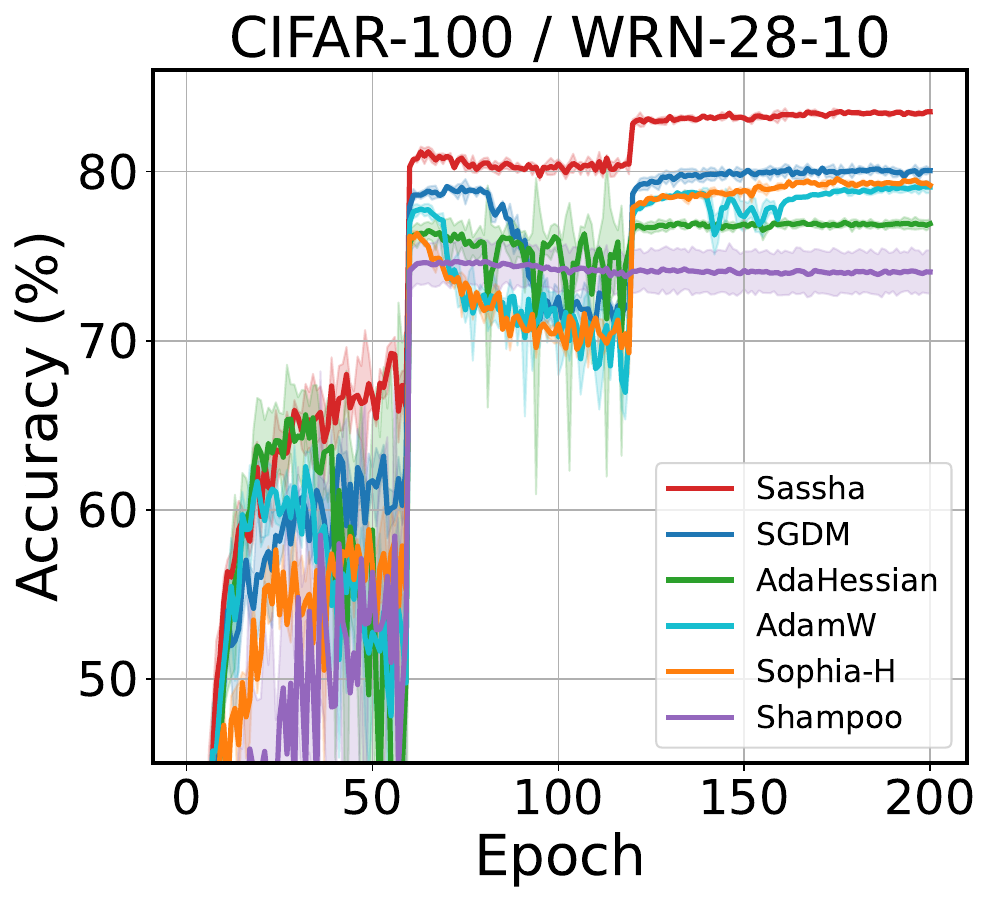}
    \includegraphics[width=0.325\linewidth]{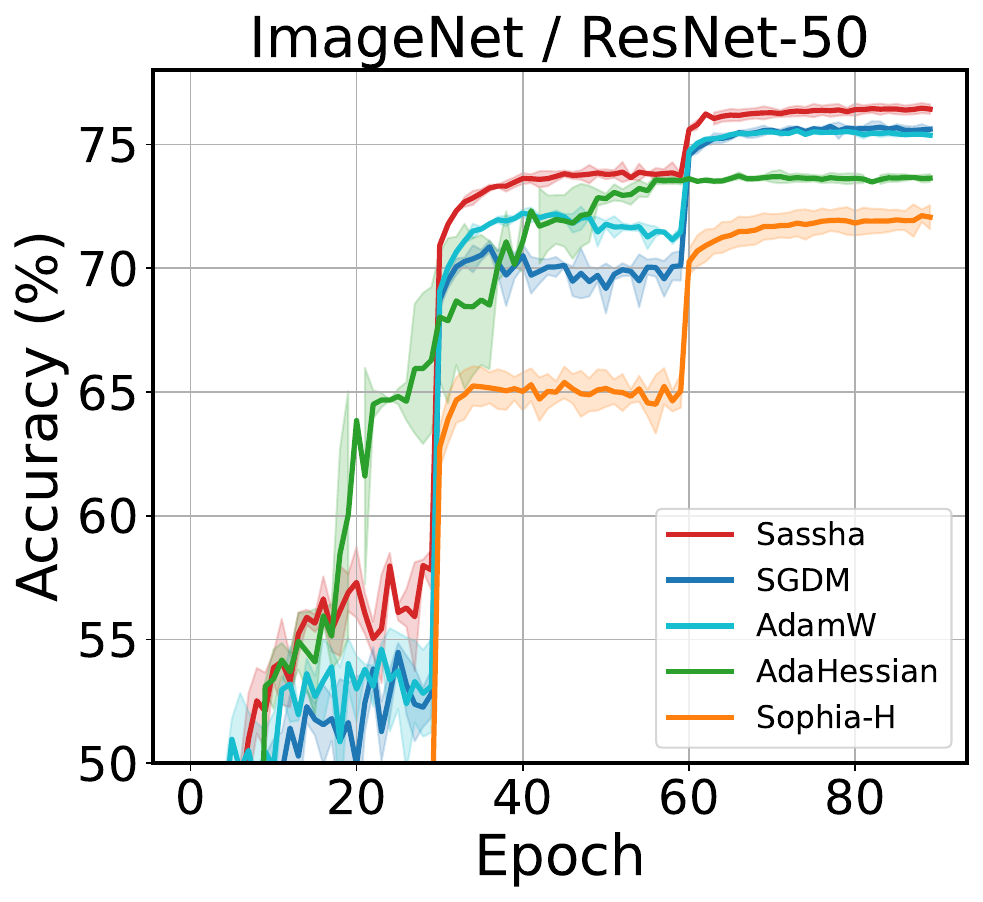}
    }
    \caption{
    Validation accuracy curves along the training trajectory.
    We also provide loss curves in \cref{app:valloss}.
    }
    \label{fig:im_cls_results}
\end{figure*}

\section{Evaluations}
\label{sec:experiment}
In this section, we demonstrate that \sassha can indeed improve upon existing second-order methods available for standard deep learning tasks.
We also show that \sassha performs competitively to the first-order baseline methods.
Specifically, \sassha is compared to AdaHessian \citep{adahessian}, Sophia-H \citep{sophia}, Shampoo \cite{gupta2018shampoo}, SGD, AdamW \citep{loshchilov2018decoupled}, and SAM \citep{sam} on a diverse set of both vision and language tasks.
We emphasize that we perform an \emph{extensive} hyperparameter search to rigorously tune all optimizers and ensure fair comparisons.
We provide the details of experiment settings to reproduce our results in \cref{app:hypersearch}.
The code to reproduce all results reported in this work is made available for download at \url{https://github.com/LOG-postech/Sassha}.

\begin{table*}[ht!]
    \centering
    \caption{
    Language finetuning and pertraining results for various optimizers. For finetuning, \sassha achieves better results than AdamW and AdaHessian and compares competitively with Sophia-H. For pretraining, \sassha achieves the lowest perplexity among all optimizers.
    }
    \vskip 0.1in
    \resizebox{\linewidth}{!}{
        \begin{tabular}{lc}
            \toprule
             & \multicolumn{1}{c}{$\textbf{Pretrain} / $ GPT1-mini} \\
             \cmidrule(l{3pt}r{3pt}){2-2}
             & Wikitext-2 \\
             & \texttt{Perplexity}\\
            \midrule
            
            AdamW & $ 175.06 _{\textcolor{black!60}{\pm 0.19}} $ \\
            SAM $_{\text{AdamW}}$ & $ 158.06 _{\textcolor{black!60}{\pm 0.23}} $ \\
            AdaHessian & $ 407.69 _{\textcolor{black!60}{\pm 0.20}} $ \\
            Sophia-H & $ 157.60 _{\textcolor{black!60}{\pm 0.37}}$ \\
            
            \midrule 
            
            \rowcolor{green!20}
            \sassha &
            $ \textbf{122.40}_{\textcolor{black!60}{\pm 0.16}} $ \\
            
            \bottomrule
        \end{tabular}
        
        \begin{tabular}{|ccccccc}
            \toprule
                         \multicolumn{7}{|c}{ \textbf{Finetune} /  SqeezeBERT } \\
                         \cmidrule(l{3pt}r{3pt}){1-7}
                         SST-2 &  MRPC & STS-B & QQP & MNLI & QNLI & RTE \\
             \texttt{Acc} &  \texttt{Acc / F1}  & \texttt{S/P corr.} & \texttt{F1 / Acc} & \texttt{mat/m.mat} &  \texttt{Acc} &  \texttt{Acc} \\
            \midrule
            
            $ 90.29 _{\textcolor{black!60}{\pm 0.52}} $ 
            & $ 84.56 _{ \textcolor{black!60}{\pm 0.25} } $ / $ 88.99 _{\textcolor{black!60}{\pm 0.11}} $ 
            & $ 88.34 _{\textcolor{black!60}{\pm 0.15}} $ / $ 88.48 _{\textcolor{black!60}{\pm 0.20}} $ 
            & $ 89.92 _{\textcolor{black!60}{\pm 0.05}} $ / $ 86.58 _{\textcolor{black!60}{\pm 0.11}} $ 
            & $ 81.22 _{\textcolor{black!60}{\pm 0.07}} $ / $ 82.26 _{\textcolor{black!60}{\pm 0.05}} $ 
            & $ 89.93 _{\textcolor{black!60}{\pm 0.14}} $ 
            & $ 68.95 _{\textcolor{black!60}{\pm 0.72}} $  \\
    
            $ \textbf{90.52} _{\textcolor{black!60}{\pm 0.27}} $ 
            & $ 83.25 _{\textcolor{black!60}{\pm 2.79}} $ / $ 87.90 _{\textcolor{black!60}{\pm 2.21}} $ 
            & $ 88.38 _{\textcolor{black!60}{\pm 0.01}} $ / $ 88.79 _{\textcolor{black!60}{\pm 0.99}} $ 
            & $ 90.26 _{\textcolor{black!60}{\pm 0.28}} $ / $ 86.99 _{\textcolor{black!60}{\pm 0.31}} $ 
            & $ 81.56 _{\textcolor{black!60}{\pm 0.18}} $ / $ \textbf{82.46} _{\textcolor{black!60}{\pm 0.19}} $ 
            & $ \textbf{90.38} _{\textcolor{black!60}{\pm 0.05}} $ 
            & $ 68.83 _{\textcolor{black!60}{\pm 1.46}} $  \\
    
            $ 89.64 _{\textcolor{black!60}{\pm 0.13}} $ 
            & $ 79.74 _{\textcolor{black!60}{\pm 4.00}} $ / $ 85.26 _{\textcolor{black!60}{\pm 3.50}} $ 
            & $ 86.08 _{\textcolor{black!60}{\pm 4.04}} $ / $ 86.46 _{\textcolor{black!60}{\pm 4.06}} $ 
            & $ 90.37 _{\textcolor{black!60}{\pm 0.05}} $ / $ 87.07 _{\textcolor{black!60}{\pm 0.05}} $ 
            & $ 81.33 _{\textcolor{black!60}{\pm 0.17}} $ / $ 82.08 _{\textcolor{black!60}{\pm 0.02}} $ 
            & $ 89.94 _{\textcolor{black!60}{\pm 0.12}} $ 
            & $ 71.00 _{\textcolor{black!60}{\pm 1.04}} $ \\
            
            $ 90.44 _{\textcolor{black!60}{\pm 0.46}} $ 
            & $ 85.78 _{\textcolor{black!60}{\pm 1.07}} $ / $ 89.90 _{\textcolor{black!60}{\pm 0.82}} $ 
            & $ 88.17 _{\textcolor{black!60}{\pm 1.07}} $ / $ 88.53 _{\textcolor{black!60}{\pm 1.13}} $ 
            & $ 90.70 _{\textcolor{black!60}{\pm 0.04}} $ / $ 87.60 _{\textcolor{black!60}{\pm 0.06}} $ 
            & $ \textbf{81.77} _{\textcolor{black!60}{\pm 0.18}} $ / $ 82.36 _{\textcolor{black!60}{\pm 0.22}} $ 
            & $ 90.12_{\textcolor{black!60}{\pm 0.14}} $ 
            & $ 70.76 _{\textcolor{black!60}{\pm 1.44}} $  \\
            
            \midrule
            
            \rowcolor{green!20} 
            $ 90.44 _{\textcolor{black!60}{\pm 0.98}} $    &
            $ \textbf{86.28} _{\textcolor{black!60}{\pm 0.28}} $ / $ \textbf{90.13} _{\textcolor{black!60}{\pm 0.161}} $     &
            $ \textbf{88.72} _{\textcolor{black!60}{\pm 0.75}} $ / $ \textbf{89.10} _{\textcolor{black!60}{\pm 0.70}}  $     &
            $ \textbf{90.91} _{\textcolor{black!60}{\pm 0.06}} $ / $ \textbf{87.85}  _{\textcolor{black!60}{\pm 0.09}} $     &
            $ 81.61 _{\textcolor{black!60}{\pm 0.25}} $ / $ 81.71 _{\textcolor{black!60}{\pm 0.11}} $     &
            $ 89.85_{\textcolor{black!60}{\pm 0.20}} $    &
            $ \textbf{72.08} _{\textcolor{black!60}{\pm 0.55}} $  \\
            \bottomrule
        \end{tabular}
    }
    \label{tab:language}
\end{table*}

\subsection{Image Classification}
We first evaluate \sassha for image classification on CIFAR-10, CIFAR-100, and ImageNet.
We train various models of the ResNet family \citep{he2016deep,zagoruyko2016wide} and an efficient variant of Vision Transformer \citep{beyer2022better}.
We adhere to standard inception-style data augmentations during training instead of making use of advanced data augmentation techniques \citep{devries2017improved} or regularization methods \citep{gastaldi2017shake}.
Results are presented in \cref{tab:im_cls_results} and \cref{fig:im_cls_results}.

We begin by comparing the generalization performance of adaptive second-order methods to that of first-order methods.
Across all settings, adaptive second-order methods consistently exhibit lower accuracy than their first-order counterparts.
This observation aligns with previous studies indicating that second-order optimization often result in poorer generalization compared to first-order approaches.
In contrast, \sassha, benefiting from sharpness minimization, consistently demonstrates superior generalization performance, outperforming both first-order and second-order methods in every setting.
Particularly, \sassha is up to 4\% more effective than the best-performing adaptive or second-order methods (\eg, WRN-28-10, ViT-s-32).
In addition, \sassha continually surpasses SGD and AdamW by approximately 0.3\% to 3\%, even when these methods are trained for twice as many epochs.
Further details on these experiments are provided in \cref{app:comp_fo_fair}.

Interestingly, \sassha also outperforms SAM.
Since first-order methods typically exhibit superior generalization performance compared to second-order methods, it might be intuitive to expect SAM to surpass \sassha if the two are viewed merely as the outcomes of applying sharpness minimization to first-order and second-order methods, respectively.
However, the results conflict with this intuition.
We attribute this to the careful design choices made in \sassha, stabilizing Hessian approximation under sharpness minimization, so as to unleash the potential of the second-order method, leading to its outstanding performance.
As a support, we show that naively incorporating SAM into other second-order methods does not yield these favorable results in \cref{app:samsophia}.
We also make more comparisons with SAM in \cref{sec:sassha_vs_sam}.

\subsection{Language Modeling}
Recent studies have shown the potential of second-order methods for pretraining language models.
Here, we first evaluate how \sassha performs on this task.
Specifically, we train GPT1-mini, a scaled-down variant of GPT1 \citep{radford2019language}, on Wikitext-2 dataset \citep{merity2022pointer} using various methods including \sassha and compare their results (see the left of \cref{tab:language}).
Our results show that \sassha achieves the lowest perplexity among all methods including Sophia-H \citep{sophia}, a recent method that is designed specifically for language modeling tasks and sets state of the art, which highlights generality in addition to the numerical advantage of \sassha.
We further evaluate \sassha for GPT2 and provide results in \cref{app:gpt2}.

We also extend our evaluation to finetuning tasks.
Specifically, we finetune SqueezeBERT \citep{iandola2020squeezebert} for diverse tasks in the GLUE benchmark \citep{wang2018glue}.
The results are on the right side of \cref{tab:language}.
It shows that \sassha compares competitively to other second-order methods.
Notably, it also outperforms AdamW---often the method of choice for training language models---on nearly all tasks.

\subsection{Comparison to SAM}\label{sec:sassha_vs_sam}

So far, we have seen that \sassha outperforms second-order methods quite consistently on both vision and language tasks.
Interestingly, we also find that \sassha often improves upon SAM.
In particular, it appears that the gain is larger for the Transformer-based architectures, \ie, ViT results in \cref{tab:im_cls_results} or GPT/BERT results in \cref{tab:language}.

To further investigate these findings, we conducted additional experiments.
First, we allocate more training budgets to SAM to see whether it compares to \sassha.
The results are presented in \cref{tab:sam}.
We find that SAM still underperforms \sassha, even though it is given more budgets of training iterations over data or wall-clock time.
Furthermore, we also compare \sassha to more advanced variants of SAM including ASAM \citep{asam} and GSAM \citep{gsam}, showing that \sassha performs competitively even to these methods (\cref{app:samvariants_vs_sassha}).
Notably, however, these variants of SAM require a lot more hyperparameter tuning to be compared.

We suspect that this may be due to the robustness of \sassha to the block heterogeneity inherent in Transformer architectures, where the Hessian spectrum varies significantly across different blocks.
This characteristic is known to make SGD perform worse than adaptive methods like Adam on Transformer-based models \citep{zhang2024why}.
Since \sassha leverages second-order information via preconditioning gradients, it has the potential to address the ill-conditioned nature of Transformers more effectively than SAM with first-order methods.

\begin{table}[t!]
    \centering
    \bgroup
    \def\arraystretch{1.2}
    \caption{
    Comparison between \sassha and SAM with more training budgets for the ViT-s-32 / ImageNet workload.
    }
    \label{tab:sam}
    \vskip 0.1in
    \resizebox{0.77\linewidth}{!}{
        \centering
        \begin{tabular}{lccc}  
        \toprule
        &\multicolumn{1}{c}{ Epoch } & Time (\texttt{s}) &\multicolumn{1}{c}{ Accuracy (\%) } \\ \midrule
        SAM$_\text{ SGD}$      &  180  & 220,852 & $  65.403 _{\textcolor{black!60}{\pm 0.63}}$ \\
        SAM$_\text{ AdamW}$    &  180 & 234,374 & $68.706 _{\textcolor{black!60}{\pm 0.16}}$ \\
        \midrule
        \rowcolor{green!20}\sassha        &  \textbf{90}  & \textbf{123,948} &  $\textbf{69.195} _{\textcolor{black!60}{\pm 0.30} } $  \\
        \bottomrule
        \end{tabular}
    }
    \egroup
\end{table}

\section{Further Analysis}
\label{sec:ablation}

\subsection{Robustness}
\label{sec:robustness}

Noisily labeled training data can critically degrade generalization performance \citep{natarajan2013learning}.
To evaluate how \sassha generalizes under these practical conditions, we randomly corrupt certain fractions of the training data and compare the validation performances between different methods.
The results show that \sassha outperforms other methods across all noise levels with minimal accuracy degradation (\cref{tab:noise_label}).
Additionally, we also observe the same trend on CIFAR-10 (\cref{tab:noise_label_sassha}).

Interestingly, \sassha surpasses SAM \citep{sam}, which is known to be one of the most robust techniques against label noise \citep{baek2024why}. 
We hypothesize that its robustness stems from the complementary benefits of the sharpness-minimization scheme and second-order methods.
Specifically, SAM enhances robustness by adversarially perturbing the parameters and giving more importance to clean data during optimization, making the model more resistant to label noise \citep{sam, baek2024why}.
Also, recent research indicates that second-order methods are robust to label noise due to preconditioning that reduces the variance in the population risk \citep{amari2021when}.

\begin{table}[t!]
    \centering
    \caption{
    Validation accuracy measured for ResNet-32/CIFAR-100 at different levels of noise.
    \sassha shows the best robustness.
    }
    \label{tab:noise_label}
    \vskip 0.1in
    \resizebox{\linewidth}{!}{%
    \begin{tabular}{lccccc}
        \toprule
        & \multicolumn{4}{c}{Noise level} \\ 
        \cmidrule(l{3pt}r{3pt}){2-5}  
        Method & {0\%} & {20\%} & {40\%} & {60\%} \\ 
        \midrule
        SGD                 &
        $69.32_{\textcolor{black!60}{\pm 0.19}}$
        & $62.18_{\textcolor{black!60}{\pm 0.06}}$ 
        & $55.78_{\textcolor{black!60}{\pm 0.55}}$  
        & $45.53_{\textcolor{black!60}{\pm 0.78}}$ \\ 
        
        SAM $_{\text{SGD}}$ & 
        $71.99_{\textcolor{black!60}{\pm 0.20}}$
        & $65.53_{\textcolor{black!60}{\pm 0.11}}$  
        & $ 61.20_{\textcolor{black!60}{\pm 0.17}}$  
        & $ 51.93_{\textcolor{black!60}{\pm 0.47}}$ \\ 
        
        AdaHessian         &
        $68.06_{\textcolor{black!60}{\pm 0.22}}$
        & $63.06_{\textcolor{black!60}{\pm 0.25}}$  
        & $58.37_{\textcolor{black!60}{\pm 0.13}}$  
        & $46.02_{\textcolor{black!60}{\pm 1.96}}$  \\

        Sophia-H           &
        $67.76_{\textcolor{black!60}{\pm 0.37}}$
        & $62.34_{\textcolor{black!60}{\pm 0.47}}$  
        & $56.54_{\textcolor{black!60}{\pm 0.28}}$  
        & $45.37_{\textcolor{black!60}{\pm 0.27}}$  \\
        
        Shampoo           & 
        $64.08_{\textcolor{black!60}{\pm 0.46}}$
        & $58.85_{\textcolor{black!60}{\pm 0.66}}$ & $ 53.82 _{\textcolor{black!60}{\pm 0.71}}$  
        & $ 42.91_{\textcolor{black!60}{\pm 0.99}}$ \\
        
        \midrule
        
        \rowcolor{green!20} \sassha         & 
        $ \textbf{72.14}_{\textcolor{black!60}{\pm 0.16}}    $
        & $\textbf{66.78}_{\textcolor{black!60}{\pm 0.47}}   $  
        & $ \textbf{ 61.97}_{\textcolor{black!60}{\pm 0.27}} $  
        & $\textbf{ 53.98}_{\textcolor{black!60}{\pm 0.57}}  $ \\
        
        \bottomrule  
    \end{tabular}}
\end{table}

\subsection{Stability} \label{sec:sqrt_ablation}
To show the effect of the square-root function on stabilizing the training process, we run \sassha without the square-root (\texttt{No-Sqrt}), repeatedly for multiple times with different random seeds.
As a result, we find that the training diverges most of the time.
A failure case is depicted in \cref{fig:sqrt_ablation}.

At first, we find that the level of training loss for \texttt{No-Sqrt} is much higher than that of \sassha, and also, it spikes up around step $200$ (\cref{fig:sqrt_ablation_train_loss}).
To look into it further, we also measure the update sizes along the trajectory (\cref{fig:sqrt_ablation_update-size}).
The results show that it matches well with the loss curves, suggesting that the training failure is somehow due to taking too large steps.

It turns out that this problem stems from the preconditioning matrix $D$ being too small;
\ie, the distribution of diagonal entries in the preconditioning matrix gradually shifts toward zero values (\cref{fig:sqrt_ablation_ridgeline-plot});
as a result, $D^{-1}$ becomes too large, creating large steps.
This progressive increase in near-zero diagonal Hessian entries is precisely due to the sharpness minimization scheme that we introduced; it penalizes the Hessian eigenspectrum to yield flat solutions, yet it could also make training unstable if taken naively.
By including square-root, the preconditioner are less situated near zero, effectively suppressing the risk of large updates, thereby stabilizing the training process.
We validate this further by showing its superiority to other alternatives including damping and clipping in \cref{app:sqrt_alternatives}.

We also provide an ablation analysis for the absolute-value function in \cref{sec:abs_ablation}, which demonstrates that it increases the stability of \sassha in tandem with square-root.

\begin{figure}[t!]
\resizebox{\linewidth}{!}{%
    \centering
    \begin{subfigure}{0.4\linewidth}
        \centering
        \includegraphics[width=\linewidth]{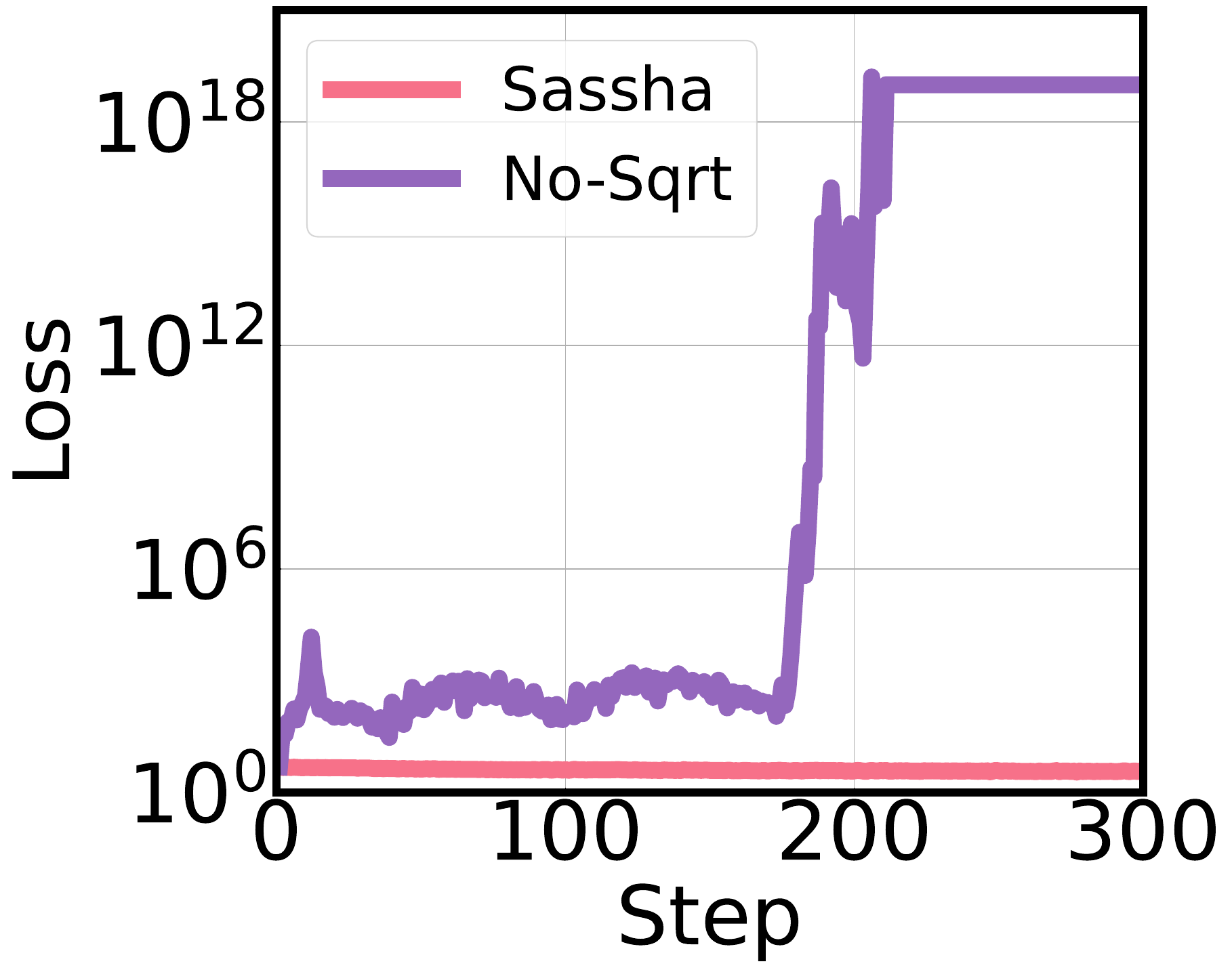}
        \caption{Train loss}
        \label{fig:sqrt_ablation_train_loss}
    \end{subfigure}
    \begin{subfigure}{0.4\linewidth}
        \centering
        \includegraphics[width=\linewidth,trim={0.5em 0.5em 1em 0},clip]{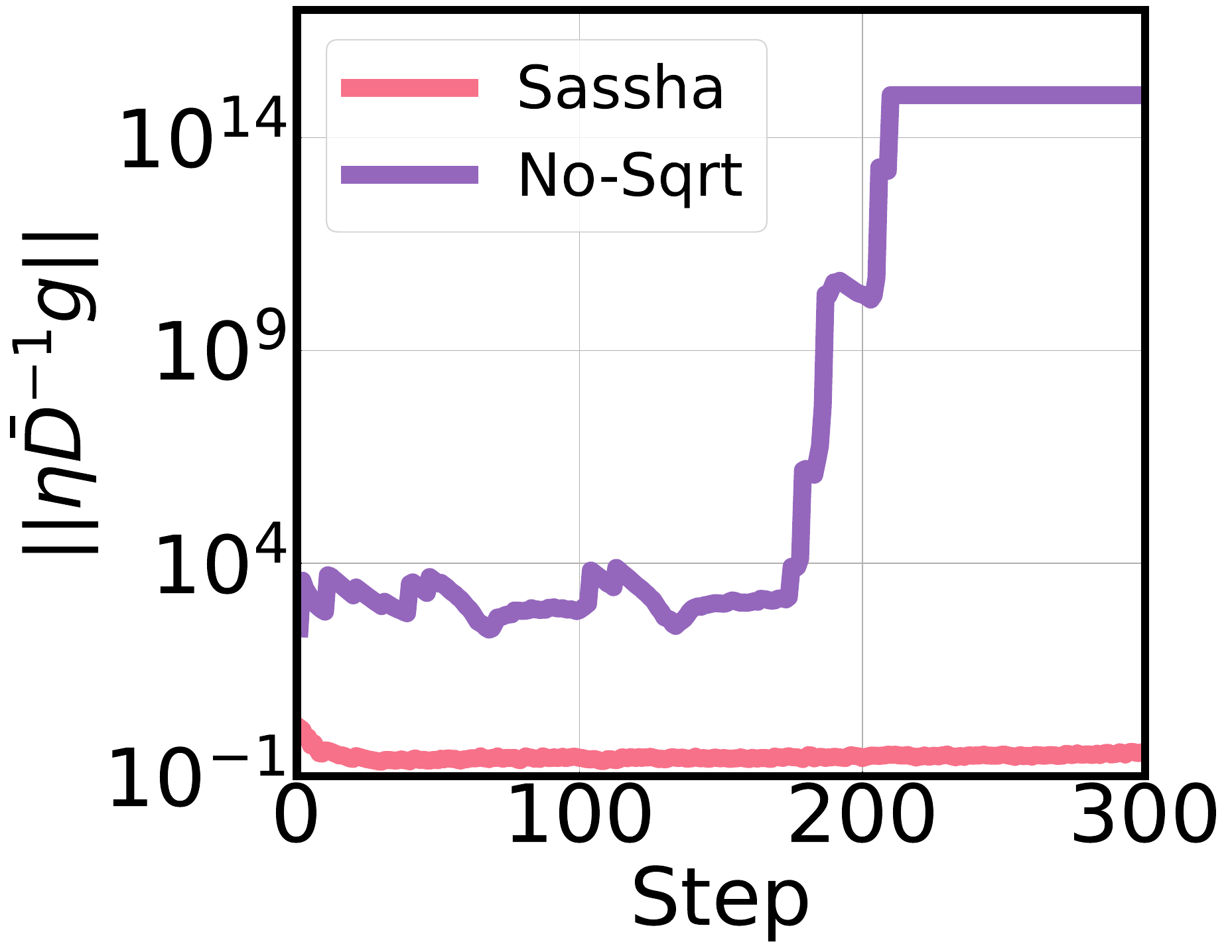}
        \caption{Update size}
        \label{fig:sqrt_ablation_update-size}
    \end{subfigure}
    \hfill
    \begin{subfigure}{0.4\linewidth}
        \centering
        \includegraphics[width=\linewidth,trim={0 1em 0 1.2em},clip]{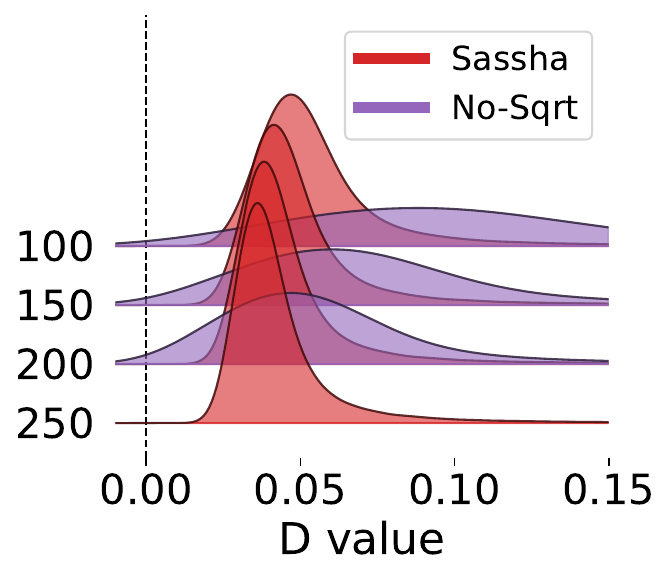}\\
        
        \caption{$D$ distribution}
        \label{fig:sqrt_ablation_ridgeline-plot}
    \end{subfigure}
    }
    \caption{
    Effects of square-root measured for ResNet-32/CIFAR-100;
    $D$ is set to be either $|\widehat{H}|^{1/2}$ for \sassha or $|\widehat{H}|$ for \texttt{No-Sqrt}.
    Sharpness minimization drives the diagonal Hessian entries move towards zero, causing divergence.
    The square-root in \sassha helps counteract this effect, stabilizing the training process.
    }
    \label{fig:sqrt_ablation}
\end{figure}

\subsection{Efficiency} \label{sec:emp_lazy_hess}

Here we show the effectiveness of lazy Hessian updates in \sassha.
The results are shown in \cref{fig:diagonal_hessian_comp}.
At first, we see that \sassha maintains its performance even at $k=100$, indicating that it is extremely robust to lazy Hessian updates (\cref{fig:lazy_results}).
We also measure the difference between the current and previous Hessians to validate lazy Hessian updates more directly (\cref{fig:lazy_hess_diff}).
The result shows that \sassha keeps the changes in Hessian to be small, and much smaller than other methods, indicating its advantage of robust reuse, and hence, computational efficiency.

We attribute this robustness to the sharpness minimization scheme incorporated in \sassha, which can potentially bias optimization toward the region of low curvature sensitivity.
To verify, we define local Hessian sensitivity as follows:
\begin{equation}\label{eq:diff_hessian(2)}
    \max_{\delta \sim \mathcal{N}(0, 1)}\left\|\widehat{H}\left(x+\rho\frac{\delta}{\|\delta\|_2}\right) - \widehat{H}(x)\right\|_F
\end{equation}
\ie, it measures the maximum change in Hessian induced from normalized random perturbations.
A smaller Hessian sensitivity would suggest reduced variability in the loss curvature, leading to greater relevance of the current Hessian for subsequent optimization steps.
We find that \sassha is far less sensitive compared to other methods (\cref{fig:lazy_perturbed}).

\begin{figure}[t!] 
    \centering
    \begin{minipage}{\linewidth}
        \centering
        \hspace{1.2em}
        \includegraphics[width=0.9\linewidth, trim={0 0 0 0},clip]{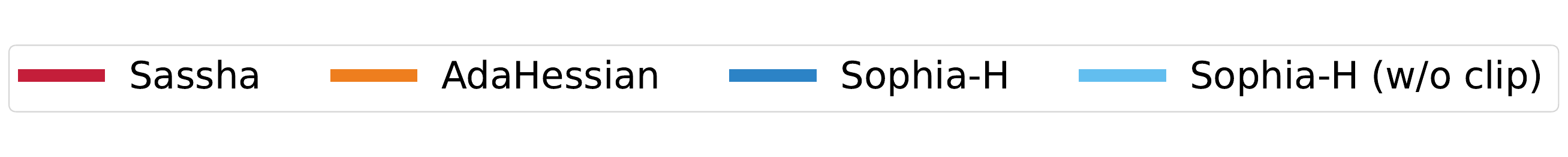} 
        \vspace{-0.6em}
    \end{minipage}

    \resizebox{\linewidth}{!}{
        \begin{subfigure}{0.288\linewidth}
        \centering%
            \includegraphics[width=\textwidth,trim={0 -1em 1.2em 0.5em},clip]{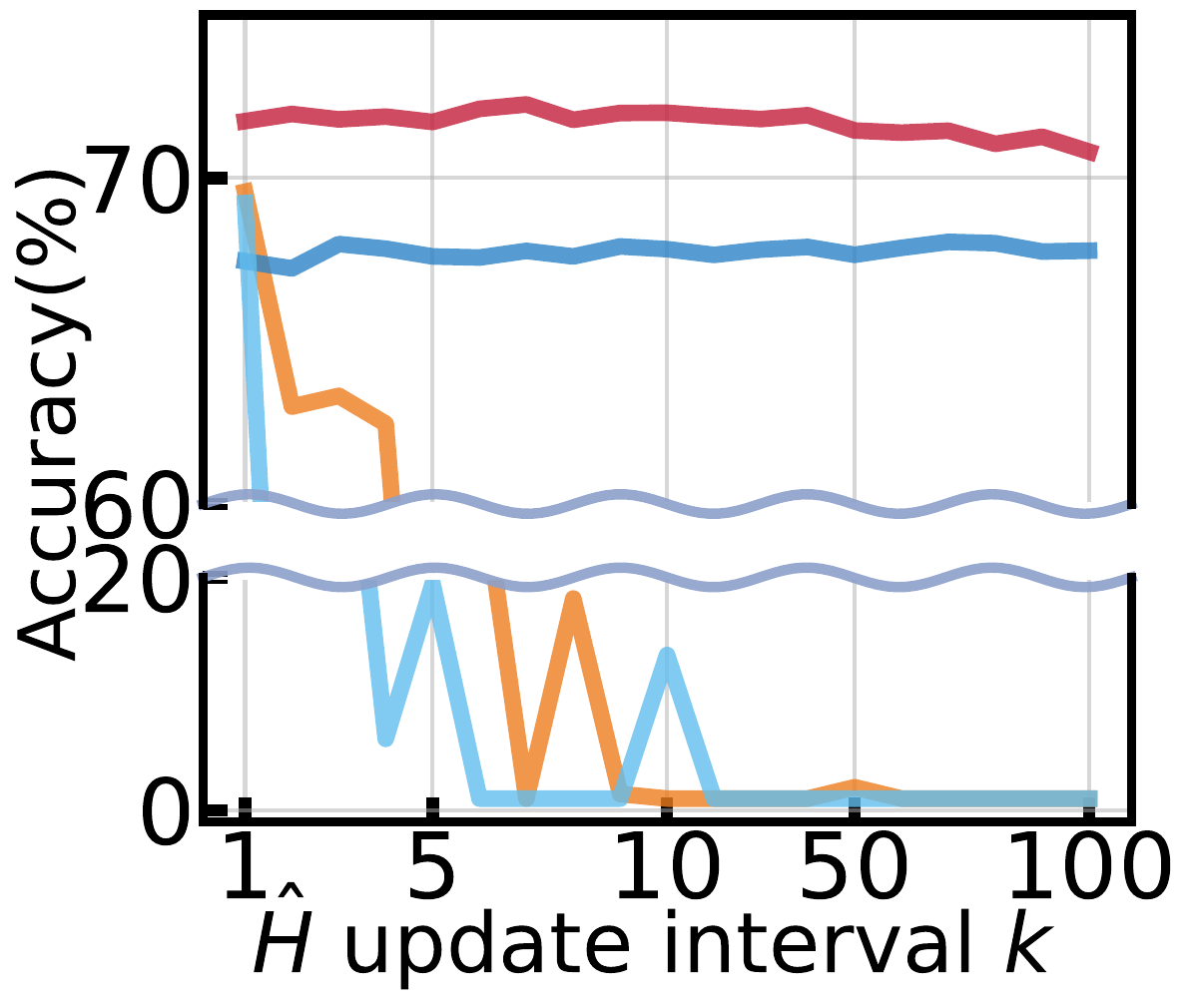}
            \caption{Lazy Hessian}
            \label{fig:lazy_results}
        \end{subfigure}
        \begin{subfigure}{0.305\linewidth}
            \centering
            \includegraphics[width=\textwidth,trim={2em 1.4em 0 0.1em},clip]{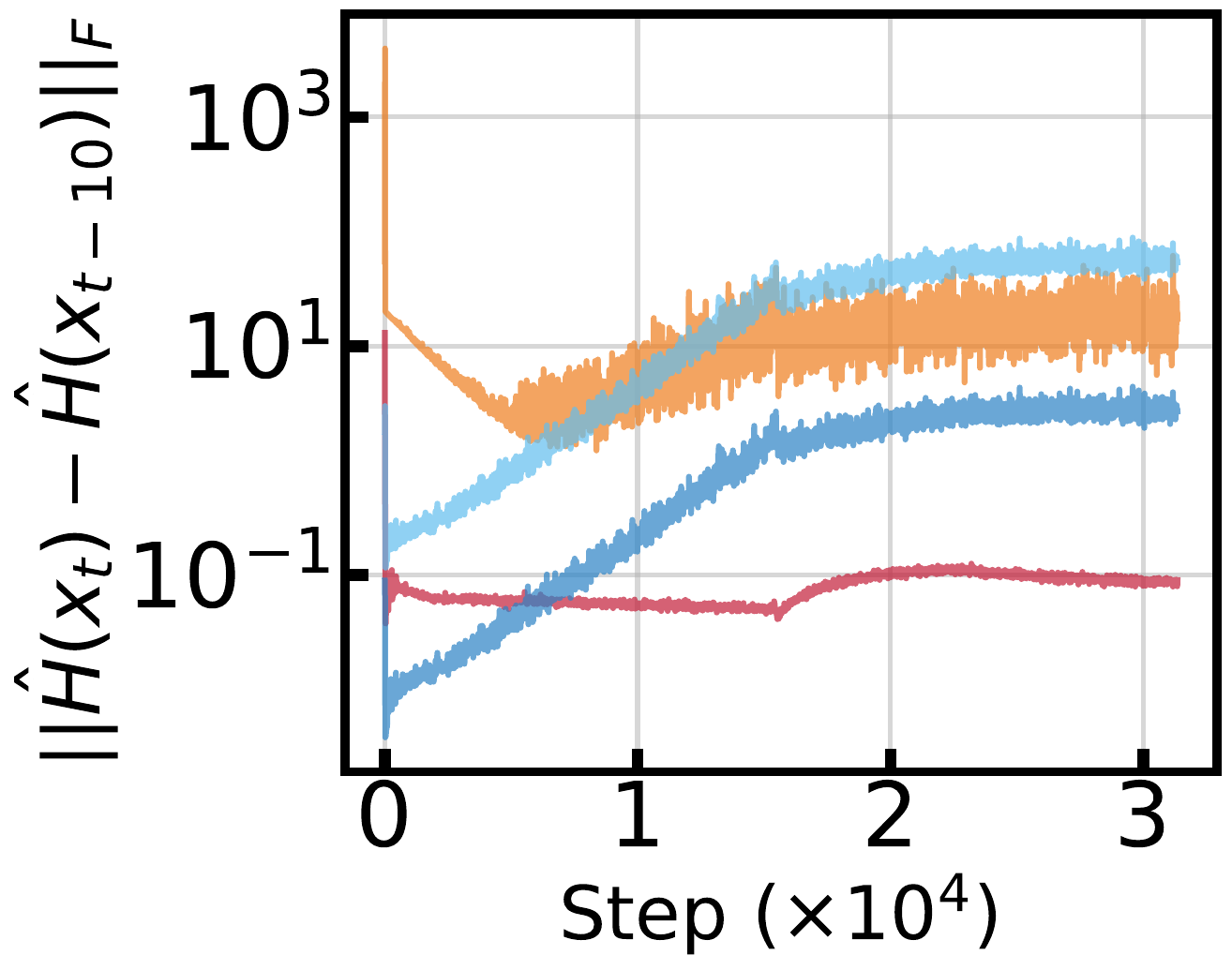}
            \caption{$\widehat{H}$ change}
            \label{fig:lazy_hess_diff}
        \end{subfigure}%
        \begin{subfigure}{0.32\linewidth}
            \centering
            \includegraphics[width=\linewidth,trim={0 -0.7em 0 0.3em},clip]{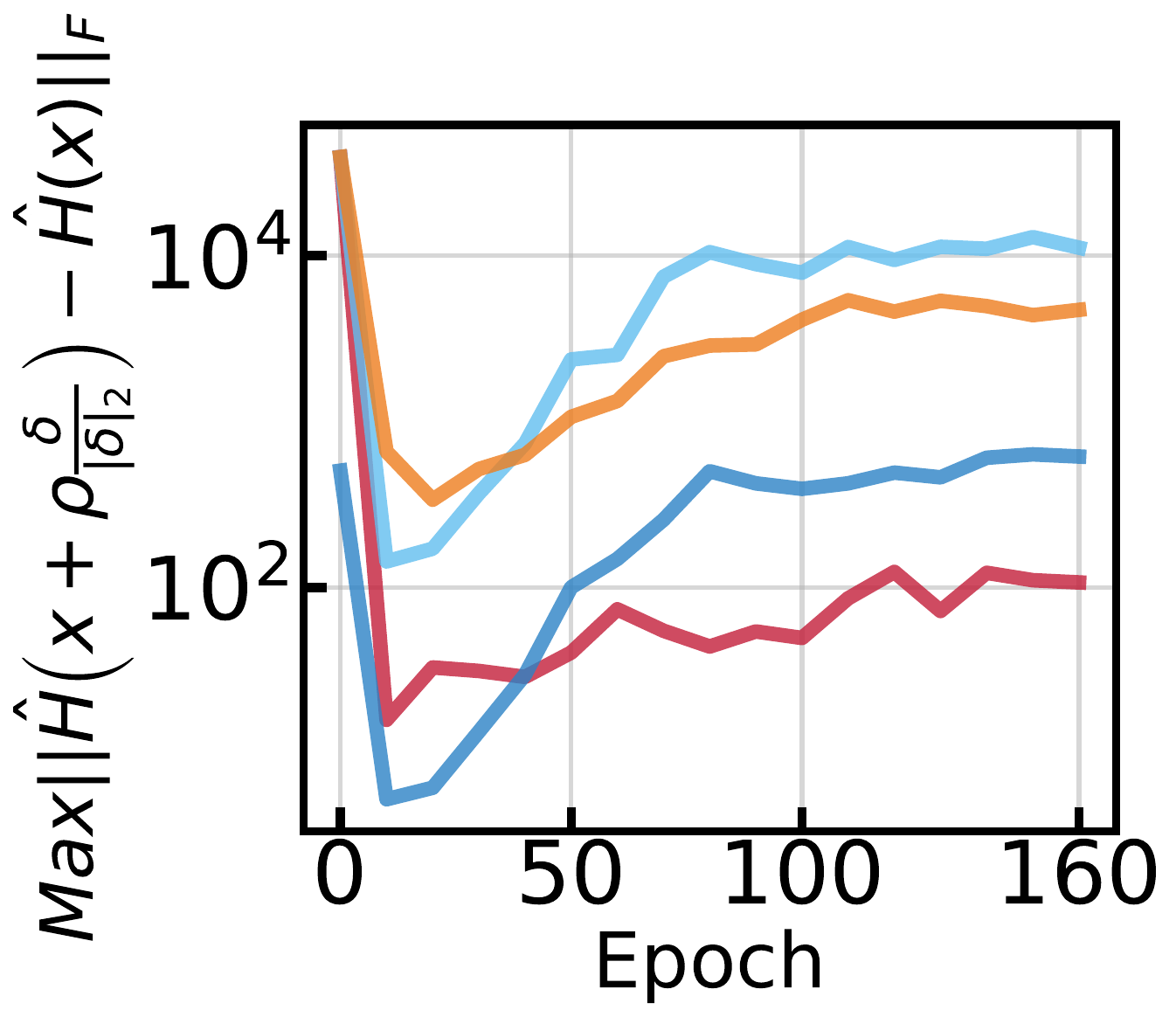}
            \caption{Local sensitivity}
            \label{fig:lazy_perturbed}
        \end{subfigure}
    }
    \caption{
    Effect of lazy Hessian for ResNet-32/CIFAR-100.
    \sassha stays within the region where the Hessian varies small.
    }
    \label{fig:diagonal_hessian_comp} %
\end{figure}

\subsection{Cost}\label{sec:cost}

Second-order methods can be highly costly.
In this section, we discuss the computational cost of \sassha and reveal its competitiveness to other methods.

\sassha requires one gradient computation (\texttt{GC}) in the sharpness minimization step, one Hessian-vector product (\texttt{HVP}) for diagonal Hessian computation, and an additional \texttt{GC} in the descent step.
That is, a total of $2$\texttt{GC}s and $1$\texttt{HVP} are required.
However, with lazy Hessian updates, the number of \texttt{HVP}s reduces drastically to $ 1 / k $.
With $ k = 10 $ as the default value used in this work, this scales down to $0.1$\texttt{HVP}s.

\begin{table}[t!]
    \centering
    \caption{
    Average wall-clock time per epoch (\texttt{s}) and the theoretical cost of different methods.
    \sassha can be an effective alternative to existing methods for its enhanced generalization performance.
    } 
    \vskip 0.1in
    \resizebox{\linewidth}{!}{%
    \begin{tabular}{l|r|r|r|r|r|r|r}
    \toprule
    \multirow{2.5}{*}{Method}
    & \multicolumn{4}{c|}{Cost}
    & \multicolumn{1}{c|}{CIFAR10} 
    & \multicolumn{1}{c|}{CIFAR100} 
    & \multicolumn{1}{c}{ImageNet} \\ 
    \cmidrule(l{3pt}r{3pt}){2-5}
    \cmidrule(l{3pt}r{3pt}){6-8} 
     & \multicolumn{1}{c|}{Descent}
     & \multicolumn{1}{c|}{Sharpness}
     & \multicolumn{1}{c|}{Hessian}
     & \multicolumn{1}{c|}{Total}
     & \multicolumn{1}{c|}{ResNet32} 
     & \multicolumn{1}{c|}{WRN28-10}  
     & \multicolumn{1}{c}{ViT-small} \\ 
    \midrule
    AdamW  & 1 \texttt{GC}
    & 0 \texttt{GC}
    & 0 \texttt{HVP}
    & 1 \texttt{GC} & 5.03 & 59.29 & 976.56 \\
    
    SAM    & 1 \texttt{GC}
    & 1 \texttt{GC}
    & 0 \texttt{HVP}
    & 2 \texttt{GC} & 9.16 & 118.46  &  1302.08\\ 
    
    AdaHessian &  1 \texttt{GC}
    & 0 \texttt{GC}
    & 1 \texttt{HVP}
    & 4 \texttt{GC} & 33.75 & 296.63  & 2489.07 \\ \midrule
    
    \rowcolor{green!20} \sassha & 1 \texttt{GC}
    & 1 \texttt{GC}
    & 0.1 \texttt{HVP}
    & 2.3 \texttt{GC} & 12.00   & 142.06   & 1377.20 \\ 
    
    \rowcolor{green!20} \msassha & 1 \texttt{GC}
    & 0 \texttt{GC}
    & 0.1 \texttt{HVP}
    & 1.3 \texttt{GC}  & 8.91   & 84.12 & 1065.40 \\ 
    \bottomrule
    \end{tabular}
    }
    \label{tab:costs}
\end{table}

It turns out that this is critical to the utility of \sassha, because $1$\texttt{HVP} is known to take about $ \times 3 $ the computation time of $1$\texttt{GC} in practice \citep{dagrou2024how}.
Compared to conventional second-order methods ($1$\texttt{GC} $+$ $1$\texttt{HVP} $\simeq$ $4$\texttt{GC}s), the cost of \sassha can roughly be a half of that ($2.3$\texttt{GC}s).
It is also comparable to standard SAM variants ($2$\texttt{GC}s).

Furthermore, we can leverage a momentum of gradients in the perturbation step to reduce the cost.
This variant \msassha requires only $1.3$\texttt{GC}s with minimal decrease in performance.
Notably, \msassha still outperforms standard first-order methods like SGD and AdamW (\cref{app:msassha}).

To verify, we measure the average wall-clock times and present the results in \cref{tab:costs}.
First, one can see that the theoretical cost is reflected well on the actual cost;
\ie, the time measurements scales proportionally roughly well with respect to the total cost.
More importantly, this result indicates the potential of \sassha for performance-critical applications.
Considering its well-balanced cost, and that it has been challenging to employ second-order methods efficiently for large-scale tasks without sacrificing performance, \sassha can be a reasonable addition to the lineup.

\section{Conclusion}

In this work, we focus on addressing the issue of poor generalization in approximate second-order methods.
Our empirical analysis indicates that this limitation may be attributed to their tendency to converge to sharp minima, which are known to correlate with weaker generalization performance.
To this end, we propose a new method called \sassha that stably minimizes sharpness within the framework of second-order optimization.
\sassha converges to flat solutions and achieves state-of-the-art performance within this class.
\sassha also performs competitively to widely-used first-order, adaptive, and sharpness-aware methods.
\sassha achieves this efficiently through lazy Hessian updates, to which it is robust, and does so without requiring extra hyperparameter tuning.
Moreover, \sassha exhibits strong resilience to label noise.
All of these are rigorously assessed with extensive experiments.

Nonetheless, there are still many limitations to be addressed to further improve this work.
Some examples may include, but are not limited to, extending experiments to various models and different data of extreme scales, as well as developing theoretical properties such as convergence rate, generalization bound, and implicit bias, all to more rigorously confirm the value of \sassha.
Seeing it as an exciting opportunity, we plan to investigate further in future work.
\section*{Acknowledgement}
This work was partly supported by the Institute of Information \& communications Technology
Planning \& Evaluation (IITP) grant funded by the Korean government (MSIT) (IITP-2019-0-01906, 
Artificial Intelligence Graduate School Program (POSTECH), RS-2024-00338140, Development of learning and utilization technology to reflect sustainability of generative language models and up-to-dateness over time), and the National Research
Foundation of Korea (NRF) grant funded by the Korean government (MSIT) (
NRF-2022R1F1A1064569, RS-2023-00210466, RS-2023-00265444).
\section*{Impact statement}
This paper contributes to advancing second-order optimization, with potential implications for both theoretical insights and practical applications. While our work does not present immediate concerns warranting specific emphasis, we recognize that progress in this field may have broader societal impacts. We remain committed to engaging in discussions on the broader implications of our research should the need arise in the future.


\bibliography{example_paper}

\begin{thebibliography}{76}
\providecommand{\natexlab}[1]{#1}
\providecommand{\url}[1]{\texttt{#1}}
\expandafter\ifx\csname urlstyle\endcsname\relax
  \providecommand{\doi}[1]{doi: #1}\else
  \providecommand{\doi}{doi: \begingroup \urlstyle{rm}\Url}\fi

\bibitem[Agarwala \& Dauphin(2023)Agarwala and Dauphin]{agarwala2023sam}
Agarwala, A. and Dauphin, Y.
\newblock Sam operates far from home: eigenvalue regularization as a dynamical phenomenon.
\newblock \emph{ICML}, 2023.

\bibitem[Amari et~al.(2000)Amari, Park, and Fukumizu]{amari2000adaptive}
Amari, S.-i., Park, H., and Fukumizu, K.
\newblock Adaptive method of realizing natural gradient learning for multilayer perceptrons.
\newblock \emph{Neural computation}, 2000.

\bibitem[Amari et~al.(2021)Amari, Ba, Grosse, Li, Nitanda, Suzuki, Wu, and Xu]{amari2021when}
Amari, S.-i., Ba, J., Grosse, R.~B., Li, X., Nitanda, A., Suzuki, T., Wu, D., and Xu, J.
\newblock When does preconditioning help or hurt generalization?
\newblock \emph{ICLR}, 2021.

\bibitem[Baek et~al.(2024)Baek, Kolter, and Raghunathan]{baek2024why}
Baek, C., Kolter, J.~Z., and Raghunathan, A.
\newblock Why is {SAM} robust to label noise?
\newblock \emph{ICLR}, 2024.

\bibitem[Bahri et~al.(2022)Bahri, Mobahi, and Tay]{bahri2022sharpness}
Bahri, D., Mobahi, H., and Tay, Y.
\newblock Sharpness-aware minimization improves language model generalization.
\newblock \emph{ACL}, 2022.

\bibitem[Becker et~al.(2024)Becker, Altrock, and Risse]{becker2024momentum}
Becker, M., Altrock, F., and Risse, B.
\newblock Momentum-sam: Sharpness aware minimization without computational overhead.
\newblock \emph{arXiv}, 2024.

\bibitem[Becker et~al.(1988)Becker, Le~Cun, et~al.]{becker1988improving}
Becker, S., Le~Cun, Y., et~al.
\newblock Improving the convergence of back-propagation learning with second order methods.
\newblock \emph{CMSS}, 1988.

\bibitem[Beyer et~al.(2022)Beyer, Zhai, and Kolesnikov]{beyer2022better}
Beyer, L., Zhai, X., and Kolesnikov, A.
\newblock Better plain vit baselines for imagenet-1k.
\newblock \emph{arXiv}, 2022.

\bibitem[Botev et~al.(2017)Botev, Ritter, and Barber]{botev2017practical}
Botev, A., Ritter, H., and Barber, D.
\newblock Practical gauss-newton optimisation for deep learning.
\newblock \emph{ICML}, 2017.

\bibitem[Bottou et~al.(2018)Bottou, Curtis, and Nocedal]{bottou}
Bottou, L., Curtis, F.~E., and Nocedal, J.
\newblock Optimization methods for large-scale machine learning.
\newblock \emph{SIAM Review}, 2018.

\bibitem[Byrd et~al.(2016)Byrd, Hansen, Nocedal, and Singer]{byrd2016stochastic}
Byrd, R.~H., Hansen, S.~L., Nocedal, J., and Singer, Y.
\newblock A stochastic quasi-newton method for large-scale optimization.
\newblock \emph{SIAM Journal on Optimization}, 2016.

\bibitem[Chaudhari et~al.(2017)Chaudhari, Choromanska, Soatto, LeCun, Baldassi, Borgs, Chayes, Sagun, and Zecchina]{chaudhari2017entropy}
Chaudhari, P., Choromanska, A., Soatto, S., LeCun, Y., Baldassi, C., Borgs, C., Chayes, J., Sagun, L., and Zecchina, R.
\newblock Entropy-{SGD}: Biasing gradient descent into wide valleys.
\newblock \emph{ICLR}, 2017.

\bibitem[Chen et~al.(2022)Chen, Hsieh, and Gong]{chenvision}
Chen, X., Hsieh, C.-J., and Gong, B.
\newblock When vision transformers outperform resnets without pre-training or strong data augmentations.
\newblock \emph{ICLR}, 2022.

\bibitem[Dagr{\'e}ou et~al.(2024)Dagr{\'e}ou, Ablin, Vaiter, and Moreau]{dagrou2024how}
Dagr{\'e}ou, M., Ablin, P., Vaiter, S., and Moreau, T.
\newblock How to compute hessian-vector products?
\newblock \emph{The Third Blogpost Track at ICLR}, 2024.

\bibitem[Dauphin et~al.(2015)Dauphin, De~Vries, and Bengio]{dauphin2015equilibrated}
Dauphin, Y., De~Vries, H., and Bengio, Y.
\newblock Equilibrated adaptive learning rates for non-convex optimization.
\newblock \emph{NeurIPS}, 2015.

\bibitem[Demeniconi \& Chawla(2020)Demeniconi and Chawla]{doi:10.1137/1.9781611976236}
Demeniconi, C. and Chawla, N.
\newblock Second-order optimization for non-convex machine learning: an empirical study.
\newblock \emph{Society for Industrial and Applied Mathematics}, 2020.

\bibitem[DeVries \& Taylor(2017)DeVries and Taylor]{devries2017improved}
DeVries, T. and Taylor, G.~W.
\newblock Improved regularization of convolutional neural networks with cutout.
\newblock \emph{arXiv}, 2017.

\bibitem[Dinh et~al.(2017)Dinh, Pascanu, Bengio, and Bengio]{dinh17b}
Dinh, L., Pascanu, R., Bengio, S., and Bengio, Y.
\newblock Sharp minima can generalize for deep nets.
\newblock \emph{ICML}, 2017.

\bibitem[Doikov et~al.(2023)Doikov, Chayti, and Jaggi]{lazyhessian}
Doikov, N., Chayti, E.~M., and Jaggi, M.
\newblock Second-order optimization with lazy hessians.
\newblock \emph{ICML}, 2023.

\bibitem[Du et~al.(2022{\natexlab{a}})Du, Yan, Feng, Zhou, Zhen, Goh, and Tan]{esam}
Du, J., Yan, H., Feng, J., Zhou, J.~T., Zhen, L., Goh, R. S.~M., and Tan, V.
\newblock Efficient sharpness-aware minimization for improved training of neural networks.
\newblock \emph{ICLR}, 2022{\natexlab{a}}.

\bibitem[Du et~al.(2022{\natexlab{b}})Du, Zhou, Feng, Tan, and Zhou]{saf}
Du, J., Zhou, D., Feng, J., Tan, V., and Zhou, J.~T.
\newblock Sharpness-aware training for free.
\newblock \emph{NeurIPS}, 2022{\natexlab{b}}.

\bibitem[Duchi et~al.(2011)Duchi, Hazan, and Singer]{duchi2011adaptive}
Duchi, J., Hazan, E., and Singer, Y.
\newblock Adaptive subgradient methods for online learning and stochastic optimization.
\newblock \emph{JMLR}, 2011.

\bibitem[Dziugaite \& Roy(2017)Dziugaite and Roy]{DR17}
Dziugaite, G.~K. and Roy, D.~M.
\newblock Computing nonvacuous generalization bounds for deep (stochastic) neural networks with many more parameters than training data.
\newblock \emph{UAI}, 2017.

\bibitem[Foret et~al.(2021)Foret, Kleiner, Mobahi, and Neyshabur]{sam}
Foret, P., Kleiner, A., Mobahi, H., and Neyshabur, B.
\newblock Sharpness-aware minimization for efficiently improving generalization.
\newblock \emph{ICLR}, 2021.

\bibitem[Gao et~al.(2020)Gao, Biderman, Black, Golding, Hoppe, Foster, Phang, He, Thite, Nabeshima, Presser, and Leahy]{pile}
Gao, L., Biderman, S., Black, S., Golding, L., Hoppe, T., Foster, C., Phang, J., He, H., Thite, A., Nabeshima, N., Presser, S., and Leahy, C.
\newblock The {P}ile: An 800gb dataset of diverse text for language modeling.
\newblock \emph{arXiv}, 2020.

\bibitem[Gastaldi(2017)]{gastaldi2017shake}
Gastaldi, X.
\newblock Shake-shake regularization.
\newblock \emph{arXiv}, 2017.

\bibitem[Goldfarb et~al.(2020)Goldfarb, Ren, and Bahamou]{goldfarb2020practical}
Goldfarb, D., Ren, Y., and Bahamou, A.
\newblock Practical quasi-newton methods for training deep neural networks.
\newblock \emph{NeurIPS}, 2020.

\bibitem[Gomes et~al.(2024)Gomes, Zhang, Belilovsky, Wolf, and Hosseini]{gomes2024adafisher}
Gomes, D.~M., Zhang, Y., Belilovsky, E., Wolf, G., and Hosseini, M.~S.
\newblock Adafisher: Adaptive second order optimization via fisher information.
\newblock \emph{arXiv}, 2024.

\bibitem[Gower et~al.(2016)Gower, Goldfarb, and Richt{\'a}rik]{gower2016stochastic}
Gower, R., Goldfarb, D., and Richt{\'a}rik, P.
\newblock Stochastic block bfgs: Squeezing more curvature out of data.
\newblock \emph{ICML}, 2016.

\bibitem[Gupta et~al.(2018)Gupta, Koren, and Singer]{gupta2018shampoo}
Gupta, V., Koren, T., and Singer, Y.
\newblock Shampoo: Preconditioned stochastic tensor optimization.
\newblock In \emph{ICLR}, 2018.

\bibitem[Hardt et~al.(2016)Hardt, Recht, and Singer]{hardt2016train}
Hardt, M., Recht, B., and Singer, Y.
\newblock Train faster, generalize better: Stability of stochastic gradient descent.
\newblock \emph{ICML}, 2016.

\bibitem[He et~al.(2016)He, Zhang, Ren, and Sun]{he2016deep}
He, K., Zhang, X., Ren, S., and Sun, J.
\newblock Deep residual learning for image recognition.
\newblock \emph{CVPR}, 2016.

\bibitem[Hinton et~al.(2012)Hinton, Srivastava, and Swersky]{hinton2012neural}
Hinton, G., Srivastava, N., and Swersky, K.
\newblock Neural networks for machine learning lecture 6a overview of mini-batch gradient descent.
\newblock \emph{Coursera Lecture slides https://class. coursera. org/neuralnets-2012-001/lecture}, 2012.

\bibitem[Hochreiter \& Schmidhuber(1994)Hochreiter and Schmidhuber]{NIPS1994_Hochreiter}
Hochreiter, S. and Schmidhuber, J.
\newblock Simplifying neural nets by discovering flat minima.
\newblock \emph{NeurIPS}, 1994.

\bibitem[Hochreiter \& Schmidhuber(1997)Hochreiter and Schmidhuber]{Hochreiter1997}
Hochreiter, S. and Schmidhuber, J.
\newblock Flat minima.
\newblock \emph{Neural Computation}, 1997.

\bibitem[Hutchinson(1989)]{hutchinson}
Hutchinson, M.
\newblock A stochastic estimator of the trace of the influence matrix for laplacian smoothing splines.
\newblock \emph{Communications in Statistics - Simulation and Computation}, 1989.

\bibitem[Iandola et~al.(2020)Iandola, Shaw, Krishna, and Keutzer]{iandola2020squeezebert}
Iandola, F., Shaw, A., Krishna, R., and Keutzer, K.
\newblock Squeezebert: What can computer vision teach nlp about efficient neural networks?
\newblock \emph{SustaiNLP: Workshop on Simple and Efficient Natural Language Processing}, 2020.

\bibitem[Izmailov et~al.(2018)Izmailov, Podoprikhin, Garipov, Vetrov, and Wilson]{izmailov2018averaging}
Izmailov, P., Podoprikhin, D., Garipov, T., Vetrov, D., and Wilson, A.~G.
\newblock Averaging weights leads to wider optima and better generalization.
\newblock \emph{UAI}, 2018.

\bibitem[Jastrz{\k{e}}bski et~al.(2018)Jastrz{\k{e}}bski, Kenton, Ballas, Fischer, Bengio, and Storkey]{jastrzkebski2018relation}
Jastrz{\k{e}}bski, S., Kenton, Z., Ballas, N., Fischer, A., Bengio, Y., and Storkey, A.
\newblock On the relation between the sharpest directions of dnn loss and the sgd step length.
\newblock \emph{ICLR}, 2018.

\bibitem[Jiang et~al.(2020)Jiang, Neyshabur, Mobahi, Krishnan, and Bengio]{2020Fantastic}
Jiang, Y., Neyshabur, B., Mobahi, H., Krishnan, D., and Bengio, S.
\newblock Fantastic generalization measures and where to find them.
\newblock \emph{ICLR}, 2020.

\bibitem[Karimireddy et~al.(2019)Karimireddy, Rebjock, Stich, and Jaggi]{karimireddy2019error}
Karimireddy, S.~P., Rebjock, Q., Stich, S., and Jaggi, M.
\newblock Error feedback fixes signsgd and other gradient compression schemes.
\newblock \emph{ICML}, 2019.

\bibitem[Keskar et~al.(2017)Keskar, Mudigere, Nocedal, Smelyanskiy, and Tang]{keskar2016large}
Keskar, N.~S., Mudigere, D., Nocedal, J., Smelyanskiy, M., and Tang, P. T.~P.
\newblock On large-batch training for deep learning: Generalization gap and sharp minima.
\newblock \emph{ICLR}, 2017.

\bibitem[Khanh et~al.(2024)Khanh, Luong, Mordukhovich, and Tran]{khanh2024fundamental}
Khanh, P.~D., Luong, H.-C., Mordukhovich, B.~S., and Tran, D.~B.
\newblock Fundamental convergence analysis of sharpness-aware minimization.
\newblock \emph{NeurIPS}, 2024.

\bibitem[Kingma \& Ba(2015)Kingma and Ba]{kingma2014adam}
Kingma, D.~P. and Ba, J.
\newblock Adam: A method for stochastic optimization.
\newblock \emph{ICLR}, 2015.

\bibitem[Kiros(2013)]{kiros2013training}
Kiros, R.
\newblock Training neural networks with stochastic hessian-free optimization.
\newblock \emph{arXiv}, 2013.

\bibitem[Kunstner et~al.(2019)Kunstner, Hennig, and Balles]{kunstner2019limitations}
Kunstner, F., Hennig, P., and Balles, L.
\newblock Limitations of the empirical fisher approximation for natural gradient descent.
\newblock \emph{NeurIPS}, 32, 2019.

\bibitem[Kwon et~al.(2021)Kwon, Kim, Park, and Choi]{asam}
Kwon, J., Kim, J., Park, H., and Choi, I.~K.
\newblock Asam: Adaptive sharpness-aware minimization for scale-invariant learning of deep neural networks.
\newblock \emph{ICML}, 2021.

\bibitem[Li et~al.(2023)Li, Rakhlin, and Jadbabaie]{li2023convergence}
Li, H., Rakhlin, A., and Jadbabaie, A.
\newblock Convergence of adam under relaxed assumptions.
\newblock \emph{NeurIPS}, 2023.

\bibitem[Liu et~al.(2024)Liu, Li, Hall, Liang, and Ma]{sophia}
Liu, H., Li, Z., Hall, D., Liang, P., and Ma, T.
\newblock Sophia: A scalable stochastic second-order optimizer for language model pre-training.
\newblock \emph{ICLR}, 2024.

\bibitem[Liu et~al.(2022)Liu, Mai, Chen, Hsieh, and You]{looksam}
Liu, Y., Mai, S., Chen, X., Hsieh, C.-J., and You, Y.
\newblock Towards efficient and scalable sharpness-aware minimization.
\newblock \emph{CVPR}, 2022.

\bibitem[Loshchilov \& Hutter(2018)Loshchilov and Hutter]{loshchilov2018decoupled}
Loshchilov, I. and Hutter, F.
\newblock Decoupled weight decay regularization.
\newblock \emph{ICLR}, 2018.

\bibitem[Martens \& Grosse(2015)Martens and Grosse]{martens2015optimizing}
Martens, J. and Grosse, R.
\newblock Optimizing neural networks with kronecker-factored approximate curvature.
\newblock \emph{ICML}, 2015.

\bibitem[Martens et~al.(2010)]{martens2010deep}
Martens, J. et~al.
\newblock Deep learning via hessian-free optimization.
\newblock \emph{ICML}, 2010.

\bibitem[Merity et~al.(2022)Merity, Xiong, Bradbury, and Socher]{merity2022pointer}
Merity, S., Xiong, C., Bradbury, J., and Socher, R.
\newblock Pointer sentinel mixture models.
\newblock \emph{ICLR}, 2022.

\bibitem[Mi et~al.(2022)Mi, Shen, Ren, Zhou, Sun, Ji, and Tao]{mi2022make}
Mi, P., Shen, L., Ren, T., Zhou, Y., Sun, X., Ji, R., and Tao, D.
\newblock Make sharpness-aware minimization stronger: A sparsified perturbation approach.
\newblock \emph{NeurIPS}, 2022.

\bibitem[Natarajan et~al.(2013)Natarajan, Dhillon, Ravikumar, and Tewari]{natarajan2013learning}
Natarajan, N., Dhillon, I.~S., Ravikumar, P.~K., and Tewari, A.
\newblock Learning with noisy labels.
\newblock \emph{NeurIPS}, 2013.

\bibitem[Neyshabur et~al.(2017)Neyshabur, Bhojanapalli, Mcallester, and Srebro]{NIPS2017_Neyshabur}
Neyshabur, B., Bhojanapalli, S., Mcallester, D., and Srebro, N.
\newblock Exploring generalization in deep learning.
\newblock \emph{NeurIPS}, 2017.

\bibitem[Orvieto et~al.(2022)Orvieto, Kersting, Proske, Bach, and Lucchi]{antipgd_orvieto22a}
Orvieto, A., Kersting, H., Proske, F., Bach, F., and Lucchi, A.
\newblock Anticorrelated noise injection for improved generalization.
\newblock \emph{ICML}, 2022.

\bibitem[Qu et~al.(2022)Qu, Li, Duan, Liu, Tang, and Lu]{qu2022generalized}
Qu, Z., Li, X., Duan, R., Liu, Y., Tang, B., and Lu, Z.
\newblock Generalized federated learning via sharpness aware minimization.
\newblock \emph{ICML}, 2022.

\bibitem[Radford et~al.(2019)Radford, Wu, Child, Luan, Amodei, Sutskever, et~al.]{radford2019language}
Radford, A., Wu, J., Child, R., Luan, D., Amodei, D., Sutskever, I., et~al.
\newblock Language models are unsupervised multitask learners.
\newblock \emph{OpenAI blog}, 2019.

\bibitem[Roosta-Khorasani \& Ascher(2014)Roosta-Khorasani and Ascher]{hutchinson2}
Roosta-Khorasani, F. and Ascher, U.
\newblock Improved bounds on sample size for implicit matrix trace estimators.
\newblock \emph{FoCM}, 2014.

\bibitem[Schraudolph(2002)]{schraudolph2002fast}
Schraudolph, N.~N.
\newblock Fast curvature matrix-vector products for second-order gradient descent.
\newblock \emph{Neural computation}, 2002.

\bibitem[Shin et~al.(2025)Shin, Lee, Andriushchenko, and Lee]{shin2024critical}
Shin, S., Lee, D., Andriushchenko, M., and Lee, N.
\newblock Critical influence of overparameterization on sharpness-aware minimization.
\newblock \emph{UAI}, 2025.

\bibitem[Wadia et~al.(2021)Wadia, Duckworth, Schoenholz, Dyer, and Sohl-Dickstein]{wadia2021whitening}
Wadia, N., Duckworth, D., Schoenholz, S.~S., Dyer, E., and Sohl-Dickstein, J.
\newblock Whitening and second order optimization both make information in the dataset unusable during training, and can reduce or prevent generalization.
\newblock \emph{ICML}, 2021.

\bibitem[Wang et~al.(2018)Wang, Singh, Michael, Hill, Levy, and Bowman]{wang2018glue}
Wang, A., Singh, A., Michael, J., Hill, F., Levy, O., and Bowman, S.~R.
\newblock Glue: A multi-task benchmark and analysis platform for natural language understanding.
\newblock \emph{ICLR}, 2018.

\bibitem[Wilson et~al.(2017)Wilson, Roelofs, Stern, Srebro, and Recht]{wilson2017marginal}
Wilson, A.~C., Roelofs, R., Stern, M., Srebro, N., and Recht, B.
\newblock The marginal value of adaptive gradient methods in machine learning.
\newblock \emph{NeurIPS}, 2017.

\bibitem[Wolf et~al.(2020)Wolf, Debut, Sanh, Chaumond, Delangue, Moi, Cistac, Rault, Louf, Funtowicz, Davison, Shleifer, von Platen, Ma, Jernite, Plu, Xu, Scao, Gugger, Drame, Lhoest, and Rush]{wolf2020huggingfaces}
Wolf, T., Debut, L., Sanh, V., Chaumond, J., Delangue, C., Moi, A., Cistac, P., Rault, T., Louf, R., Funtowicz, M., Davison, J., Shleifer, S., von Platen, P., Ma, C., Jernite, Y., Plu, J., Xu, C., Scao, T.~L., Gugger, S., Drame, M., Lhoest, Q., and Rush, A.~M.
\newblock Huggingface's transformers: State-of-the-art natural language processing.
\newblock \emph{arXiv}, 2020.

\bibitem[Wu et~al.(2018)Wu, Ma, et~al.]{wu2018sgd}
Wu, L., Ma, C., et~al.
\newblock How sgd selects the global minima in over-parameterized learning: A dynamical stability perspective.
\newblock \emph{NeurIPS}, 31, 2018.

\bibitem[Xie et~al.(2020)Xie, Sato, and Sugiyama]{xie2020diffusion}
Xie, Z., Sato, I., and Sugiyama, M.
\newblock A diffusion theory for deep learning dynamics: Stochastic gradient descent exponentially favors flat minima.
\newblock \emph{ICLR}, 2020.

\bibitem[Yao et~al.(2020)Yao, Gholami, Keutzer, and Mahoney]{yao2020pyhessian}
Yao, Z., Gholami, A., Keutzer, K., and Mahoney, M.~W.
\newblock Pyhessian: Neural networks through the lens of the hessian.
\newblock \emph{IEEE BigData}, 2020.

\bibitem[Yao et~al.(2021)Yao, Gholami, Shen, Keutzer, and Mahoney]{adahessian}
Yao, Z., Gholami, A., Shen, S., Keutzer, K., and Mahoney, M.~W.
\newblock Adahessian: An adaptive second order optimizer for machine learning.
\newblock \emph{AAAI}, 2021.

\bibitem[Zagoruyko \& Komodakis(2016)Zagoruyko and Komodakis]{zagoruyko2016wide}
Zagoruyko, S. and Komodakis, N.
\newblock Wide residual networks.
\newblock \emph{BMVC}, 2016.

\bibitem[Zhang et~al.(2024)Zhang, Chen, Ding, Li, Sun, and Luo]{zhang2024why}
Zhang, Y., Chen, C., Ding, T., Li, Z., Sun, R., and Luo, Z.-Q.
\newblock Why transformers need adam: A hessian perspective.
\newblock \emph{NeurIPS}, 2024.

\bibitem[Zhou et~al.(2020)Zhou, Feng, Ma, Xiong, Hoi, et~al.]{zhou2020towards}
Zhou, P., Feng, J., Ma, C., Xiong, C., Hoi, S. C.~H., et~al.
\newblock Towards theoretically understanding why sgd generalizes better than adam in deep learning.
\newblock \emph{NeurIPS}, 2020.

\bibitem[Zhuang et~al.(2022)Zhuang, Gong, Yuan, Cui, Adam, Dvornek, sekhar tatikonda, s~Duncan, and Liu]{gsam}
Zhuang, J., Gong, B., Yuan, L., Cui, Y., Adam, H., Dvornek, N.~C., sekhar tatikonda, s~Duncan, J., and Liu, T.
\newblock Surrogate gap minimization improves sharpness-aware training.
\newblock \emph{ICLR}, 2022.

\bibitem[Zou et~al.(2022)Zou, Cao, Li, and Gu]{zou2022understanding}
Zou, D., Cao, Y., Li, Y., and Gu, Q.
\newblock Understanding the generalization of adam in learning neural networks with proper regularization.
\newblock \emph{ICLR}, 2022.

\end{thebibliography}
\bibliographystyle{icml2025}


\newpage
\appendix
\onecolumn
\section{Sharpness Measurements for Other Settings} \label{app:sharp}

\bgroup
\begin{table}[ht!]
    \centering
    \caption{Sharpness measurements of the solutions found by seven different optimizers and their generalization on CIFAR-10/100 and Wikitext-2. Approximate second-order methods tend to yield highly sharp solutions and poor generalization compared to SGD; \sassha and \msassha effectively recover this. Here, we measure sharpness in terms of maximum Hessian eigenvalue $\lambda_{max}(H)$, trace of Hessian $\operatorname{tr}(H)$, gradient-direction sharpness $\delta L_{\text{grad}}$, and average sharpness $\delta L_{\text{avg}}$, along with generalization using validation loss $L_{val}$ and accuracy $\text{Acc}_{val}$.}
    \vskip 0.1in
    \resizebox{\linewidth}{!}{
        \begin{tabular}{clcccccccc}
            \toprule
             & &  \multicolumn{4}{c}{Sharpness} & \multicolumn{2}{c}{Generalization} \\
             \cmidrule(l{3pt}r{3pt}){3-6} \cmidrule(l{3pt}r{3pt}){7-8}    &    &
             {$\lambda_{max}(H)$}                              &
             {$\operatorname{tr}(H)_{\times 10^3}$}            &
             $\delta L_{\text{grad}}$                          &
             $\delta L_{\text{avg} \times 10^{-3}}$            &
             $ L_\text{val} $                                  &
             $\text{Acc}_\text{val}$                          \\ 
             
             \midrule
             
             \rowcolor{lgray} \multicolumn{8}{c}{CIFAR-10} 
             
             \\ \midrule
             
             \multirow{7}{*}{ResNet20}
             
             & SGD        &
             \alignednum{107}[4.370]         &
             \alignednum{1.380}[0.010]       &
             \alignednum{0.840}[0.304]       & 
             \alignednum{0.690}[0.390]        &
             \alignednum{0.295}[0.008]       &
             \alignednum{92.03}[0.32]       \\
             
             & SAM         &  
             \alignednum{58}[2.980]           &
             \alignednum{0.730}[0.040]         &
             \alignednum{0.171}[0.038]       & 
             \alignednum{0.461}[0.240]        & 
             \alignednum{0.119}[0.002]       & 
             \alignednum{92.85}[0.07]       \\
             
             & AdaHessian  & 
             \alignednum{23048}[29932]       &
             \alignednum{189.5}[240.6]    &
             \alignednum{4.538}[1.634]       &
             \alignednum{198.7}[266.0]        & 
             \alignednum{0.260}[0.006]       & 
             \alignednum{92.00}[0.17]       \\

             & Sophia-H    & 
             \alignednum{3606}[303.0]         &
             \alignednum{31.24}[2.628]      &
             \alignednum{6.120}[1.634]       &
             \alignednum{18.11}[1.000]           &
             \alignednum{0.316}[0.002]      &
             \alignednum{91.81}[0.27]        \\
                                       
             & Shampoo     & 
             \alignednum{647066}[419964]     & 
             \alignednum{3900}[1825]         & 
             \alignednum{166.3}[48.00]       &
             \alignednum{2177189}[1628993]   &
             \alignednum{0.381}[0.028]       &
             \alignednum{88.55}[0.83]        \\
             
             \rowcolor{green!20}\cellcolor{white} &  
             \msassha &
             \alignednum{129}[17.00]            &
             \alignednum{1.580}[0.080]        &
             \alignednum{1.551}[0.684]       &
             \alignednum{1.025}[0.360]        &
             \alignednum{0.234}[0.003]       &
             \alignednum{92.36}[0.23]        \\

            \rowcolor{green!20}\cellcolor{white} &
             \sassha  &
             \alignednum{78}[5.090]          &
             \alignednum{0.860}[0.030]      & 
             \alignednum{0.184}[0.053]       &
             \alignednum{0.388}[0.704]      &
             \alignednum{0.209}[0.001]     &
             \alignednum{92.98}[0.05] \\ 
             
             \midrule
             
             \multirow{7}{*}{ResNet32}
             & SGD        &
             \alignednum{56}[5.100]           &
             \alignednum{0.800}[0.040]         &
             \alignednum{0.560}[0.219]       &
             \alignednum{0.196}[0.146]       &
             \alignednum{0.309}[0.002]       &
             \alignednum{92.69}[0.06]       \\
             
             & SAM        & 
             \alignednum{45}[2.670]         & 
             \alignednum{0.580}[0.020]        &
             \alignednum{0.107}[0.005]     &  
             \alignednum{0.753}[0.351]     &
             \alignednum{0.128}[0.001]      &
             \alignednum{93.89}[0.13] \\

             & AdaHessian  &
             \alignednum{1746}[1018]       &
             \alignednum{17.06}[10.24]      &
             \alignednum{4.599}[1.710]     &
             \alignednum{5.518}[3.623]     &
             \alignednum{0.278}[0.006]     &
             \alignednum{92.48}[0.15] \\
             
             & Sophia-H    &
             \alignednum{7167}[2755]       &
             \alignednum{18.82}[5.500]       &
             \alignednum{9.399}[2.283]   &
             \alignednum{7.915}[3.397]     & 
             \alignednum{0.394}[0.010]      &
             \alignednum{91.99}[0.08] \\
             
             & Shampoo &
             \alignednum{717553}[93129]    &
             \alignednum{4523}[629.7]        &
             \alignednum{162.1}[123.2]      & 
             \alignednum{105322}[82246]     & 
             \alignednum{0.348}[0.008]      &
             \alignednum{90.23}[0.24] \\

             \rowcolor{green!20}\cellcolor{white} &
             \msassha & 
             \alignednum{283}[10.00]          &
             \alignednum{3.960}[0.100]       & 
             \alignednum{2.986}[1.133]     & 
             \alignednum{1.300}[0.969]     &
             \alignednum{0.211}[0.010]      & 
             \alignednum{93.18}[0.30] \\

            \rowcolor{green!20}\cellcolor{white} &
            \sassha   &
            \alignednum{47}[1.880]           &
            \alignednum{0.590}[0.020]         &
            \alignednum{0.136}[0.019]       &
            \alignednum{0.714}[0.090]      &
            \alignednum{0.177}[0.002]      &
            \alignednum{94.09}[0.24]  \\ 
            
            \midrule
            
            \rowcolor{lgray} 
            \multicolumn{8}{c}{CIFAR-100} \\ \midrule
            \multirow{7}{*}{ResNet32} &
             
             SGD        &
             \alignednum{265}[25.00]          &
             \alignednum{7.290}[0.300]        &
             \alignednum{0.703}[0.132]        &
             \alignednum{1.310}[1.030]        &
             \alignednum{1.260}[0.001]        &
             \alignednum{69.32}[0.19]         \\
             
             &SAM        & 
             \alignednum{123}[11.00]          &
             \alignednum{2.630}[0.090]        &
             \alignednum{0.266}[0.025]        &  
             \alignednum{-0.619}[0.594]       &
             \alignednum{0.512}[0.016]        &
             \alignednum{71.99}[0.20]         \\

             &AdaHessian  &
             \alignednum{11992}[5779]         & 
             \alignednum{46.94}[17.60]        & 
             \alignednum{4.119}[1.136]        & 
             \alignednum{12.50}[6.080]        & 
             \alignednum{1.377}[0.070]        & 
             \alignednum{68.06}[0.22]         \\
             
             &Sophia-H    & 
             \alignednum{22797}[10857]        &
             \alignednum{68.15}[20.19]        & 
             \alignednum{8.130}[3.082]        &
             \alignednum{19.19}[6.380]        & 
             \alignednum{1.463}[0.022]        & 
             \alignednum{67.76}[0.37]         \\
             
             &Shampoo &
             \alignednum{436374}[9017]        &
             \alignednum{6823}[664.7]         &
             \alignednum{73.27}[12.51]        & 
             \alignednum{49307489}[56979794]  & 
             \alignednum{1.386}[0.010]        & 
             \alignednum{64.08}[0.46]         \\
             
             \rowcolor{green!20}\cellcolor{white} &
             \msassha & 
             \alignednum{382}[65.00]          & 
             \alignednum{8.750}[0.310]        & 
             \alignednum{2.391}[0.425]        & 
             \alignednum{2.260}[1.660]        & 
             \alignednum{1.067}[0.001]        & 
             \alignednum{70.93}[0.21]         \\
             
            \rowcolor{green!20}\cellcolor{white} &
            \sassha   &
            \alignednum{107}[40.00]           &
            \alignednum{1.870}[0.650]         &
            \alignednum{0.238}[0.088]         &
            \alignednum{0.650}[0.860]         &
            \alignednum{0.961}[0.005]         & 
            \alignednum{72.14}[0.16]          \\ 
            
            \midrule
            
            \multirow{7}{*}{WRN28-10}
             & SGD        &
             \alignednum{18}[1.170] & 
             \alignednum{0.660}[0.040]        & 
             \alignednum{1.984}[0.506] & 
             \alignednum{-0.007}[0.028] & 
             \alignednum{0.820}[0.005] & 
             \alignednum{80.06}[0.15] \\
             
             & SAM         &  
             \alignednum{9}[0.866] & 
             \alignednum{0.230}[0.010]   &
             \alignednum{0.841}[0.084]  &
             \alignednum{0.024}[0.041]  &  
             \alignednum{0.648}[0.006]  & 
             \alignednum{83.14}[0.13] \\
             
             & AdaHessian  & 
             \alignednum{35119}[46936]      &
             \alignednum{139.5}[191.0]   &
             \alignednum{6.745}[1.932]  & 
             \alignednum{19.727}[27.87] & 
             \alignednum{1.005}[0.008] & 
             \alignednum{76.92}[0.26] \\
             
             & Sophia-H    & 
             \alignednum{3419}[3240]        &
             \alignednum{13.57}[3.300]               &
             \alignednum{5.073}[0.268] & 
             \alignednum{0.067}[0.054] &
             \alignednum{0.866}[0.003] & 
             \alignednum{79.35}[0.24] \\
                          
             & Shampoo     &
             \alignednum{102129}[60722]   & 
             \alignednum{1459}[709.4]  &
             \alignednum{483.0}[172.0] &
             \alignednum{98.558}[123.1] & 
             \alignednum{1.173}[0.088] & 
             \alignednum{74.06}[1.28] \\
             
             \rowcolor{green!20}\cellcolor{white} &
             \msassha &
             \alignednum{2257}[248.0]         & 
             \alignednum{30.40}[4.780]               & 
             \alignednum{4.599}[0.003]  & 
             \alignednum{0.301}[0.047] &
             \alignednum{0.757}[0.011] & 
             \alignednum{81.53}[0.27] \\
             
             \rowcolor{green!20}\cellcolor{white} & 
             \sassha  & 
             \alignednum{84}[3.150]          & 
             \alignednum{2.030}[0.110]      & 
             \alignednum{4.540}[0.122] &  
             \alignednum{0.007}[0.129]  &
             \alignednum{0.625}[0.002] & 
             \alignednum{83.54}[0.08] \\
             
             \midrule
             
             \rowcolor{lgray} \multicolumn{8}{c}{Wikitext-2} \\ \midrule
            \multirow{5}{*}{Mini-GPT1} 
            & AdamW       & \alignednum{836}[13.00]           
                          & \alignednum{31.61}[0.433]    
                          & \alignednum{1.642}[1.036] 
                          & \alignednum{7}[0] 
                          & \alignednum{5.072}[0.013] 
                          & \alignednum{175.06}[0.19] 
                          \\
                          
             & AdaHessian & \alignednum{13141}[14432]           
                          & \alignednum{46.36}[26.85]           
                          & \alignednum{0.289}[0.187]    
                          & \alignednum{9}[5]   
                          & \alignednum{7.231}[0.043]    
                          & \alignednum{407.69}[0.20] \\
                          
             & Sophia-H   & \alignednum{319}[14.00]      
                          & \alignednum{55.17}[1.100] 
                          & \alignednum{0.824}[0.089] 
                          & \alignednum{13}[1] 
                          & \alignednum{5.077}[0.014] 
                          & \alignednum{157.60}[0.37] \\
            \rowcolor{green!20}\cellcolor{white} &
                          \msassha 
                          & \alignednum{145}[125.0]           
                          & \alignednum{13.23}[17.19]    
                          & \alignednum{0.379}[0.275] 
                          & \alignednum{3}[1] 
                          & \alignednum{5.259}[0.010] 
                          & \alignednum{125.01}[0.21] \\
            \rowcolor{green!20}\cellcolor{white} & 
                          \sassha   
                          & \alignednum{79}[2.000]  
                          & \alignednum{14.50}[0.325] 
                          & \alignednum{0.221}[0.023] 
                          & \alignednum{3}[0] 
                          & \alignednum{4.808}[0.001] 
                          & \alignednum{122.40}[0.16]\\
            \bottomrule

        \end{tabular}
    }
    \label{tab:add_sharp}
\end{table}
\egroup

\newpage

\section{Algorithm Comparison} \label{app:comparison}

\begin{table}[ht!]
\renewcommand{\arraystretch}{1.8}
\caption{Comparison of various optimization algorithms in terms of gradient momentum $m_t$, diagonal preconditioning matrix $D_t$, and method-specific operations $\mathbf{U}(z)$. Here $g_t, \widehat{H}_t$ are the stochastic gradient and the Hessian estimation respectively, and $\beta_1, \beta_2$ denotes the momentum hyperparameters for the gradient and estimated Hessian. Bias correction \texttt{bc}$(\cdot)$ compensates for initialization biases in the gradient and Hessian momentum variables due to zero initialization.} 

\vskip 0.1in
\resizebox{\linewidth}{!}{
    \begin{tabular}{lllll}
        \toprule
        \rowcolor{lgray} \multicolumn{4}{c}{\large $x_{t+1} = x_t - \eta_{t} \mathbf{U}(D_{t}^{-1}m_{t})$}\\[0.3em] \midrule
         & $m_t$ & $D_t$ & $ \mathbf{U}(z) $ \\ \midrule
        SGD with momentum      &
        $\beta_1 m_{t-1} + (1 - \beta_1)g_t $ &
        $I$ &
        $z$ \\ 
        
        Stochastic Newton &
        $g_t$                                 &
        $H_t(x_t)$ &
        $z$ \\ 
        
        Adam \citep{kingma2014adam}             &
        $\beta_1 m_{t-1} + (1 - \beta_1)g_t $ &
        $ \sqrt{ \beta_2 v_{t-1} + (1-\beta_2) \colorbox{red!20}{$\operatorname{diag}(g_tg_t^\top)$} }$ &
        $ \texttt{bc}(z) $ \\ 
        
        
        AdaHessian \citep{adahessian}       &
        \large "                              &
        $ \sqrt{ \beta_2 v_{t-1} + (1-\beta_2) \colorbox{red!20}{$\widehat{H}_t^{(s)}(x_t)^2$} }$ &
        $ \texttt{bc}(z) $ \\ 
        
        Sophia-H \citep{sophia}         &
        \large "                              &
        \hspace{1em} $ \beta_2 v_{t-1} + (1-\beta_2) \colorbox{red!20}{$\widehat{H}^{(c)}_t(x_t)$}$ every $k$ steps &
        \colorbox{red!20}{$\operatorname{clip}(z)$}  \\
        
        \midrule

        \sassha (Ours)       &
        $ \beta_1 m_{t-1} + (1-\beta_1) \colorbox{red!20}{$g_{t}(x_t+\boldsymbol{\epsilon_t^\star})$} $ &
        $ \sqrt{ \beta_2 v_{t-1} + (1-\beta_2) \colorbox{red!20}{$|\widehat{H}_t(x_t+\boldsymbol{\epsilon_t^\star})|$} }$ every $k$ steps & 
        $ \texttt{bc}(z) $ \\
        
        \bottomrule
    \end{tabular}
}
\label{tab:comp_algo}
\end{table}
In this section, we compare our algorithm with other adaptive and approximate second-order methods designed for deep learning to better illustrate our contributions within concurrent literature.
We present a detailed comparison of each methods in \cref{tab:comp_algo}.

Adam \citep{kingma2014adam} is an adaptive method popular among practitioners, which rescales the learning rate for each parameter dimension by dividing by the square root of the moving average of squared gradients.
This adaptive learning rate effectively adjusts the gradient (momentum) at each descent step, accelerating convergence and improving update stability.
Although Adam is not explicitly a second-order method, its process is related to second-order methods as it can be viewed as preconditioning via a diagonal approximation of the empirical Fisher information matrix.
AdamW \citep{loshchilov2018decoupled} proposes to improve Adam by decoupling the weight decay from the update rule for better generalization.
This is also shown to be effective in most approximate second-order methods, thus employed in all subsequently mentioned algorithms.

AdaHessian \citep{adahessian} is one of the earliest approximate second-order optimization methods tailored for deep learning.
To reduce the prohibitive cost of computing the Hessian, it uses Hutchinson's method \citep{hutchinson, hutchinson2} to estimate a diagonal Hessian approximation $\widehat{H}_{t}$ and applies a moving average to reduce variance in the estimation.
The authors also propose spatial averaging of the Hessian estimate, denoted as ($\widehat{H}_{t}^{(s)}$), which involves averaging the diagonal element within a filter of a convolution layer for filter-wise gradient scaling.
Sophia \citep{sophia} is an approximate second-order method specifically designed for language model pretraining.
Its primary feature is the use of the clipping mechanism $\operatorname{clip}(z) = \max\{\min\{z, \rho\}, -\rho\}$ with a predefined threshold $\rho$ to control the worst-case update size resulting from errorneous diagonal Hessian estimates in preconditioning.
Additionally, a hard adjustment is applied to each Hessian entry, substituting negative and very small values with a constant $\epsilon$, such as $\widehat{H}^{(c)}_t = \max \{ \widehat{h}_t , \epsilon \}$ to prevent convergence to saddle points and mitigate numerical instability.
Furthermore, Sophia also incorporates lazy Hessian updates to enhance computational efficiency. 
This works without significant performance degradation as the clipping technique and hard adjustment prevent a rapid change of the Hessian, keeping the previous Hessian relevant over more extended steps.

Our method, \sassha, adds a perturbation $\epsilon^\star_t$ before computing the gradient and diagonal Hessian estimation to penalize sharpness during training for promoting improved generalization -- an approach that, to our knowledge, has not been previously explored in the literature.
However, naive second-order optimization under sharpness minimization can be numerically unstable, because, when the overall curvature is small due to sharpness reduction, Hessian underestimation can cause the optimization to yield extremely large update.
To address this issue, we introduce two simple techniques to Hessian estimates: the square root and the absolute function.
The square root smoothly adjusts underestimated curvature while preserving the relative scale among Hessian entries.
The absolute function enforces Hessian estimates to be semi-positive definite while maintaining their magnitude.
This not only prevents convergence to saddle points or local maxima, but also allows the square root to operate on Hessian entries retaining the original Hessian magnitude.  
Together, the combination of sharpness minimization and Hessian stabilization enables efficient reuse of previously computed Hessians, resulting in a stable and efficient algorithm.

\section{Convergence Analysis of \sassha} \label{app:convergence}

In this section, we provide preliminary convergence analysis results.
Based on the well-established analyses of \citet{li2023convergence, khanh2024fundamental}, we further investigate the complexities arising from preconditioned perturbed gradients.\\

\begin{assumption}
    \label{assumption:convex_smooth_bounded_nonzero}
    The function $ f : \mathbb{R}^d \to \mathbb{R} $ is convex, $\beta$-smooth, and bounded from below, i.e., $ f^* := \inf_x f(x) > -\infty $. Additionally, the gradient $ \nabla f(x_t) $  is non-zero for a finite number of iterations, i.e., $ \nabla f(x_t) \neq 0 $ for all $ t \in \{1, 2, \dots, n\} $.
\end{assumption}
\vspace{5mm}
\begin{assumption}
    \label{assumption:learning_rate_and_perturbation_radius}
    Step sizes $\eta_t$ and perturbation radii $\rho_t$ are assumed to satisfy the following conditions:
    \begin{align*}
        \sum_{t=1}^{\infty} \eta_t = \infty, \quad \sum_{t=1}^{\infty} \eta_t^2 < \infty, \quad \sum_{t=1}^{\infty} \rho_t^2 \eta_t < \infty.
    \end{align*}
\end{assumption}
\vspace{3mm}
\begin{remark}
    The following notations will be used throughout
    \begin{enumerate}
        \item $ g_t := \nabla f(x_t) $ denotes the gradient of $ f $ at iteration $t$.
        \item The intermediate points and the difference between the gradients are defined as
        \begin{align*}
            x_{t+\frac{1}{2}} := x_t + \rho_t \frac{g_t}{\|g_t\|}, \quad g_{t+\frac{1}{2}} := \nabla f(x_{t+\frac{1}{2}}), \quad \delta_t := g_{t+\frac{1}{2}} - g_t.
        \end{align*}
        \item For $ u, v \in \mathbb{R}^d $, operations such as $ \sqrt{v}, |v| $ and $ \frac{v}{u} $, as well as the symbols $ \preceq $ and $ \succeq $, are applied element-wise.
    \end{enumerate}
\end{remark}
\vspace{3mm}
\begin{remark}
    The update rule for the iterates is given by
    \begin{align}\label{remark:sassha_iteration_adamlike}
        x_{t+1} = x_t - \frac{\eta_t}{\sqrt{|\operatorname{diag}(\nabla^2 f(x_{t+\frac{1}{2}}))|} + \epsilon} \odot g_{t+\frac{1}{2}},
    \end{align}
    where $ \operatorname{diag} $ extracts the diagonal elements of a matrix as a vector, or constructs a diagonal matrix from a vector, and  
    $ \epsilon $ is a small constant introduced to prevent division by zero.
    Define  $h_t$  as
    \begin{align*}
        h_t = \frac{\eta_t}{\sqrt{|\operatorname{diag}(\nabla^2 f(x_{t+\frac{1}{2}}))|} + \epsilon},
    \end{align*}
    then the following hold
    \begin{enumerate}
        \item From the convexity and $\beta$-smoothness of $ f $, the diagonal elements of $ \nabla^2 f(x) $ are bounded within the interval $[0, \beta]$, i.e.,
        \begin{align*}
            0 \leq \left[\nabla^2 f(x)\right]_{(i,i)} = e_i^\top \nabla^2 f(x) e_i \leq \beta,        
        \end{align*}
        where $ e_i $ is the $ i $-th standard basis vector in $ \mathbb{R}^d $.
        \item The term $h_t$ is bounded as 
        \begin{align*}
            \frac{\eta_t}{\sqrt{\beta} + \epsilon} \preceq h_t \preceq \frac{\eta_t}{\epsilon}.
        \end{align*}
    \end{enumerate}
\end{remark}

\begin{remark}
    For the matrix representation
    \begin{enumerate}
        \item Denoting $ H_t := \operatorname{diag}(h_t) $, the matrix bounds for $H_t$  are given by
        \begin{align}\label{remark:matrix_bound}
            \frac{\eta_t}{\sqrt{\beta} + \epsilon} I \preceq H_t \preceq \frac{\eta_t}{\epsilon} I,
        \end{align}
        where  $I$  is the identity matrix.
        \item Using the matrix notation $ H_t $, the update for the iterates is expressed as
        \begin{align*}
            x_{t+1} = x_t - H_t g_{t+\frac{1}{2}}.    
        \end{align*}
        \end{enumerate}
\end{remark}

\begin{remark}
    From the $\beta$-smoothness of $f$, $\delta_t$ is bounded by
    \begin{align}\label{remark:delta_bound}
        \|\delta_t\| \leq \beta \| x_t + \rho_t \frac{\nabla f(x_t)}{\|\nabla f(x_t)\|} - x_t \| = \beta \rho_t.
    \end{align}
\end{remark}

\begin{lemma}[Descent Lemma]\label{lemma:descent_lemma}
    Under \cref{assumption:convex_smooth_bounded_nonzero} and \cref{assumption:learning_rate_and_perturbation_radius}, for given $\beta$ and $\epsilon$, there exists a $T \in \mathbb{N}$ such that for $\forall t \geq T$, $\eta_t$ satisfies $\eta_t \leq \min\left\{ \frac{\epsilon^2}{6\beta(\sqrt{\beta} + \epsilon)}, \frac{\epsilon}{4\beta} \right\}$.
    For such $t\geq T$, the following inequality holds
    \begin{align}
        f(x_{t+1}) \leq f(x_t) - \frac{\eta_t}{2(\sqrt{\beta} + \epsilon)} \|g_t\|^2 + \frac{\eta_t}{\epsilon} \|\delta_t\|^2.
    \end{align}
\end{lemma}

\begin{proof}
    We begin by applying the $\beta$-smoothness of $f$, 
    \begin{align*} 
        f(x_{t+1}) & \leq f(x_t) + \left\langle g_t, x_{t+1} - x_t \right\rangle + \frac{\beta}{2} \|x_{t+1} - x_t\|^2\\
        & =f\left(x_t\right)-\left\langle g_t, H_t (g_t+\delta_t)\right\rangle+\frac{\beta}{2}\left\|H_t (g_t+\delta_t)\right\|^2 \\
        & \leq f\left(x_t\right)- g_t^\top H_tg_t 
        + \frac{1}{2\alpha} g_t^\top H_t g_t + \frac{\alpha}{2} \delta_t^\top H_t \delta_t
        +\frac{\beta}{2}\left\|H_t (g_t+\delta_t)\right\|^2\\
        & \leq f\left(x_t\right)-(1-\frac{1}{2\alpha})  \frac{\eta_t}{\sqrt{\beta}+\epsilon} \|g_t\|^2  +\frac{\alpha}{2} \frac{\eta_t}{\epsilon} \|\delta_t\|^2 +\frac{\beta}{2}\frac{\eta_t^2}{\epsilon^2}\left\|g_t+\delta_t\right\|^2 \\
        & \leq f\left(x_t\right)-(1-\frac{1}{2\alpha})  \frac{\eta_t}{\sqrt{\beta}+\epsilon} \|g_t\|^2  +\frac{\alpha}{2} \frac{\eta_t}{\epsilon} \|\delta_t\|^2 +\beta\frac{\eta_t^2}{\epsilon^2}(\left\|g_t\|^2+\|\delta_t\right\|^2) \\
        & = f\left(x_t\right)-\eta_t((1-\frac{1}{2\alpha})\frac{1}{\sqrt{\beta}+\epsilon}-\beta\frac{\eta_t}{\epsilon^2}) \|g_t\|^2  +\eta_t(\frac{\alpha}{2\epsilon} + \beta\frac{\eta_t}{\epsilon^2})\|\delta_t\|^2.
\end{align*}
The second inequality follows from Young's inequality, the third inequality is obtained from \cref{remark:matrix_bound}, and the last inequality is simplified using the property \( \|a + b\|^2 \leq 2\|a\|^2 + 2\|b\|^2 \). By setting \( \alpha = \frac{3}{2} \), we get
\begin{align*}
        \hspace{3.5mm}= f\left(x_t\right)-\eta_t(\frac{2}{3}\left(\frac{1}{\sqrt{\beta}+\epsilon}\right)-\beta\frac{\eta_t}{\epsilon^2}) \|g_t\|^2  +\eta_t(\frac{3}{4\epsilon} + \beta\frac{\eta_t}{\epsilon^2})\|\delta_t\|^2.
    \end{align*}    
    Since $\eta_t \to 0,\ \exists T\in\mathbb{N}$ such that $\eta_t\leq\min\{\frac{\epsilon^2}{6\beta(\sqrt{\beta}+\epsilon)}, \frac{\epsilon}{4\beta}\}$,
    this gives $\frac{2}{3}\left(\frac{1}{\sqrt{\beta}+\epsilon}\right)-\beta\frac{\eta_t}{\epsilon^2}\geq \frac{1}{2(\sqrt{\beta}+\epsilon)}$ and $\frac{3}{4\epsilon} + \beta\frac{\eta_t}{\epsilon^2} \leq \frac{1}{\epsilon}$, which implies
    \begin{align*}
    \hspace{-34mm} \leq f\left(x_t\right)-\frac{\eta_t}{2(\sqrt{\beta}+\epsilon)} \|g_t\|^2 +\frac{\eta_t}{\epsilon} \|\delta_t\|^2     
    \end{align*}
\end{proof}
\begin{theorem}
    Under \cref{assumption:convex_smooth_bounded_nonzero} and \cref{assumption:learning_rate_and_perturbation_radius}, given any initial point $x_0 \in \mathbb{R}^d$, let $\{x_t\}$ be generated by \cref{remark:sassha_iteration_adamlike}. Then, it holds that $\liminf_{t \to \infty} \|g_t\| = 0$.
\end{theorem}
\begin{proof}
    From \cref{lemma:descent_lemma} and \cref{remark:delta_bound}, we have the bound
    \begin{align*}
        f(x_{t+1}) &\leq f(x_t) - \frac{\eta_t}{2(\sqrt{\beta} + \epsilon)} \|g_t\|^2 + \frac{\eta_t}{\epsilon} \|\delta_t\|^2 \\
        &\leq f(x_t) - \frac{\eta_t}{2(\sqrt{\beta} + \epsilon)} \|g_t\|^2 + \frac{\eta_t}{\epsilon} \beta^2 \rho_t^2.
    \end{align*}
    
    By rearranging the terms, we obtain the following
    \begin{align*}
        \frac{\eta_t}{2(\sqrt{\beta} + \epsilon)} \|g_t\|^2 \leq f(x_t) - f(x_{t+1}) + \frac{\eta_t}{\epsilon} \beta^2 \rho_t^2.
    \end{align*}
    
    For any $M > T$, we have
    \begin{align*}
    \frac{1}{2(\sqrt{\beta} + \epsilon)} \sum_{t=T}^M \eta_t \|g_t\|^2 &\leq \sum_{t=T}^M \left(f(x_t) - f(x_{t+1})\right) + \frac{\beta^2}{\epsilon} \sum_{t=T}^M \rho_t^2 \eta_t \\
    &= f(x_T) - f(x_{M+1}) + \frac{\beta^2}{\epsilon} \sum_{t=T}^{M} \rho_t^2 \eta_t \\
    &\leq f(x_T) - \inf_{t \in \mathbb{N}} f(x_t) + \frac{\beta^2}{\epsilon} \sum_{t=T}^{M} \rho_t^2 \eta_t.
    \end{align*}
    As \( M \to \infty \), the series \( \sum_{t=T}^{\infty} \eta_t \|g_t\|^2 \) converges. Now, assume for contradiction that \( \liminf_{t \to \infty} \|g_t\| \neq 0 \). This means there exists some \( \xi > 0 \) and \( N \geq T \) such that \( \|g_t\| \geq \xi \) for all \( t \geq N \). Consequently, we have
    \begin{align*}
        \infty > \sum_{t=N}^{\infty} \eta_t \|g_t\|^2 \geq \xi^2 \sum_{t=N}^{\infty} \eta_t = \infty,    
    \end{align*}
    which is a contradiction. Therefore, \( \liminf_{t \to \infty} \|g_t\| = 0 \).
\end{proof}

\newpage

\section{Linear stability analysis on \sassha}
\label{app:linear_stability}
In this section, we provide the detailed proof of \cref{sec:linstab}.

We begin by considering the minimization of the training error

\begin{equation}
    f(x) = \frac{1}{n} \sum^n_{i=1} f_i (x) \nonumber
\end{equation}
via a general second-order optimization method:
\begin{equation} \label{eq:general_update}
    x_{t+1} = x_t - 
    P(x_t; \xi_t )^{-1} \nabla f (x_t; \xi_t)
\end{equation}
where $ \xi_t $ is a \textit{i.i.d.} random variable independent of $x_t$. 
For \sassha, the update is specified as
\begin{align}\label{eq:update}
    x_{t+1} = x_t - 
    \eta \, \frac{1}{\sqrt{|\operatorname{diag} \bigl( \nabla^2 f_{ \xi_t } ( x_t + \rho \, \nabla f_{ \xi_t }(x_t) \bigr) |} + \epsilon} \odot 
    \nabla f_{\xi_t} ( x_t + \rho\,\nabla f_{\xi_t}(x_t)),
\end{align}
where $ \operatorname{diag} $ extracts the diagonal elements of a matrix as a vector, or constructs a diagonal matrix from a vector, and $ \epsilon $ is a small constant introduced to prevent division by zero.

We now derive the linear stability conditions for the fixed points of \sassha.

\begin{definition} (Fixed point).
    A point $ x^\star $ is called a fixed point of the stochastic dynamics (\ref{eq:general_update}) if, for any $ \xi $, we have $ \nabla f_{\xi}(x^\star) = 0 $.
\end{definition}
\vspace{0.5mm}
\begin{definition} (Linearized \sassha). 
    Let $x^\star$ be the fixed point of interest and assume $ f(x^\star) = 0 $.
    Consider the quadratic approximation of $ f $ near $ x^\star $: $ f_{\xi}(x) \approx \frac{1}{2} (x - x^\star)^\top H_{\xi} (x - x^\star)$  with $ H_{\xi} = \nabla^2 f_{\xi_t} (x^\star)$.
    The corresponding linearized \sassha is given by 
    \begin{equation} \label{eq:linearized_sassha}
        x_{t+1} = x_t
        -\eta \frac{1}{ \sqrt{\operatorname{diag} ( H_{\xi_t }) } + \epsilon } \odot H_{\xi_t} ( \Tilde{x}_t - x^\star )
    \end{equation}
    where $ \Tilde{x}_t = x_t + \rho H_{\xi_t} ( x_t - x^\star ) $ is the linearized perturbed point.
\end{definition}
\vspace{0.5mm}

For brevity, we define
\begin{equation}
   d_{\xi_t} := \sqrt{ \operatorname{diag} \bigl( H_{\xi_t} \bigr) } + \epsilon , \nonumber
\end{equation}
Then, the update (\ref{eq:linearized_sassha}) can be rewritten as
\begin{align} \label{eq:simple_linearized_rule}
    x_{t+1} = x_t - 
    \eta \frac{1}{ d_{\xi_t} } \odot H_{\xi_t} ( \Tilde{x}_t - x^\star )
    ,
\end{align}

\begin{remark} 
    In a neighborhood of $x^\star$, the inverse scaling vector is uniformly upper bounded by 
    \begin{equation}\label{eq:lsa_lemma}
        \frac{1}{d_{ \xi_t }}  \preceq \frac{1}{\epsilon} \mathbf{1}
    \end{equation}
    where $\mathbf{1}$ is the vector of all ones.
\end{remark}

\begin{definition}
    (Linear stability). Consider a fixed point $ x^\star $ of linearized stochastic dynamic such as (\ref{eq:linearized_sassha}). 
    We say that $x^\star$ is \emph{linearly stable} if there exists a constant $ C $ such that 
    \begin{equation}
    \mathbb{E} \bigl[ \lVert x_t - x^\star \rVert^2 \bigr] 
    \leq C \lVert x_0 - x^\star \rVert^2, \text{ for all } t > 0 \nonumber 
    \end{equation}
\end{definition}

\begin{theorem}
Assume without loss of generality that $ x^\star = 0 $.
Then, the fixed point $ x^\star $ of \sassha is linearly stable if:
\begin{equation}
\lambda_{max} \Bigl( 
\bigl( I - \frac{\eta}{ \epsilon } H - \frac{\eta \rho}{ \epsilon } H^2 \bigr)^2
+ \frac{ \eta^2 -2 \eta \rho \epsilon }{ \epsilon^2 } \, \bigl( \mathbb{E} H_{\xi_t}^2 - H^2 \bigr) 
+ \frac{2 \eta^2 \rho }{ \epsilon^2 } \, \bigl( \mathbb{E} H_{\xi_t}^3 - H^3 \bigr) 
+ \frac{ \eta^2 \rho^2 }{ \epsilon^2 } \, \bigl( \mathbb{E} H_{\xi_t}^4 - H^4 \bigr)
\Bigr) \leq 1. \nonumber
\end{equation}

\begin{proof}
Our goal is to obtain a bound of the form $ \mathbb{E} \lVert x_t \rVert^{2} \leq C \lVert x_0 \rVert^{2} $, for some constant $C$.
We begin by substituting (\ref{eq:simple_linearized_rule}) into $ \mathbb{E} \lVert x_{t+1} \rVert^{2} $ and proceed to expand the terms as follows:
\begin{align}
\mathbb{E} \Bigl[ \lVert x_{t+1} \rVert^2 \Bigr | x_t \Bigr]
&= x_t^{\top} \mathbb{E} \Bigl[ 
\bigl(I - \eta \operatorname{diag} ( \frac{1}{d_{\xi_t}} ) H_{\xi_t} (I+\rho\,H_{\xi_t}) \bigr)^{\top} \bigl(I - \eta \operatorname{diag} ( \frac{1}{d_{\xi_t}} ) H_{\xi_t} (I+\rho\,H_{\xi_t}) \bigr) | x_t \Bigr] \, x_t \nonumber \\
&\overset{\text{( \ref{eq:lsa_lemma} )}}{\leq} x_t^{\top} \mathbb{E} \Bigl[ 
\bigl(I - \eta \frac{1}{\epsilon} H_{\xi_t} (I+\rho\,H_{\xi_t}) \bigr)^{\top} \bigl(I - \eta \frac{1}{\epsilon} H_{\xi_t} (I+\rho\,H_{\xi_t}) \bigr) 
| x_t \Bigr] \, x_t \nonumber \\
&= x_t^{\top} \mathbb{E} \Bigl[ 
I - \frac{2\eta}{\epsilon} H_{\xi_t} (I+\rho\,H_{\xi_t}) 
+ \frac{ \eta^2 }{ \epsilon^2 } H_{\xi_t}^2 (I+\rho\,H_{\xi_t})^2 
| x_t \Bigr] \, x_t \nonumber \\
&= x_t^{\top} \mathbb{E} \Bigl[ 
I - \frac{2\eta}{\epsilon} H_{\xi_t} 
- \frac{2 \eta \rho}{\epsilon} \, H_{\xi_t}^2 
+ \frac{\eta^2}{\epsilon^2} H_{\xi_t}^2 
+ \frac{2 \eta^2 \rho }{ \epsilon^2 }  H_{\xi_t}^3 
+ \frac{\eta^2 \rho^2 }{ \epsilon^2 }  H_{\xi_t}^4 | x_t \Bigr] \, x_t \nonumber \\
&= x_t^{\top} \Bigl( 
I - \frac{ 2 \eta }{ \epsilon } H 
+ \frac{ \eta^2 -2 \eta \rho \epsilon }{ \epsilon^2 } \, H^2
+ \frac{ 2 \eta^2 \rho }{ \epsilon^2 } \, H^3
+ \frac{ \eta^2 \rho^2 }{ \epsilon^2 } \, H^4  \nonumber \\ 
&+ \frac{ \eta^2 -2 \eta \rho \epsilon }{ \epsilon^2 } \, \bigl( \mathbb{E} H_{\xi_t}^2 - H^2 \bigr) 
+ \frac{2 \eta^2 \rho }{ \epsilon^2 } \, \bigl( \mathbb{E} H_{\xi_t}^3 - H^3 \bigr) 
+ \frac{ \eta^2 \rho^2 }{ \epsilon^2 } \, \bigl( \mathbb{E} H_{\xi_t}^4 - H^4 \bigr)
\Bigr) \, x_t  \nonumber \\
&= x_t^{\top} \Bigl( 
\bigl( I - \frac{\eta}{\epsilon} H - \frac{\eta \rho}{ \epsilon } H^2 \bigr)^2
+ \frac{ \eta^2 -2 \eta \rho \epsilon }{ \epsilon^2 } \, \bigl( \mathbb{E} H_{\xi_t}^2 - H^2 \bigr) 
+ \frac{2 \eta^2 \rho }{ \epsilon^2 } \, \bigl( \mathbb{E} H_{\xi_t}^3 - H^3 \bigr) 
+ \frac{ \eta^2 \rho^2 }{ \epsilon^2 } \, \bigl( \mathbb{E} H_{\xi_t}^4 - H^4 \bigr)
\Bigr) \, x_t  \nonumber
\end{align}
Since for any $ x $ and any matrix $ A $ the inequality $ x^{\top} Ax \leq \lambda_{max} (A) \lVert x \rVert^2$ with $ \lambda_{max}(A) $ denoting the maximum eigenvalue of $A$, holds true, applying this inequality and taking the total expectation yields the following:
\begin{equation}
\mathbb{E}\lVert x_{t+1} \rVert^2 \leq \lambda_{max} \Bigl( 
\bigl( I - \frac{\eta}{ \epsilon } H - \frac{\eta \rho}{ \epsilon } H^2 \bigr)^2
+ \frac{ \eta^2 -2 \eta \rho \epsilon }{ \epsilon^2 } \, \bigl( \mathbb{E} H_{\xi_t}^2 - H^2 \bigr) 
+ \frac{2 \eta^2 \rho }{ \epsilon^2 } \, \bigl( \mathbb{E} H_{\xi_t}^3 - H^3 \bigr) 
+ \frac{ \eta^2 \rho^2 }{ \epsilon^2 } \, \bigl( \mathbb{E} H_{\xi_t}^4 - H^4 \bigr)
\Bigr) \, \mathbb{E} \lVert x_{t} \rVert^2  \nonumber
\end{equation}
Recursively applying this bound gives
\begin{equation}
\mathbb{E} \lVert x_{t} \rVert^2
= \lambda_{max} \Bigl( 
\bigl( I - \frac{\eta}{ \epsilon } H - \frac{\eta \rho}{ \epsilon } H^2 \bigr)^2
+ \frac{ \eta^2 -2 \eta \rho \epsilon }{ \epsilon^2 } \, \bigl( \mathbb{E} H_{\xi_t}^2 - H^2 \bigr) 
+ \frac{2 \eta^2 \rho }{ \epsilon^2 } \, \bigl( \mathbb{E} H_{\xi_t}^3 - H^3 \bigr) 
+ \frac{ \eta^2 \rho^2 }{ \epsilon^2 } \, \bigl( \mathbb{E} H_{\xi_t}^4 - H^4 \bigr)
\Bigr)^t \, \mathbb{E} \lVert x_{0} \rVert^2 \nonumber
\end{equation}
Here, we can see that $ x^* $ is linearly stable if
\begin{equation}
\lambda_{max} \Bigl( 
\bigl( I - \frac{\eta}{ \epsilon } H - \frac{\eta \rho}{ \epsilon } H^2 \bigr)^2
+ \frac{ \eta^2 -2 \eta \rho \epsilon }{ \epsilon^2 } \, \bigl( \mathbb{E} H_{\xi_t}^2 - H^2 \bigr) 
+ \frac{2 \eta^2 \rho }{ \epsilon^2 } \, \bigl( \mathbb{E} H_{\xi_t}^3 - H^3 \bigr) 
+ \frac{ \eta^2 \rho^2 }{ \epsilon^2 } \, \bigl( \mathbb{E} H_{\xi_t}^4 - H^4 \bigr)
\Bigr) \leq 1. \nonumber
\end{equation}
\end{proof}
\end{theorem}

\newpage

\section{Experiment Setting} \label{app:hypersearch}
Here, we describe our experiment settings in detail. 
We evaluate \sassha against AdaHessian \citep{adahessian}, Sophia-H \citep{sophia}, Shampoo \citep{gupta2018shampoo}, SGD, AdamW \citep{loshchilov2018decoupled}, and SAM \citep{sam} across a diverse set of vision and language tasks.
Across all evaluations except for language finetuning, we set lazy Hessian update interval to $k = 10$ for \sassha. 
In fact, Sophia-H also supports lazy Hessian updates, but \citet{sophia} reports that it achieves the best performance when $k = 1$, without lazy updating.
Since our goal is to demonstrate that \sassha{} exhibits better generalization than existing approximate second-order methods, we compare it with Sophia-H without lazy Hessian updating $k = 1$, ensuring that the algorithm is assessed under its optimal configuration.

\subsection{Image Classification}
\paragraph{CIFAR}
We trained ResNet-20 and ResNet-32 on the CIFAR datasets for $160$ epochs and Wide-ResNet28-10 for $200$ epochs. 
Only standard inception-style data augmentations, such as random cropping and horizontal flipping, were applied, without any additional regularization techniques or extra augmentations.
We used standard cross-entropy without label smoothing as a loss function.
Also, we adopted a multi-step decay learning rate schedule. 
Specifically, for ResNet-20 and ResNet-32, the learning rate was decayed by a factor of $0.1$ at epochs $80$ and $120$. 
For Wide-ResNet28-10, the learning rate was decayed by a factor of $0.2$ at epochs $60$, $120$ and $160$. 
The exponential moving average hyperparameters were set to $\beta_1 = 0.9$ and $\beta_2 = 0.999$.
All experiments were conducted with a batch size of 256.
The hyperparameter search space for each method is detailed in \cref{tab:tuning_cifar}.

\begin{table}[ht!]
\renewcommand{\arraystretch}{2}
\centering
\resizebox{\textwidth}{!}{%
\begin{tabular}{lcc|cccc|c}
    \toprule
    \rowcolor{lgray} 
    Method        
    & \sassha
    & \msassha
    & AdaHessian 
    & Sophia-H 
    & AdamW / SGD
    & SAM
    & shampoo \\
    
    \midrule
    
    Learning Rate    & 
    \multicolumn{2}{c|}{ $\Bigl\{0.3, 0.15, 0.03, 0.015 \Bigr\}$ } 
    & \multicolumn{4}{c|}{ $\Bigl\{ 0.3, 0.15, 0.1, 0.03, 0.015, 0.01, 0.003, 0.001, 0.0003, 0.0001  \Bigr\}$ } 
    & $\Bigl\{ \substack{1.5, 1.4, 1.3, 1.2, 1.1, 1, 0.9, 0.8, 0.7, 0.6, \\ 0.5, 0.4, 0.3, 0.2, 0.1, 0.01, 0.04, 0.004} \Bigr\} $ \\ 
    
    \midrule
    Weight Decay     & \multicolumn{7}{c}{$\Bigl\{ \text{1e-3}, \text{5e-4}, \text{1e-4}, \text{5e-5}, \text{1e-5}, \text{5e-6}, \text{1e-6} \Bigr\}$}   \\ 
    
    \midrule
    
    Perturbation radius $\rho$      & $\Bigl\{0.1, 0.15, 0.2, 0.25 \Bigr\}$ 
    & $\Bigl\{ 0.1, 0.2, 0.3, 0.6, 0.8 \Bigr\}$ 
    & - 
    & - 
    & - 
    & $\Bigl\{ \substack{0.01, 0.05, 0.1, 0.15, 0.2, 0.25, 0.3,\\ 0.35, 0.4, 0.45, 0.5, 0.55, 0.6} \Bigr\}$
    & - \\ 
    
    \midrule
    
    Clipping-threshold   & - 
    & - 
    & - 
    & $\Bigl\{ \substack{0.1, 0.05, 0.01, 0.005,\\ 0.001, 0.0005, 0.0001} \Bigr\}$ 
    & -   
    & - 
    & - \\ 
    
    \midrule
    
    Damping   & - 
    & - 
    & -  
    & - 
    & -
    & -
    & 1e-$\Bigl\{2, 3, 4, 6, 8 \Bigr\}$ \\ 
    
    \midrule
    
    Hessian Update Interval $k$ & 10 
    & 10 
    & 1 
    & 1 
    & - 
    & -
    &  1\\
    
    \midrule
    
    learning rate schedule & \multicolumn{7}{c}{\text{Multi-step decay}} \\
    
    \bottomrule
    \end{tabular}
    }
    \caption{Hyperparameter search space for CIFAR datasets}
    \label{tab:tuning_cifar}
    
\end{table}

\paragraph{ImageNet}

We trained ResNet-50 and \textit{plain Vision Transformer} (plain ViT) \citep{beyer2022better} for 90 epochs.
Remarkably, plain ViT converges in just 90 epochs on ImageNet, attaining performance comparable to the original ViT trained for 300 epochs \citep{beyer2022better}.
This faster convergence allows us to efficiently assess whether \sassha can enhance the generalization in ViT architectures. 
Consistent with our CIFAR training settings, we applied only standard inception-style data augmentations and used standard cross-entropy as a loss function. 
For ResNet-50, we adopted a multi-step decay learning rate schedule, reducing the learning rate by a factor of 0.1 at epochs 30 and 60. 
However, AdaHessian could not be trained with a multi-step decay schedule; therefore, as recommended by \citet{adahessian}, we employed a plateau decay schedule instead. 
For Vision Transformer training, following \citet{chenvision}, we used a cosine learning rate schedule with an 8-epoch warm-up phase. 
Additionally, the exponential moving average hyperparameters $ \beta_1 $ and $ \beta_2 $ were set to $0.9$ and $0.999$ respectively. 
We used a batch size of 256 for ResNet50 and 1024 for ViT. 
The hyperparameter search spaces for each methods used during training on the ImageNet dataset are detailed in \cref{tab:tuning_imag}.

\begin{table}[ht!]
    \renewcommand{\arraystretch}{2}
    \centering
    \resizebox{\textwidth}{!}{%
    \begin{tabular}{lccc|cccc}
    \toprule
    \rowcolor{lgray} 
    methods     
    & \sassha
    & \msassha
    & AdaHessian 
    & Sophia-H
    & AdamW / SGD 
    & SAM \\
    \midrule
    Learning Rate    &  $\Bigl\{0.6, 0.3, 0.15 \Bigr\}$
    &  $\Bigl\{0.6, 0.3, 0.15 \Bigr\}$
    &  $\Bigl\{0.6 0.3, 0.15 \Bigr\}$
    &  \multicolumn{3}{c}{$ \Bigl\{0.4, 0.2, 0.1, 0.04, 0.02, 0.01, 0.001 \Bigr\} $}  \\ 
    
    \midrule
    
    Weight Decay     & \multicolumn{6}{c}{$\Bigl\{ \text{1e-3}, \text{5e-4}, \text{1e-4}, \text{5e-5}, \text{1e-5} \Bigr\}$}   \\ 
    
    \midrule
    
    Perturbation radius $\rho$       & $\Bigl\{ 0.1, 0.15, 0.2, 0.25 \Bigr\}$ 
    & $\Bigl\{ 0.1, 0.2, 0.4, 0.8 \Bigr\}$ 
    & -
    & -
    & -
    & $\Bigl\{ 0.1, 0.15, 0.2, 0.25, 0.3 \Bigr\} $ \\ 
    
    \midrule
    
    Clipping-threshold   & - 
    & - 
    & - 
    & $\Bigl\{ 0.1, 0.05, 0.01, 0.005, 0.001, 0.0005, 0.0001 \Bigr\}$  
    & -
    & - \\ 
    
    \midrule
    
    Hessian Update Interval $k$ & 10 
    & 10
    & 1
    & 1
    & - 
    & - \\ 
    
    \bottomrule
    
    \end{tabular}
    }
    \caption{Hyperparameter search space for ImageNet}
    \label{tab:tuning_imag}
    
\end{table}

\subsection{Language}
\paragraph{Language Pretraining}
Following the training settings introduced in \citet{gomes2024adafisher}, we conducted experiments on a mini GPT-1 model using the Wikitext-2 dataset. 
This scaled-down version of GPT-1 maintains essential modeling capabilities while reducing computational demands. 
We trained the model with three methods: \sassha{}, \msassha{}, and Sophia-H. 
The hyperparameter tuning spaces for these methods are summarized in \cref{tab:tuning_pretraining}. 
For other methods not listed in the table, we directly reported the results from \citet{gomes2024adafisher}.

\begin{table}[ht!]
    \renewcommand{\arraystretch}{2}
    \centering
    \resizebox{0.8\textwidth}{!}{%
    \begin{tabular}{lccc}
    \toprule
    \rowcolor{lgray} 
    methods        & \sassha{} / \msassha{}     
    & Sophia-H 
    & SAM        \\
    
    \midrule
    
    Learning Rate     & $\Bigl\{0.15, 0.1, 0.03, 0.01, 0.003, 0.0015 \Bigr\}$      
    &   $\Bigl\{ \text{1e-2}, \text{5e-3}, \text{1e-3}, \text{5e-4}, \text{1e-4}, \text{5e-5}, \text{1e-5} \Bigr\}$ 
    &   $\Bigl\{ \text{1e-2}, \text{1e-3}, \text{1e-4}, \text{1e-5}, \text{1e-6} \Bigr\}$                      \\ 
    
    \midrule
    
    Weight Decay      & \multicolumn{3}{c}{ $ \text{1e-}\{1, 2, 3, 4, 5, 6, 7, 8\}$} \\ 
    
    \midrule
    
    Perturbation radius $\rho$    
    &  $ \text{1e-}\{1, 2, 3, 4, 5\} $ 
    & - 
    &  $ \text{1e-}\{1, 2, 3, 4, 5, 6, 7, 8 \} $ \\ 
    
    \midrule
    
    Clipping-threshold   
    & - 
    & $\Bigl\{ \text{1e-1}, \text{5e-2}, \text{1e-2}, \text{5e-3}, \text{1e-3}, \text{5e-4}, \text{1e-4} \Bigr\}$ 
    & - \\ 
    
    \midrule
    
    Hessian Update Interval $k$  & 10  
    & 1 
    & - \\ 
    
    \midrule
    
    Epochs        & \multicolumn{2}{c}{$50$} 
    & 55 \\ 
    
    \bottomrule
    \end{tabular}
    }
    \caption{Hyperparameter search space for language pretraining}
    \label{tab:tuning_pretraining}
    
\end{table}

\paragraph{Language Finetuning} 
We utilized a pretrained SqueezeBERT \citep{iandola2020squeezebert} from the HuggingFace Hub \citep{wolf2020huggingfaces}.
We set the batch size to 16, the maximum sequence length to 512, and the dropout rate to 0.
The number of training epochs varied depending on the specific GLUE task: 5 epochs for MNLI, QQP, QNLI, and SST-2; 10 epochs for STS-B, MRPC, and RTE.
Additionally, We adopted a polynomial learning rate decay scheduler. 
The detailed hyperparameter search spaces are presented in \cref{tab:tuning_fine}.

\begin{table}[ht]
    \renewcommand{\arraystretch}{2}
    \centering
    \resizebox{0.8\linewidth}{!}{%
    \begin{tabular}{lccccc}
    \toprule
    \rowcolor{lgray} 
    methods        & \sassha{} / \msassha{} 
    & Sophia-H 
    & AdaHessian 
    & AdamW
    & SAM          \\
    \midrule
    Learning Rate         &   \multicolumn{5}{c}{ $ \text{1e-}\{1, 2, 3, 4, 5, 6, 7, 8\}$} \\ 
    
    \midrule
    
    Weight Decay          &   \multicolumn{5}{c}{ $ \text{1e-}\{1, 2, 3, 4, 5, 6, 7, 8\}$}  \\ 
    
    \midrule
    
    Perturbation radius $\rho$       &  $ \text{1e-}\{2, 3, 4, 5\}$  
    & - 
    & - 
    & -
    & $ \text{1e-}\{1, 2, 3, 4, 5, 6, 7, 8 \}$ \\ 
    
    \midrule
    
    Clipping-threshold  & - 
    & $\Bigl\{ \substack{0.1, 0.05, 0.01, 0.005,\\ 0.001, 0.0005, 0.0001} \Bigr\}$ 
    & - 
    & - \\ 
    
    \midrule
    
    Hessian Update Interval $k$ & 1 
    & 1 
    & 1 
    & -
    & - \\ 
    
    \bottomrule
    
    \end{tabular}
    }
    \caption{Hyperparameter search space for language finetuning}
    \label{tab:tuning_fine}
    
\end{table}

\subsection{Label Noise}
We introduced label noise by randomly corrupting a fraction of the training data at rates of 20\%, 40\%, and 60\%.
Using this setup, we trained ResNet-32 for 160 epochs with a batch size of 256.
We adopted a multi-step decay learning rate schedule, reducing the learning rate by a factor of 0.1 at epochs 80 and 120.
The specific hyperparameters explored during these experiments are detailed in \cref{tab:tuning_noise}.

\begin{table}[ht!]
\renewcommand{\arraystretch}{2}
\centering
\resizebox{\textwidth}{!}{%
    \begin{tabular}{lcccccc}
    \toprule
    \rowcolor{lgray} 
    Methods        & \sassha{} & \msassha{} 
    & Sophia-H 
    & AdaHessian 
    & SAM
    & SGD \\
    \midrule
    Learning Rate    & \multicolumn{6}{c}{$ \Bigl\{ 0.3, 0.15, 0.1, 0.03, 0.015, 0.01, 0.003, 0.0015, 0.001 \Bigr\} $}   \\ 
    
    \midrule
    
    Weight Decay     & \multicolumn{6}{c}{$\Bigl\{ \text{1e-3}, \text{5e-4}, \text{5e-5}, \text{1e-5}, \text{5e-6}, \text{1e-6} \Bigr\}$}   \\ 
    
    \midrule
    
    Perturbation radius $\rho$       & $\Bigl\{ 0.25, 0.2, 0.15, 0.1  \Bigr\}$ & $\Bigl\{ 0.8, 0.6, 0.3, 0.2, 0.1 \Bigr\}$ 
    & - 
    & - 
    & $\Bigl\{ 0.3, 0.25, 0.2, 0.15, 0.1, 0.05, 0.02, 0.01, 0.002, 0.001      \Bigr\}$ 
    & - \\ 
    
    \midrule
    
    Clipping-threshold   & - 
    & - 
    & $\Bigl\{ \substack{0.1, 0.05, 0.01, 0.005,\\ 0.001, 0.0005, 0.0001} \Bigr\}$  & - & - & - \\ 
    
    \midrule
    
    Hessian Update Interval $k$ & 10 
    & 10 
    & 1 
    & 1 
    & - 
    & - \\ \bottomrule
    \end{tabular}
    }
    \caption{Hyperparameter search space for label noise experiments}
    \label{tab:tuning_noise}
\end{table}

\section{More Ablations} \label{app:add_ablation}

\subsection{Square Root Function}\label{app:sqrt_alternatives}

\begin{table}[ht!]
    \centering
    \caption{
    Comparison of square root against damping and clipping.
    }
    \vskip 0.1in
    \resizebox{0.55\linewidth}{!}{
        \begin{tabular}{lccc}
        \toprule
         & \multicolumn{2}{c}{CIFAR-10} 
         & \multicolumn{1}{c}{CIFAR-100}  \\
         
         \cmidrule(l{3pt}r{3pt}){2-3} \cmidrule(l{3pt}r{3pt}){4-4} 
         
         & \multicolumn{1}{c}{ ResNet-20 } & \multicolumn{1}{c}{ ResNet-32 }   & \multicolumn{1}{c}{ ResNet-32 } \\  
         
         \midrule
           
        Clipping   & $ 92.78 _{\textcolor{black!60}{\pm 0.18} } $ 
        & $ 93.80 _{\textcolor{black!60}{\pm 0.16} } $ 
        & $ 69.47 _{\textcolor{black!60}{\pm 0.20} } $ \\
              
        Damping   & $  92.74 _{\textcolor{black!60}{\pm 0.06} } $ 
        & $  93.68 _{\textcolor{black!60}{\pm 0.29} } $ 
        & $ 71.27 _{\textcolor{black!60}{\pm 0.43} } $ \\
        
        \midrule
        
        \rowcolor{green!20} Square root (\sassha)    & $ \textbf{92.98} _{\textcolor{black!60}{\pm 0.05}} $ 
        & $ \textbf{94.09} _{\textcolor{black!60}{\pm 0.24} } $ 
        & $ \textbf{72.14} _{\textcolor{black!60}{\pm 0.16} } $ \\
        \bottomrule
        \end{tabular}
    }
    \label{tab:sqrt generalization}    
\end{table}

We conduct an ablation study to support our use of the square rooted preconditioner in \sassha, comparing it to other alternatives to stabilize the preconditioner such as damping or clipping.
We search damping and clipping hyperparameters over $\{10^{-4}, 10^{-6}, 10^{-8}, 10^{-12}\}$ and $\{0.1, 0.05, 0.01, 0.005, 0.001, 0.0005, 0.0001\}$, respectively.
We note that the square-root employed in \sassha does not require such extensive hyperparamter search.
The results are presented in \cref{tab:sqrt generalization}.

Our experiments demonstrate that the square rooted preconditioner achieves higher validation accuracy than those with damping or clipping, even with a three times smaller hyperparameter search budget.
We provide two possible explanations for this observation.
First, taking square root preserves the relative scale among Hessian elements while smoothly amplifying near-zero entries in the denominator (\ie, $h < \sqrt{h} $ if $0 < h < 1$). 
This property is particularly valuable during sharpness minimization, where the overall magnitude of the Hessian components tends to be small.
In such cases, even small differences between Hessian elements may carry nontrivial curvature information.
Applying Square root can help retain this relative structure while also reducing numerical instability caused by underestimated curvature.
In contrast, both damping and clipping modify the Hessian by entirely shifting or abruptly replacing values based on a predefined and fixed threshold criterion.
As a result, when the Hessian is generally small due to sharpness minimization, informative dimensions may fall below the threshold, removing potentially critical variations and hence deteriorating the quality of preconditioning.
This behavior can also increase the sensitivity to the choice of the threshold hyperparameter.
Second, square-rooted preconditioner can be interpreted as the result of a geometric interpolation between the identity matrix $I$ and $H^\alpha$. 
This interpolation has been demonstrated to enable selecting an optimal preconditioner that balances the bias and the variance of the population risk, thereby minimizing generalization error \citep{amari2021when}.
In general, $\alpha=1/2$ (i.e., square root) has consistently shown moderate performance across various scenarios \citep{{amari2021when, duchi2011adaptive, kingma2014adam}}.

\newpage 

\subsection{Absolute Value Function} \label{sec:abs_ablation}

\begin{wrapfigure}{r}{0.49\linewidth}
    \centering
    \vspace{-1em}
    \begin{subfigure}{0.45\linewidth}
        \centering
        \includegraphics[width=\linewidth]{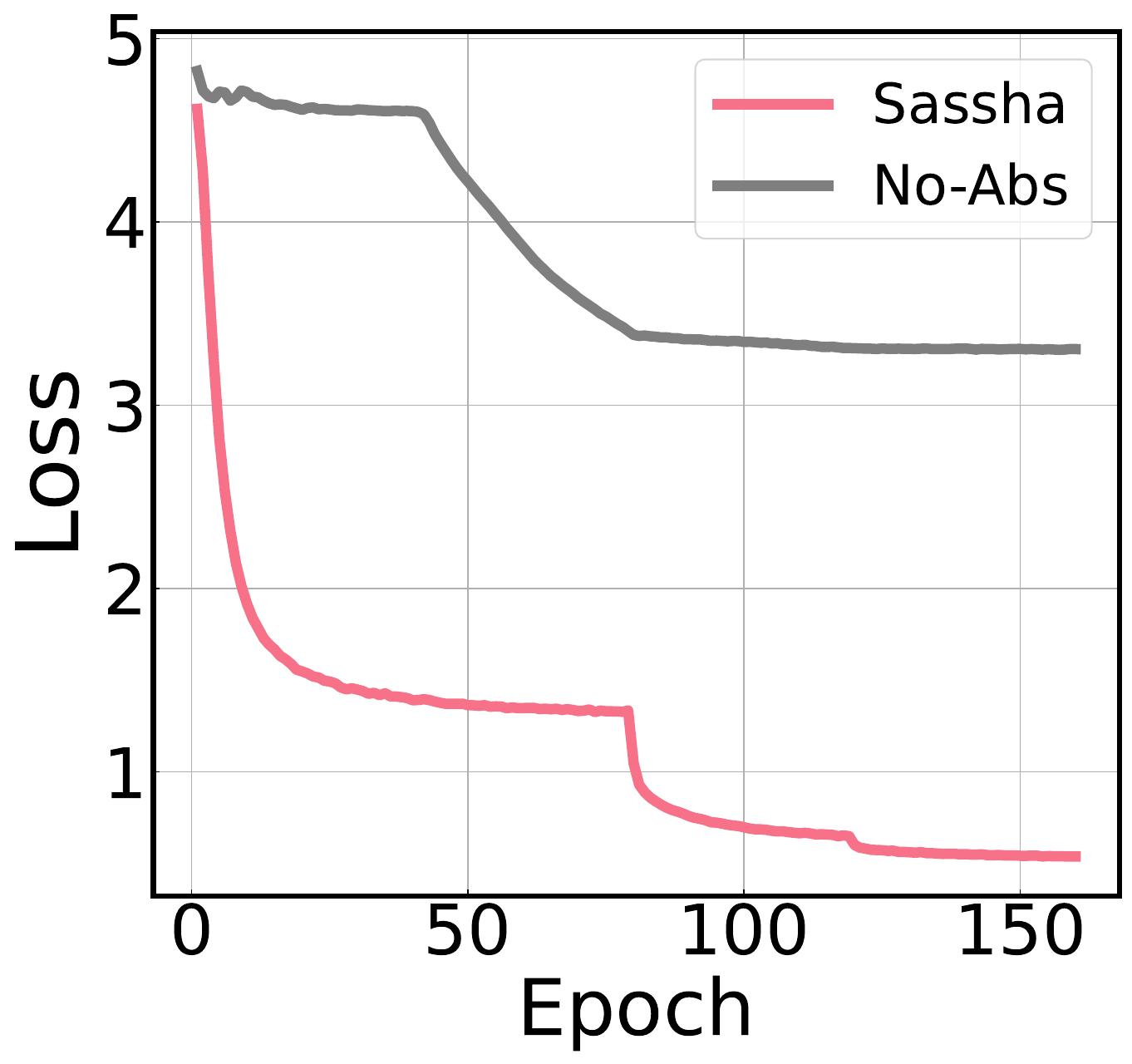}
        \caption{Train loss}
        \label{fig:abs_ablation_val}
    \end{subfigure}
    \begin{subfigure}{0.52\linewidth}
        \centering
        \includegraphics[width=\linewidth]{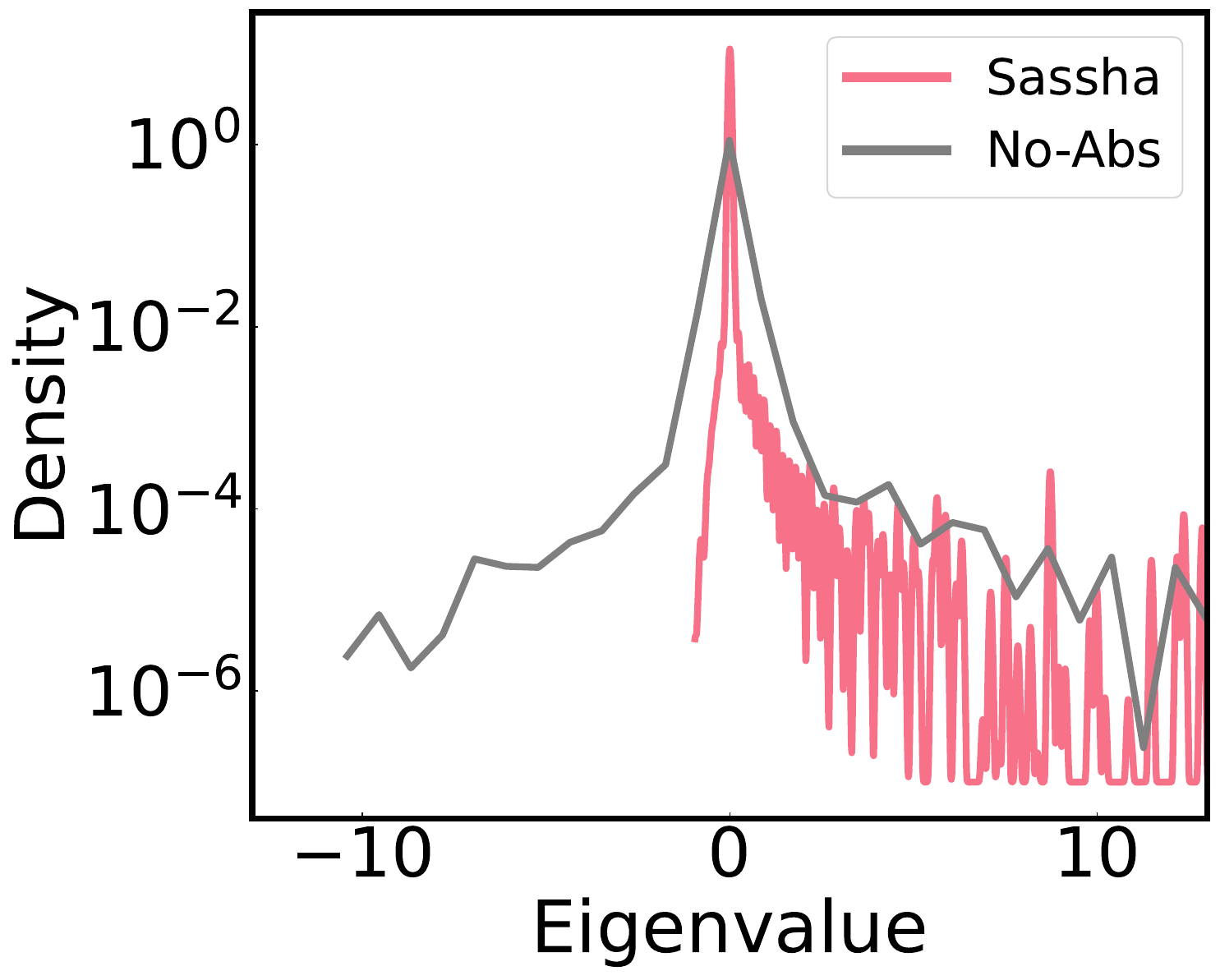}
        \caption{Hess. eigenspectrum}
        \label{fig:abs_ablation_dist}
    \end{subfigure}
    \caption{Effect of the absolute function on the training loss and the Hessian eigenspectrum of the found solution of \sassha on ResNet-32/CIFAR-10. Without the absolute function, \sassha converges to sub-optimal saddle point.}
    \label{fig:abs_ablation}
    
\end{wrapfigure}

We observe how the absolute function influences the training process to avoid convergence to a critical solution that could result in sub-optimal performance.
We train ResNet-32 on CIFAR-100 using \sassha without the absolute function (\texttt{No-Abs}) and compare the resulting training loss to that of the original \sassha.
We also plot the Hessian eigenspectrum of the found solution via the Lanczos algorithm \citep{yao2020pyhessian} to determine whether the found solution corresponds to a minimum or a saddle point.
The results are illustrated in \cref{fig:abs_ablation}.
We can see that without the absolute function, the training loss converges to a sub-optimal solution, where the prevalent negative values in the diagonal Hessian distribution indicate it as a saddle point.
This shows the necessity of the absolute function for preventing convergence to these critical regions.

\section{Validation Loss Curve for Vision Task} \label{app:valloss}

\begin{figure}[ht!]
    \centering
    \includegraphics[width=0.27\linewidth]{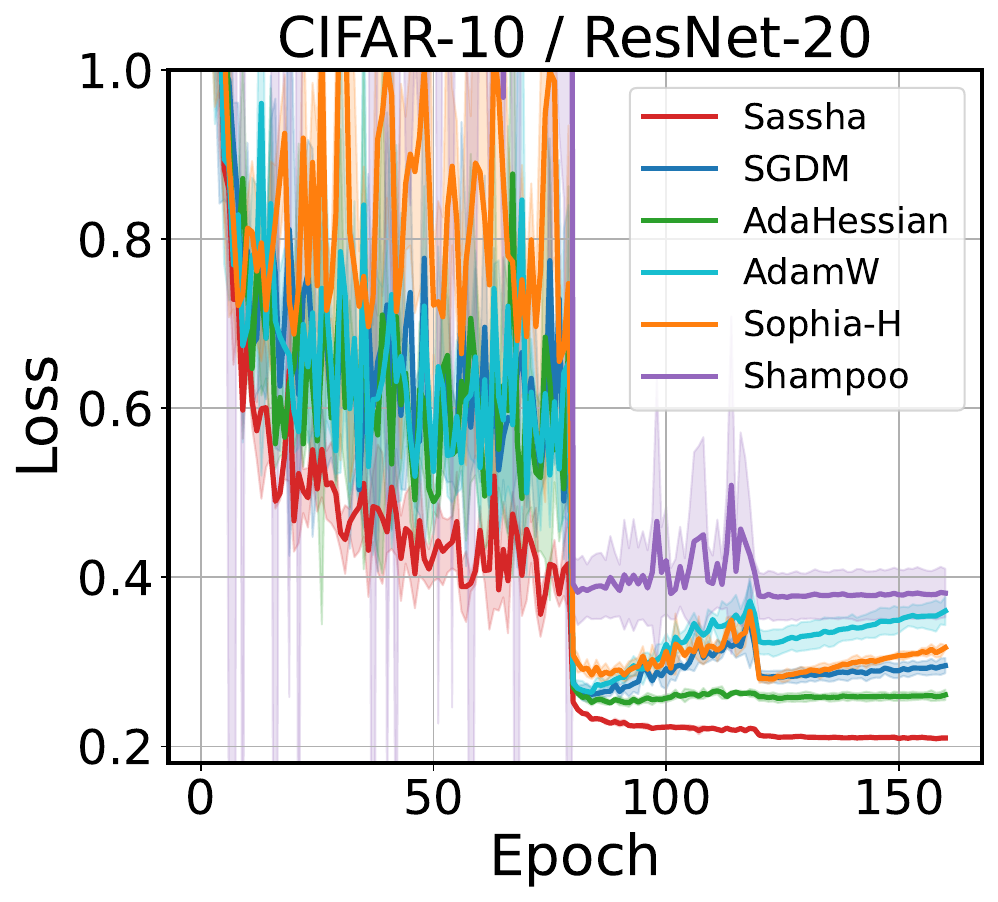}
    \includegraphics[width=0.27\linewidth]{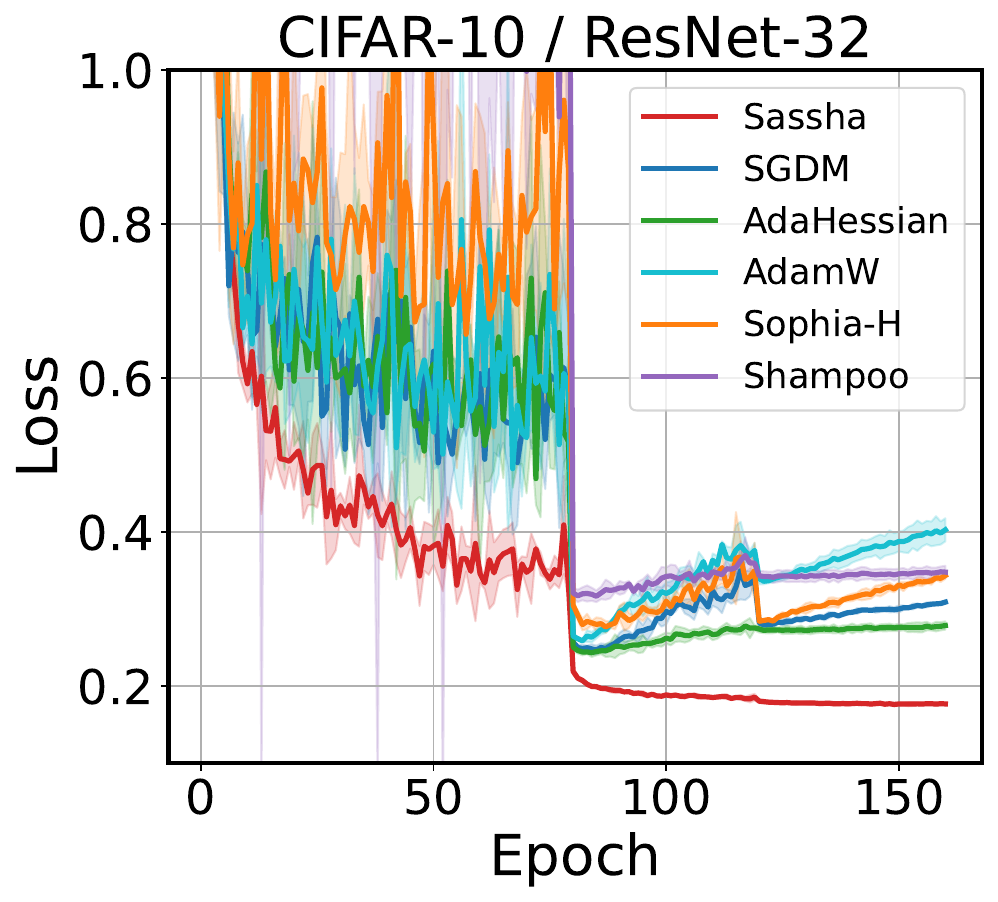}
    \includegraphics[width=0.27\linewidth]{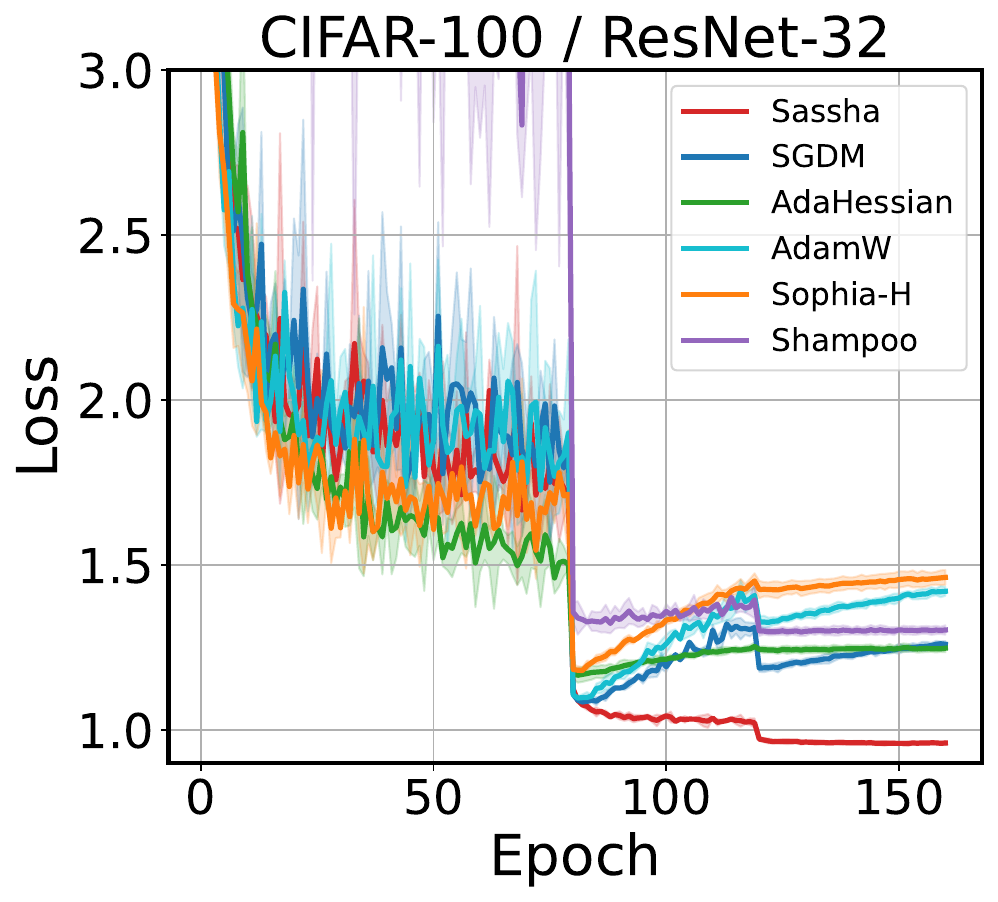}
    \includegraphics[width=0.27\linewidth]{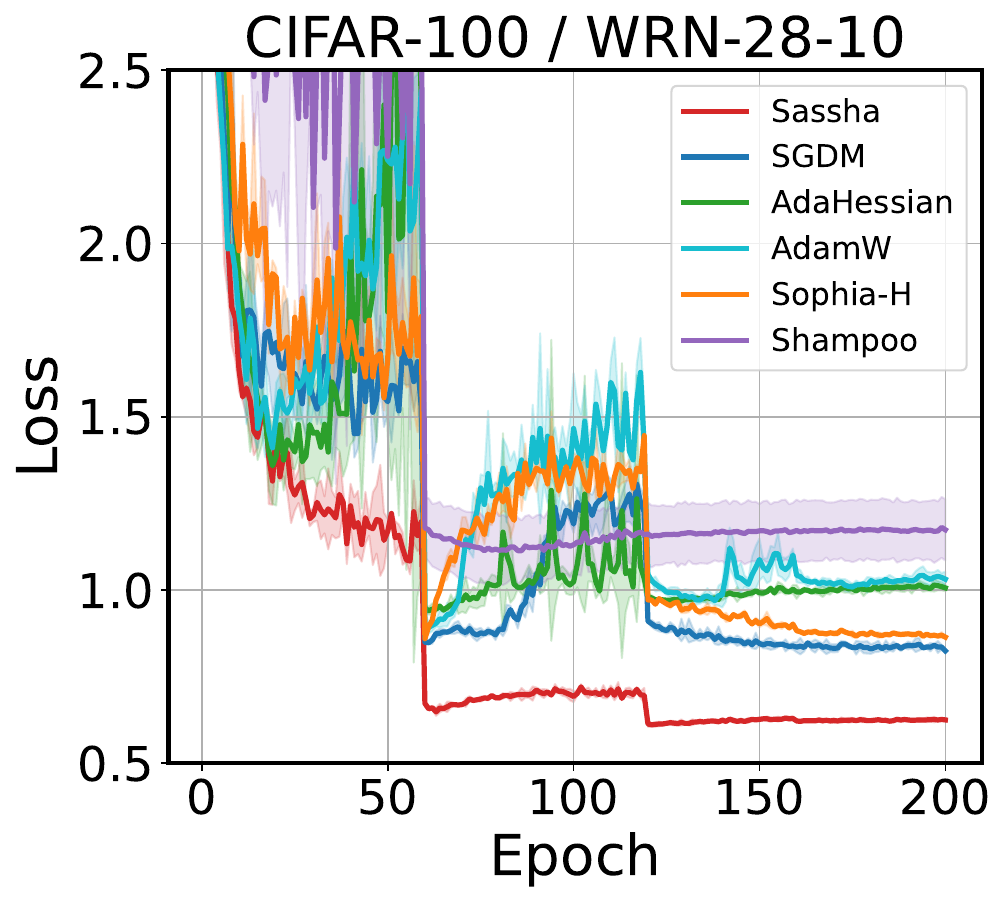}
    \includegraphics[width=0.27\linewidth]{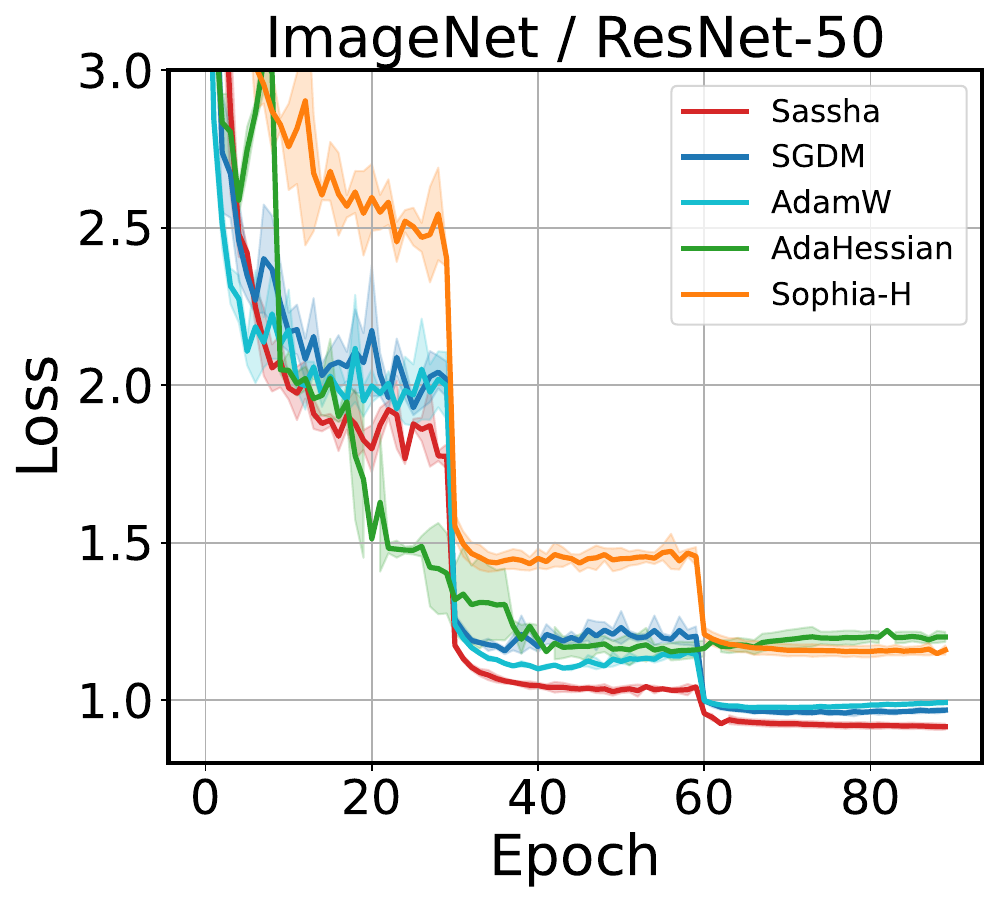}
    \includegraphics[width=0.27\linewidth]{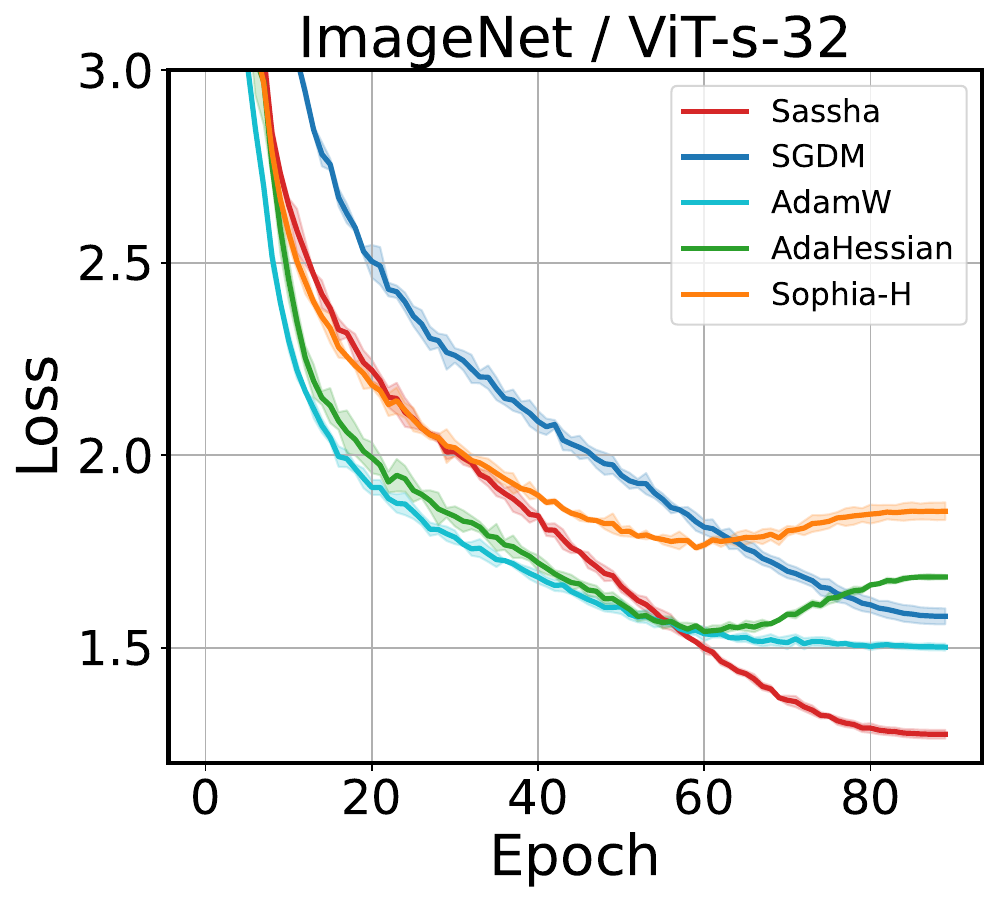}
    \caption{Validation loss curve of \sassha, SGD, AdaHessian, AdamW, and Sophia-H on various image classification models and tasks. \sassha outperforms all first-order and second-order baseline methods.}
    \label{fig:valloss}
\end{figure}

The experimental results in \cref{fig:valloss} demonstrate better generalization capability of \sassha over the related methods. 
Across all datasets and model architectures, our method consistently achieves the lowest validation loss, indicative of its enhanced ability to generalize from training to validation data effectively.
This robust performance of \sassha underscores its potential as a leading optimization method for various deep learning applications, particularly in image classification.

\newpage

\section{Comparison with First-order Baselines with Given More Training Budget than \sassha }\label{app:comp_fo_fair}

We train SGD and AdamW for twice as many epochs as \sassha and compare their final validation accuracies.
The results are presented in \cref{tab:comp_fo_fair}.
Despite this extended training budget, these first-order methods fall short of the performance attained with \sassha, demonstrating their limited effectiveness compared to \sassha. 
We attribute this outcome to \sassha reaching a flatter and better generalizing solution along with stable preconditioning, which together enables consistent outperformance over first-order baselines.

\begin{table}[ht!]
    \centering
    \caption{
    Performance comparison of \sassha against SGD and AdamW with twice the epoch allocation. \sassha achieves better results with significantly fewer epochs.}
    \vskip 0.1in
    
    \resizebox{\linewidth}{!}{
        \begin{tabular}{lcccccc}
        \toprule
         & RN20 - CIFAR-10 & RN32 - CIFAR-10 & RN32 - CIFAR-100 & WRN28 - CIFAR-100 & RN50 - ImageNet & ViT\_s - ImageNet \\
         & \texttt{Acc (epoch)} & \texttt{Acc (epoch)} & \texttt{Acc (epoch)} & \texttt{Acc (epoch)} & \texttt{Acc (epoch)} & \texttt{Acc (epoch)} \\
        \midrule
        SGD        & $ 92.62 $ (320e) & $ 93.43 $ (320e) & $ 69.93 $ (320e) & $ 80.50 $ (400e) & $ 75.90 $ (180e) & $ 63.64 $ (180e) \\
        AdamW      & $ 92.55 $ (320e) & $ 92.97 $ (320e) & $ 69.50 $ (320e) & $ 79.46 $ (400e) & $ 75.57 $ (180e) & $ 66.97 $ (180e) \\
        \midrule
        \rowcolor{green!20} \sassha  & $ \textbf{92.98} $ (160e) & $ \textbf{94.09} $ (160e) & $ \textbf{72.14} $ (160e) & $ \textbf{83.54} $ (200e) & $ \textbf{76.43} $ (90e) & $ \textbf{69.20} $ (90e) \\
        \bottomrule
        \end{tabular}
    }
    \label{tab:comp_fo_fair}
\end{table}

\section{Effectiveness of Stable Hessian Approximations in \sassha } 
\label{app:samsophia}

\begin{table}[ht!]
    \centering
    \caption{
    Results of Sophia-H with sharpness minimization.
    }
    \vskip 0.1in
    \resizebox{0.7\linewidth}{!}{
        \begin{tabular}{lcccc}
        \toprule
         & \multicolumn{2}{c}{CIFAR-10} 
         & \multicolumn{2}{c}{CIFAR-100}  \\
         \cmidrule(l{3pt}r{3pt}){2-3} \cmidrule(l{3pt}r{3pt}){4-5} 
         & \multicolumn{1}{c}{ ResNet-20 } & \multicolumn{1}{c}{ ResNet-32 }   & \multicolumn{1}{c}{ ResNet-32 }   & \multicolumn{1}{c}{ WRN-28-10}    \\
         
         \midrule
           
         SAM   & $ 92.85 _{\textcolor{black!60}{\pm 0.07} } $ 
         & $ 93.89 _{\textcolor{black!60}{\pm 0.13} } $ 
         & $ 71.99 _{\textcolor{black!60}{\pm 0.20} } $ 
         & $ 83.14 _{\textcolor{black!60}{\pm 0.13} } $ \\
              
         Sophia-H (with SAM) $_{\text{}}$ & $ 92.53 _{\textcolor{black!60}{\pm 0.39} } $ 
         & $  93.59 _{\textcolor{black!60}{\pm 0.31} } $ 
         & $ 71.31 _{\textcolor{black!60}{\pm 0.43} } $ 
         & $ 80.15 _{\textcolor{black!60}{\pm 0.35} } $  \\
        
         \midrule
        
         \rowcolor{green!20} \sassha    & $ \textbf{92.98} _{\textcolor{black!60}{\pm 0.05}} $ 
         & $ \textbf{94.09} _{\textcolor{black!60}{\pm 0.24} } $ 
         & $ \textbf{72.14} _{\textcolor{black!60}{\pm 0.16} } $ 
         & $\textbf{83.54} _{\textcolor{black!60}{\pm 0.08}}$ 
         \\
        
        \bottomrule
        \end{tabular}
    }
    \label{tab:im_cls_samsophia}
    
\end{table}

We demonstrate limited benefit from naively combining SAM with existing approximate second-order methods without the carefully designed stabilization strategies of \sassha.  
Precisely, we compare the validation accuracy of \sassha with a simple combination of SAM and Sophia, denoted as Sophia-H (with SAM).
We provide results in \cref{tab:im_cls_samsophia}.

We observe that Sophia-H (with SAM) performs worse than SAM, whereas \sassha outperforms both methods, validating the effectiveness of the design choices made in \sassha.
We attribute this to the reduced compatibility of Sophia-H with SAM compared to \sassha.
First, Sophia clipping destroys the relative scale between individual elements of the Hessian. 
This may be particularly problematic when using SAM, where the overall Hessian values tend to be small.
In such cases, even very small differences between Hessian elements may carry nontrivial curvature information.
However, Sophia clipping abruptly replaces small or negative Hessian values based on a predefined and fixed threshold criterion.
As a result, when the Hessian is generally small due to SAM, informative dimensions may fall below the threshold, removing potentially critical variations and thereby deteriorating the quality of preconditioning.
This situation also raises the sensitivity to hyperparameters like the clipping threshold and makes the optimization process more dependent on careful tuning.
Conversely, the stable Hessian approximation in \sassha, incorporating the absolute function and square rooting, preserves the relative scale among the Hessian entries by smoothly adjusting their magnitudes.

In addition, the use of sophia clipping results in Sophia partially performing signSGD over a subset of parameters \citep{sophia}, which may lead to suboptimal convergence in typical situations \cite{karimireddy2019error}.

\newpage

\section{Additional Language Modeling Experiments} \label{app:gpt2}

\begin{wraptable}{r}{0.35\linewidth}
    \vspace{-3em}
    \centering
    \caption{
    GPT2 pretraining results.
    \sassha achieves similar performance to state-of-the-art methods.}
    \vskip 0.1in
    \small
    \label{tab:gpt2_results}
    \resizebox{\linewidth}{!}{
    \begin{tabular}{lrr}
        \toprule
        Method               & \texttt{Loss} \hspace{0.05em} & \texttt{Perplexity} \\
        \midrule
        AdamW                           & 2.9622   &  19.353 \hspace{1em} \\
        Sophia-G                        & 2.9307   &  18.751 \hspace{1em} \\
        SAM $_{\text{AdamW}}$           & 2.9558   &  19.196 \hspace{1em} \\
        Sophia-G (with SAM)             & 2.9319   &  18.773 \hspace{1em} \\  
        \midrule
        \rowcolor{green!20} \sassha     & 2.9445   &  19.015 \hspace{1em} \\
         \bottomrule
    \end{tabular}
    }
\end{wraptable}

In this section, we provide additional language modeling experiments.
We train GPT2-small for 50k iterations on OpenWebText \citep{pile}, using various optimization methods for comparison.
For AdamW and Sophia-G, we adopt the best hyperparameter configurations reported by \citep{sophia}, and set $\rho = 0.1$ for all SAM-related methods.
Due to limited computational resources, all experiments were run with a single random seed.
The results are presented in \cref{tab:gpt2_results}. 
We observe that \sassha achieves performance comparable to related methods.

\section{Comparison with Advanced SAM Variants} \label{app:samvariants_vs_sassha}

Thus far, our primary focus has centered on validating the effectiveness of \sassha in the context of approximate second-order optimization.
While this remains the principal objective of our study, here we additionally compare \sassha with advanced SAM variants (\ie ASAM \citep{asam}, GSAM \citep{gsam}) to prove that \sassha is a sensible approach.
We also evaluate \gsassha (\sassha with surrogate gap guided sharpness from \citep{gsam}) for fair comparison.
The results are represented in \cref{tab:im_cls_samvariants}.
\begin{table}[ht!]
    \centering
    \caption{\sassha v.s. advanced SAM variants in Image classification.}
    \vskip 0.1in
    
    \resizebox{0.8\linewidth}{!}{
        \begin{tabular}{lcccccc}
        \toprule
         & \multicolumn{2}{c}{CIFAR-10} 
         & \multicolumn{2}{c}{CIFAR-100} 
         & \multicolumn{2}{c}{ImageNet} \\
         \cmidrule(l{3pt}r{3pt}){2-3} \cmidrule(l{3pt}r{3pt}){4-5} \cmidrule(l{3pt}r{3pt}){6-7}
           & \multicolumn{1}{c}{ ResNet-20 } & \multicolumn{1}{c}{ ResNet-32 } & \multicolumn{1}{c}{ ResNet-32 }  & \multicolumn{1}{c}{ WRN-28-10} & \multicolumn{1}{c}{ ResNet-50 } & \multicolumn{1}{c}{ ViT-s-32} \\ \midrule
           
        ASAM       & $ 92.96 _{\textcolor{black!60}{\pm 0.25}} $ 
        & $ 93.85 _{\textcolor{black!60}{\pm 0.15} } $ 
        & $ 72.02 _{\textcolor{black!60}{\pm 0.28} } $ 
        & $ 83.39_{\textcolor{black!60}{\pm 0.06} } $ 
        & $ 76.54 _{\textcolor{black!60}{\pm 0.15} } $ 
        & $ 68.26 _{\textcolor{black!60}{\pm 0.36}}$  \\
        
        GSAM       & $ 92.72 _{\textcolor{black!60}{\pm 0.39} } $ 
        & $ 93.76 _{\textcolor{black!60}{\pm 0.31} } $ 
        & $ 72.10 _{\textcolor{black!60}{\pm 0.43} } $ 
        & $ 83.21 _{\textcolor{black!60}{\pm 0.39} } $ 
        & $ 76.45 _{\textcolor{black!60}{\pm 0.22} } $ 
        & $ 69.60 _{\textcolor{black!60}{\pm 0.16} } $  \\
        
        \midrule
        
        \rowcolor{green!20} \sassha    & $ \textbf{92.98} _{\textcolor{black!60}{\pm 0.05}} $ 
        & $ 94.09 _{\textcolor{black!60}{\pm 0.24} } $ 
        & $ 72.14 _{\textcolor{black!60}{\pm 0.16} } $ 
        & $ 83.54 _{\textcolor{black!60}{\pm 0.08}}$ 
        & $ 76.43 _{\textcolor{black!60}{\pm 0.18} }$ 
        & $ 69.20 _{\textcolor{black!60}{\pm 0.30} }$ \\
        
        \rowcolor{green!20} \gsassha   & $ 92.94 _{\textcolor{black!60}{\pm 0.18}} $   
        & $ \textbf{94.15} _{\textcolor{black!60}{\pm 0.12} } $ 
        & $ \textbf{72.18} _{\textcolor{black!60}{\pm 0.52} } $ 
        & $ \textbf{83.56} _{\textcolor{black!60}{\pm 0.27} } $ 
        & $\textbf{76.66} _{\textcolor{black!60}{\pm 0.23}}$  
        & $\textbf{69.67} _{\textcolor{black!60}{\pm 0.14}}$ \\
        \bottomrule
        \end{tabular}
    }
    \label{tab:im_cls_samvariants}    
\end{table}

We find that \sassha is competitive with these advanced SAM variants.
However, we note clearly that those SAM variants require considerably more hyperparameter tuning to achieve generalization performance comparable to \sassha. For example, GSAM introduces an additional hyperparameter $\alpha$, demanding as much tuning effort as tuning $\rho$. 
Similarly, ASAM, as noted by its authors, typically necessitates exploring a broader $\rho$ range, as its appropriate value is approximately 10 times larger than that of SAM.
In our setup, tuning GSAM and ASAM involved $4.5\times \sim 15.75 \times$ and $3\times \sim 8 \times$ larger search grids compared to \sassha, respectively.
We provide detailed setup and hyperparameter search space below.

\textbf{Setup and Search space.} For ResNet, we use SGD as the base methods for ASAM and GSAM, while for ViT, AdamW with gradient clipping set to 1.0 serves as the base methods. 
For all models, typical cross entropy loss is used (not label-smoothing cross entropy), and the best learning rate and weight decay of the base methods are selected in experiments with ASAM and GSAM.
All algorithms are evaluated with constant $\rho$ (without scheduling). 
For learning rate schedule, we apply multi-step decay with a decay rate of 0.1 for ResNet on CIFAR, and use cosine learning rate decay with 8 warm-up epochs for ViTs.

\begin{table}[ht!]
\renewcommand{\arraystretch}{2}
\centering
\resizebox{0.8\textwidth}{!}{%
    \begin{tabular}{lccc}
    \toprule
    \rowcolor{lgray} 
      
    & ResNet/CIFAR
    & ResNet/ImageNet 
    & ViT/ImageNet \\
    \midrule
    \hspace{1em} $\rho$ \hspace{1em}  
    & $\{0.01, 0.05, 0.1, 0.15, 0.2, 0.25, 0.3, 0.35, 0.4 \} $
    & $\{0.01, 0.05, 0.1, 0.15, 0.2, 0.25, 0.3\}$
    & $\{0.1, 0.2, 0.3, 0.4, 0.5, 0.6 \}$\\
    \midrule
    \hspace{1em} $\alpha$ \hspace{1em}
    & $\{0.01, 0.05, 0.1, 0.15, 0.2, 0.25, 0.3 \} $
    & $ \{0.01, 0.05, 0.1, 0.15, 0.2 \}$
    & $\{0.1, 0.2, 0.3 \}$\\ \bottomrule
    \end{tabular}
    }
    \caption{Hyperparameter search space for GSAM and G-\sassha}
    \label{tab:hyper_gsam}
\end{table}

\newpage

\begin{table}[ht!]
\renewcommand{\arraystretch}{2}
\centering
\resizebox{0.8\textwidth}{!}{%
    \begin{tabular}{lcc}
    \toprule
    \rowcolor{lgray} 
      
    & ResNet/CIFAR
    & ImageNet  \\
    \midrule
    \hspace{1em} $\rho$ \hspace{1em}
    & $\Bigl\{ \substack{0.01, 0.05, 0.1, 0.15, 0.2, 0.25, 0.3, 0.4, 0.5, 0.6, 0.7, 0.8, 0.9, 1, \\ 1.1, 1.2, 1.3, 1.4, 1.5, 1.6, 1.7, 1.8, 1.9, 2, 2.5, 3, 3.5, 4, 4.5, 5, 5.5, 6} \Bigr\} $
    & $\{0.1, 0.15, 0.2, 0.3, 0.4, 0.5, 0.6, 0.7, 0.8, 0.9, 1, 1.5, 2 \}$\\ \bottomrule
    \end{tabular}
    }
    \caption{Hyperparameter search space for ASAM}
    \label{tab:hyper_gsam2}
\end{table}

\section{Additional Label Noise Experiments}
\label{app:add_label_noise}

\begin{table}[ht!]
    \centering
    \caption{Robustness to label noise. Here we measure the validation accuracy under various levels of label noise using ResNet-32 trained on CIFAR-100 and CIFAR-10. \sassha shows much robust performance under label noise.}
    \label{tab:noise_label_sassha}
    \vskip 0.1in
    \resizebox{\linewidth}{!}{%
    \begin{tabular}{lccccccccc}
        \toprule
        & \multicolumn{4}{c}{CIFAR-10} & \multicolumn{4}{c}{CIFAR-100} \\ 
        \cmidrule(l{3pt}r{3pt}){2-5} \cmidrule(l{3pt}r{3pt}){6-9} 
        Noise level & {0\%} & {20\%} & {40\%} & {60\%} & {0\%} & {20\%} & {40\%} & {60\%} \\ 
        \midrule
        SGD                 & 
        $ 92.69 _{\textcolor{black!60}{\pm 0.06}}$     &
        $ 89.91 _{\textcolor{black!60}{\pm 0.87}}$     &
        $ 87.26 _{\textcolor{black!60}{\pm 0.40}}$     &
        $ 82.72 _{\textcolor{black!60}{\pm 1.59}}$     &
        $ 69.32 _{\textcolor{black!60}{\pm 0.19}}$     &
        $ 62.18 _{\textcolor{black!60}{\pm 0.06}}$     & 
        $ 55.78 _{\textcolor{black!60}{\pm 0.55}}$     &  
        $ 45.53 _{\textcolor{black!60}{\pm 0.78}}$     \\
        
        SAM $_{\text{SGD}}$ & 
        $ 93.89 _{\textcolor{black!60}{\pm 0.13}}$     &
        $ 92.27 _{\textcolor{black!60}{\pm 0.14}}$     &
        $ 90.11 _{\textcolor{black!60}{\pm 0.25}}$     &
        $ 85.79 _{\textcolor{black!60}{\pm 0.30}}$     &
        $ 71.99 _{\textcolor{black!60}{\pm 0.20}}$     &
        $ 65.53 _{\textcolor{black!60}{\pm 0.11}}$     & 
        $ 61.20 _{\textcolor{black!60}{\pm 0.17}}$     &
        $ 51.93 _{\textcolor{black!60}{\pm 0.47}}$     \\

        AdaHessian         & 
        $ 92.48 _{\textcolor{black!60}{\pm 0.15}}$     &
        $ 90.11 _{\textcolor{black!60}{\pm 0.01}}$     & 
        $ 86.88 _{\textcolor{black!60}{\pm 0.04}}$     &
        $ 83.25 _{\textcolor{black!60}{\pm 0.01}}$     &
        $ 68.06 _{\textcolor{black!60}{\pm 0.22}}$     &
        $ 63.06 _{\textcolor{black!60}{\pm 0.25}}$     & 
        $ 58.37 _{\textcolor{black!60}{\pm 0.13}}$     & 
        $ 46.02 _{\textcolor{black!60}{\pm 1.96}}$     \\

        Sophia-H           & 
        $ 91.99 _{\textcolor{black!60}{\pm 0.08}}$     &
        $ 89.93 _{\textcolor{black!60}{\pm 0.01}}$     &
        $ 87.30 _{\textcolor{black!60}{\pm 0.51}}$     & 
        $ 82.78 _{\textcolor{black!60}{\pm 1.43}}$     &
        $ 67.76 _{\textcolor{black!60}{\pm 0.37}}$     &
        $ 62.34 _{\textcolor{black!60}{\pm 0.47}}$     & 
        $ 56.54 _{\textcolor{black!60}{\pm 0.28}}$     & 
        $ 45.37 _{\textcolor{black!60}{\pm 0.27}}$     \\

        Shampoo           &  
        $ 90.23 _{\textcolor{black!60}{\pm 0.83}}$     &
        $ 88.14 _{\textcolor{black!60}{\pm 0.29}}$     & 
        $ 85.15 _{\textcolor{black!60}{\pm 0.61}}$     & 
        $ 81.16 _{\textcolor{black!60}{\pm 0.30}}$     &
        $ 64.08 _{\textcolor{black!60}{\pm 0.46}}$     &
        $ 58.85 _{\textcolor{black!60}{\pm 0.66}}$     & 
        $ 53.82 _{\textcolor{black!60}{\pm 0.71}}$     & 
        $ 42.91 _{\textcolor{black!60}{\pm 0.99}}$    \\
        
        \midrule
        
        \rowcolor{green!20} \sassha         & 
        $ \textbf{94.09} _{\textcolor{black!60}{\pm 0.24}}$    &
        $ \textbf{92.49} _{\textcolor{black!60}{\pm 0.11}}$    &
        $ \textbf{90.29} _{\textcolor{black!60}{\pm 0.11}}$    &
        $ \textbf{86.50} _{\textcolor{black!60}{\pm 0.08}}$    &
        $ \textbf{72.14} _{\textcolor{black!60}{\pm 0.16}}$    &
        $ \textbf{66.78} _{\textcolor{black!60}{\pm 0.47}}$    &
        $ \textbf{61.97} _{\textcolor{black!60}{\pm 0.27}}$    &
        $ \textbf{53.98} _{\textcolor{black!60}{\pm 0.57}}$    \\ 
        
        \bottomrule
    \end{tabular}}
\end{table}

\section{\msassha: Efficient Perturbation} \label{app:msassha}

Having explored techniques to reduce the computational cost of second-order methods, here we consider employing techniques to alleviate the additional gradient computation in sharpness-minimization.
Prior works have suggested different ways to reduce this computational overhead including infrequent computations \citep{looksam}, use of sparse perturbations \citep{mi2022make}, or computing with selective weight and data \citep{esam}.
In particular, we employ the approaches of \citet{becker2024momentum}, which uses the normalized negative momentum as the perturbation:
\begin{equation}\label{eq:lsam-perturb}
    \epsilon^\star_t = \rho\frac{m_{t-1}}{\|m_{t-1}\|_2},
\end{equation}
which entirely eliminates the need for additional gradient computation with similar generalization improvement as the original SAM.
We call this low-computation alternative as \msassha and evaluate this across vision, language, and label noise tasks, as we did in the main sections. 
The results are presented in \cref{tab:im_cls_results_msassha,tab:language_msassha,tab:noise_label_msassha}, respectively. 

Despite having a computational cost comparable to first-order methods like SGD and Adam, and significantly lower than approximate second-order methods, \msassha demonstrates superior performance over both first-order and second-order approaches.
In image classification, \msassha proves more effective than the best-performing approximate second-order methods by 2\% on CIFAR-100 with ResNet-32 and by 2.5\% with ResNet-50, while also exceeding AdamW by approximately 1.6\% on ViT.
For language pretraining, it attains a test perplexity that is 22 points lower than the second-best performinig Sophia-H  and outperforms AdamW in nearly all language tasks.
Lastly, \msassha surpasses other methods across all noise levels, proving highly resilient in the presence of extreme label noise.
These results reaffirm the effectiveness and consistency of our well-engineered design choices, which enable the stable integration of efficient sharpness minimization into second-order optimization while retaining its benefits.

\begin{table*}[ht!]
    \vspace{-1em}
    \centering
    \caption{\msassha $\text{v.s.}$ baselines in image classification. \msassha shows superior performance.}
    \vskip 0.1in
    \resizebox{\linewidth}{!}{
        \begin{tabular}{clcccccc}
        \toprule
         & 
         & \multicolumn{2}{c}{CIFAR-10} 
         & \multicolumn{2}{c}{CIFAR-100} 
         & \multicolumn{2}{c}{ImageNet} \\
         \cmidrule(l{3pt}r{3pt}){3-4} \cmidrule(l{3pt}r{3pt}){5-6} \cmidrule(l{3pt}r{3pt}){7-8}
         \multicolumn{1}{c}{ Category }
         & \multicolumn{1}{c}{ Method }
         & \multicolumn{1}{c}{ ResNet-20 } 
         & \multicolumn{1}{c}{ ResNet-32 } 
         & \multicolumn{1}{c}{ ResNet-32 }  
         & \multicolumn{1}{c}{ WRN-28-10} 
         & \multicolumn{1}{c}{ ResNet-50 } 
         & \multicolumn{1}{c}{ ViT-s-32} \\ \midrule

       \multirow{2}{*}{First-order}  & SGD       & $ 92.03 _{ \textcolor{black!60}{\pm 0.32} } $ 
        & $ 92.69 _{\textcolor{black!60}{\pm 0.06} } $ 
        & $ 69.32 _{\textcolor{black!60}{\pm 0.19} } $ 
        & $ 80.06 _{\textcolor{black!60}{\pm 0.15} } $ 
        & $ 75.58 _{\textcolor{black!60}{\pm 0.05} } $ 
        & $62.90 _{\textcolor{black!60}{\pm 0.36}}$  \\

        & AdamW      & $ 92.04 _{\textcolor{black!60}{\pm 0.11} } $ & $ 92.42 _{\textcolor{black!60}{\pm 0.13} } $ & $ 68.78 _{\textcolor{black!60}{\pm 0.22} } $  & $ 79.09 _{\textcolor{black!60}{\pm 0.35} } $ & $ 75.38 _{\textcolor{black!60}{\pm 0.08} } $ & $ 66.46 _{\textcolor{black!60}{\pm 0.15} } $ \\
        
        

        \midrule
        
        \multirow{4.5}{*}{Second-order}                   &
        AdaHessian                                      &
        $ 92.00 _{\textcolor{black!60}{\pm 0.17} } $    &
        $ 92.48 _{\textcolor{black!60}{\pm 0.15} } $    &
        $ 68.06 _{\textcolor{black!60}{\pm 0.22} } $    &
        $ 76.92 _{\textcolor{black!60}{\pm 0.26} } $    &
        $ 73.64 _{\textcolor{black!60}{\pm 0.16} } $    &
        $ 66.42 _{\textcolor{black!60}{\pm 0.23} } $   \\
        
        & Sophia-H   &
        $ 91.81 _{\textcolor{black!60}{\pm 0.27} } $    &
        $ 91.99 _{\textcolor{black!60}{\pm 0.08} } $    &
        $ 67.76 _{\textcolor{black!60}{\pm 0.37} } $    &
        $ 79.35 _{\textcolor{black!60}{\pm 0.24} } $    &
        $ 72.06 _{\textcolor{black!60}{\pm 0.49} } $    &
        $62.44 _{\textcolor{black!60}{\pm 0.36}}$      \\
        
        & Shampoo    &
        $88.55 _ {\textcolor{black!60}{\pm 0.83}}$   &
        $90.23 _{\textcolor{black!60}{\pm 0.24}}$    &
        $64.08 _{\textcolor{black!60}{\pm 0.46}}$    &
        $74.06 _{\textcolor{black!60}{\pm 1.28}}$    &
        $*$                                          &
        $*$  \\
        
        \cmidrule(l{3pt}r{3pt}){2-8}
        
        \rowcolor{green!20} \cellcolor{white} &  
        \msassha   &
        $ \textbf{92.36} _{\textcolor{black!60}{\pm 0.23} } $   
        & $ \textbf{93.18} _{\textcolor{black!60}{\pm 0.30} } $ 
        & $ \textbf{70.93} _{\textcolor{black!60}{\pm 0.21} } $ 
        & $ \textbf{81.53} _{\textcolor{black!60}{\pm 0.27} } $ 
        & $ \textbf{76.00} _{\textcolor{black!60}{\pm 0.04}}$  
        & $ \textbf{68.04} _{\textcolor{black!60}{\pm 0.14}}$ \\
        
        \bottomrule
        \end{tabular}
    }
    \label{tab:im_cls_results_msassha}
\end{table*}

\begin{table}[ht!]
    \centering
    \caption{
    \msassha $\text{v.s.}$ baselines in language tasks. For pretraining, \msassha achieves the lowest perplexity among all methods. For finetuning, \msassha performs better than AdamW and compares competitively with Sophia-H.
    }
    \vskip 0.1in
    \resizebox{\linewidth}{!}{
        \begin{tabular}{lc}
            \toprule
             & \multicolumn{1}{c}{\textbf{Pretrain} / GPT1-mini} \\
             \cmidrule(l{3pt}r{3pt}){2-2}
             & Wikitext-2 \\
             & \texttt{Perplexity}\\
            \midrule
            
            AdamW & $ 175.06 $ \\
            AdaHessian & $ 407.69 $ \\
            Sophia-H & $ 157.60 $ \\
            \midrule \rowcolor{green!20} \msassha &
            $ \textbf{125.01} $ \\
            
            \bottomrule
        \end{tabular}
        \begin{tabular}{|ccccccc}
            \toprule
                         \multicolumn{7}{|c}{\textbf{Finetune} / SqeezeBERT} \\
                         \cmidrule(l{3pt}r{3pt}){1-7}
                         SST-2 &  MRPC & STS-B & QQP & MNLI & QNLI & RTE \\
             \texttt{Acc} &  \texttt{Acc / F1}  & \texttt{S/P corr.} & \texttt{F1 / Acc} & \texttt{mat/m.mat} &  \texttt{Acc} &  \texttt{Acc} \\
            \midrule
            
            $ 90.29 _{\textcolor{black!60}{\pm 0.52}} $ 
            & $ 84.56 _{ \textcolor{black!60}{\pm 0.25} } $ / $ 88.99 _{\textcolor{black!60}{\pm 0.11}} $ 
            & $ 88.34 _{\textcolor{black!60}{\pm 0.15}} $ / $ 88.48 _{\textcolor{black!60}{\pm 0.20}} $ 
            & $ 89.92 _{\textcolor{black!60}{\pm 0.05}} $ / $ 86.58 _{\textcolor{black!60}{\pm 0.11}} $ 
            & $ 81.22 _{\textcolor{black!60}{\pm 0.07}} $ / $ 82.26 _{\textcolor{black!60}{\pm 0.05}} $ 
            & $ 89.93 _{\textcolor{black!60}{\pm 0.14}} $ 
            & $ 68.95 _{\textcolor{black!60}{\pm 0.72}} $  \\
    
    
            $ 89.64 _{\textcolor{black!60}{\pm 0.13}} $ 
            & $ 79.74 _{\textcolor{black!60}{\pm 4.00}} $ / $ 85.26 _{\textcolor{black!60}{\pm 3.50}} $ 
            & $ 86.08 _{\textcolor{black!60}{\pm 4.04}} $ / $ 86.46 _{\textcolor{black!60}{\pm 4.06}} $ 
            & $ 90.37 _{\textcolor{black!60}{\pm 0.05}} $ / $ 87.07 _{\textcolor{black!60}{\pm 0.05}} $ 
            & $ 81.33 _{\textcolor{black!60}{\pm 0.17}} $ / $ 82.08 _{\textcolor{black!60}{\pm 0.02}} $ 
            & $ 89.94 _{\textcolor{black!60}{\pm 0.12}} $ 
            & $ \textbf{71.00} _{\textcolor{black!60}{\pm 1.04}} $ \\
            
            $ \textbf{90.44} _{\textcolor{black!60}{\pm 0.46}} $ 
            & $ 85.78 _{\textcolor{black!60}{\pm 1.07}} $ / $ 89.90 _{\textcolor{black!60}{\pm 0.82}} $ 
            & $ 88.17 _{\textcolor{black!60}{\pm 1.07}} $ / $ \textbf{88.53} _{\textcolor{black!60}{\pm 1.13}} $ 
            & $ 90.70 _{\textcolor{black!60}{\pm 0.04}} $ / $ 87.60 _{\textcolor{black!60}{\pm 0.06}} $ 
            & $ \textbf{81.77} _{\textcolor{black!60}{\pm 0.18}} $ / $ \textbf{82.36} _{\textcolor{black!60}{\pm 0.22}} $ 
            & $ \textbf{90.12}_{\textcolor{black!60}{\pm 0.14}} $ 
            & $ 70.76 _{\textcolor{black!60}{\pm 1.44}} $  \\
            
            \midrule
            \rowcolor{green!20} 
             $ 90.332 _{\pm 0.88} $ 
             & $ \textbf{87.092} _{\pm 1.98} $ / $ \textbf{90.599} _{\pm 1.51} $ 
             & $ \textbf{88.37} _{\pm 0.04} $ / $  88.46 _{\pm 0.07} $ 
             & $ \textbf{90.78} _{\pm 0.05} $ / $ \textbf{87.61} _{\pm 0.07} $ 
             & $ 81.42 _{\pm 0.19} $ / $ 81.94 _{\pm 0.09} $ 
             & $ 89.84 _{\pm 0.22} $ 
             & $ 70.40 _{\pm 0.96} $ \\

            
            \bottomrule
        \end{tabular}
    }
    \label{tab:language_msassha}
\end{table}

\begin{table}[ht!]
    \vspace{-1em}
    \centering
    \caption{Robustness to label noise. Here we measure the validation accuracy under various levels of label noise using ResNet-32 trained on CIFAR-100 and CIFAR-10. \msassha shows much robust performance under label noise.}
    \label{tab:noise_label_msassha}
    \resizebox{\linewidth}{!}{%
    \begin{tabular}{lccccccccc}
        \toprule
        & \multicolumn{4}{c}{CIFAR-10} & \multicolumn{4}{c}{CIFAR-100} \\ 
        \cmidrule(l{3pt}r{3pt}){2-5} \cmidrule(l{3pt}r{3pt}){6-9} 
        Noise level & {0\%} & {20\%} & {40\%} & {60\%} & {0\%} & {20\%} & {40\%} & {60\%} \\ 
        \midrule
        SGD                 & 
        $ 92.69 _{\textcolor{black!60}{\pm 0.06}}$     &
        $ 89.91 _{\textcolor{black!60}{\pm 0.87}}$     &
        $ 87.26 _{\textcolor{black!60}{\pm 0.40}}$     &
        $ 82.72 _{\textcolor{black!60}{\pm 1.59}}$     &
        $ 69.32 _{\textcolor{black!60}{\pm 0.19}}$     &
        $ 62.18 _{\textcolor{black!60}{\pm 0.06}}$     & 
        $ 55.78 _{\textcolor{black!60}{\pm 0.55}}$     &  
        $ 45.53 _{\textcolor{black!60}{\pm 0.78}}$     \\
        
        
        AdaHessian         & 
        $ 92.48 _{\textcolor{black!60}{\pm 0.15}}$     &
        $ 90.11 _{\textcolor{black!60}{\pm 0.01}}$     & 
        $ 86.88 _{\textcolor{black!60}{\pm 0.04}}$     &
        $ 83.25 _{\textcolor{black!60}{\pm 0.01}}$     &
        $ 68.06 _{\textcolor{black!60}{\pm 0.22}}$     &
        $ 63.06 _{\textcolor{black!60}{\pm 0.25}}$     & 
        $ 58.37 _{\textcolor{black!60}{\pm 0.13}}$     & 
        $ 46.02 _{\textcolor{black!60}{\pm 1.96}}$     \\

        Sophia-H           & 
        $ 91.99 _{\textcolor{black!60}{\pm 0.08}}$     &
        $ 89.93 _{\textcolor{black!60}{\pm 0.01}}$     &
        $ 87.30 _{\textcolor{black!60}{\pm 0.51}}$     & 
        $ 82.78 _{\textcolor{black!60}{\pm 1.43}}$     &
        $ 67.76 _{\textcolor{black!60}{\pm 0.37}}$     &
        $ 62.34 _{\textcolor{black!60}{\pm 0.47}}$     & 
        $ 56.54 _{\textcolor{black!60}{\pm 0.28}}$     & 
        $ 45.37 _{\textcolor{black!60}{\pm 0.27}}$     \\

        Shampoo           &  
        $ 90.23 _{\textcolor{black!60}{\pm 0.83}}$     &
        $ 88.14 _{\textcolor{black!60}{\pm 0.29}}$     & 
        $ 85.15 _{\textcolor{black!60}{\pm 0.61}}$     & 
        $ 81.16 _{\textcolor{black!60}{\pm 0.30}}$     &
        $ 64.08 _{\textcolor{black!60}{\pm 0.46}}$     &
        $ 58.85 _{\textcolor{black!60}{\pm 0.66}}$     & 
        $ 53.82 _{\textcolor{black!60}{\pm 0.71}}$     & 
        $ 42.91 _{\textcolor{black!60}{\pm 0.99}}$    \\
        
        \midrule
        
        
        \rowcolor{green!20} \msassha        &  
        $ \textbf{93.18} _{\textcolor{black!60}{\pm 0.23}}$    &
        $ \textbf{91.27} _{\textcolor{black!60}{\pm 0.31}}$    &
        $ \textbf{88.85} _{\textcolor{black!60}{\pm 0.31}}$    &
        $ \textbf{85.17} _{\textcolor{black!60}{\pm 0.24}}$    &
        $ \textbf{70.93} _{\textcolor{black!60}{\pm 0.21}}$    &
        $ \textbf{66.10} _{\textcolor{black!60}{\pm 0.26}}$    &
        $ \textbf{61.13} _{\textcolor{black!60}{\pm 0.28}}$    &
        $ \textbf{52.45} _{\textcolor{black!60}{\pm 0.34}}$    \\    
        \bottomrule
    \end{tabular}}
\end{table}


\end{document}